\pdfoutput=1
\documentclass[10pt]{article}

\providecommand{\tmlroption}{preprint}
\usepackage[\tmlroption]{tmlr}

\usepackage{amsmath}
\usepackage{amssymb}
\usepackage{amsfonts}
\usepackage{amsthm}
\usepackage{algorithm}
\usepackage{algpseudocode}
\usepackage{subcaption}
\usepackage{graphicx}
\usepackage{booktabs}
\usepackage{url}
\usepackage[colorlinks=true,linkcolor=blue,citecolor=blue,urlcolor=blue]{hyperref}
\hypersetup{pdftitle={K-Means as the Zero-Temperature Limit of Normalized RBF Assignments: Rates, Trajectories, and Exact Recovery with Entmax}}

\theoremstyle{plain}
\newtheorem{theorem}{Theorem}
\newtheorem{lemma}[theorem]{Lemma}
\newtheorem{proposition}[theorem]{Proposition}
\newtheorem{corollary}[theorem]{Corollary}
\theoremstyle{definition}

\newtheorem{remark}[theorem]{Remark}

\makeatletter
\newif\ifdeanonymized
\if@accepted\deanonymizedtrue\fi
\if@preprint\deanonymizedtrue\fi
\makeatother
\newcommand{\codeurl}{\ifdeanonymized\url{https://github.com/FelixStaff/KmeansRBF}\else\url{https://anonymous.4open.science/r/XXXX}\fi}

\def\month{MM}
\def\year{YYYY}
\def\openreview{\url{https://openreview.net/forum?id=XXXX}}

\title{K-Means as the Zero-Temperature Limit of\texorpdfstring{\\}{ } Normalized RBF Assignments: Rates, Trajectories,\texorpdfstring{\\}{ } and Exact Recovery with Entmax}

\author{\name Felipe de Jesus Felix Arredondo \email ffelix3ro@gmail.com \\
    \addr School of Engineering and Sciences\\
    Tecnológico de Monterrey\\
    Monterrey, Mexico
    \AND
    \name Manuel Alejandro Ucan Puc \email Alejandro.ucan-puc@tec.mx \\
    \addr School of Engineering and Sciences\\
    Tecnológico de Monterrey\\
    Monterrey, Mexico
    \AND
    \name Carlos Astengo Nouguez \email castengo@tec.mx \\
    \addr School of Engineering and Sciences\\
    Tecnológico de Monterrey\\
    Monterrey, Mexico
}

\begin{document}

\maketitle

\begin{abstract}
    K-Means is usually treated as a discrete procedure outside gradient-based learning, whereas normalized radial basis function (RBF) layers are smooth. We make precise the sense in which K-Means is the zero-temperature limit of such a layer. Relaxing the assignments of the K-Means distortion $J$ to the simplex with a Shannon entropy penalty yields the normalized Gaussian RBF responsibilities as optimal assignments and a reduced objective $F_\sigma$ with $0\le J-F_\sigma\le 2n\sigma^2\log k$, so $F_\sigma$ $\Gamma$-converges to $J$. We keep four levels of statement apart. For the objective, minimal values and global minimizers converge, and above the critical temperature $\sigma^2=\lambda_{\max}$ of the data covariance the configuration with all centroids at the mean is a strict local minimizer. For optimization, alternating minimization is a preconditioned gradient method on $F_\sigma$ whose hard limit is Lloyd's algorithm. For one update, under a squared-distance margin $\Delta_{\min}>0$, the soft centroid map approaches Lloyd's at rate $O(e^{-\Delta_{\min}/2\sigma^2})$. For fixed points and trajectories, at low temperature every nondegenerate Lloyd fixed point has a nearby soft local minimizer, and soft iterates keep Lloyd's partitions while Lloyd's path keeps a margin. With the Tsallis entropy of order $3/2$, Entmax-1.5 responsibilities recover the hard assignments, the Lloyd update and, along a path with margin, Lloyd's whole trajectory exactly for $\sigma\le\sqrt{\Delta_{\min}}/2$, a sharp threshold. Experiments on sixteen datasets confirm every bound, rate and threshold, and the trajectory statements without exception at the edge of their hypotheses; deterministic annealing removes the dependence on initialization on all datasets but one. Used as a differentiable K-Means term, $F_\sigma$ shows the predicted effects: its gradients do not vanish at low temperature, its closed-form centroid update is a preconditioned step that needs no learning rate, a fixed temperature crosses the critical one as learned representations contract, and soft responsibilities keep the codebook of a vector-quantized autoencoder in use.
\end{abstract}

\section{Introduction}

K-Means remains one of the most widely used clustering algorithms \citep{jain2010data,ahmed2020k}, thanks to its simplicity, its closed-form centroid updates and an iteration cost linear in the number of observations \citep{kanungo2002efficient}. Its hard assignments, however, make the distortion a non-differentiable function of the data, so K-Means usually sits outside gradient-based learning: it initializes the centers of a radial basis function (RBF) network \citep{moody1989fast,zhu2011research} or produces targets for a representation \citep{caron2018deep}, and is then left behind. Normalized Gaussian RBF layers compute soft, differentiable assignments to the same kind of centers \citep{broomhead1988multivariable,park1991universal}, and their centers can be trained by gradient descent with the rest of a network \citep{poggio1990networks}. The two are known to be related: soft Gaussian assignments are the E-step of EM for Gaussian mixtures \citep{xu1996convergence}, and K-Means is a variational approximation of EM for isotropic mixtures \citep{lucke2019k}. This paper asks how close the relation is. Is K-Means the limit of a normalized RBF assignment layer, and in what sense: for the objective, for the optimization algorithm, for one update, for fixed points, for whole trajectories?

We answer at each of these levels. Relaxing the assignments of the K-Means distortion to the probability simplex with an entropy penalty at temperature $2\sigma^2$ gives the normalized Gaussian RBF responsibilities as optimal assignments and a smooth reduced objective $F_\sigma$, which approximates the K-Means distortion uniformly and $\Gamma$-converges to it. Alternating minimization of the relaxed problem is a preconditioned gradient method on $F_\sigma$ whose hard limit is Lloyd's algorithm. Under a positive assignment margin, one soft update approaches one Lloyd update exponentially fast, every nondegenerate K-Means fixed point has a soft counterpart nearby, and as long as Lloyd's iterates keep a margin, the soft iterates follow them with the same partitions. Replacing the Shannon entropy by the Tsallis entropy of order $3/2$ gives Entmax-1.5 responsibilities, for which all of these limits are reached exactly at a finite temperature. Table~\ref{tab:summary} collects the statements. Figure~\ref{fig:overview} places them on the temperature axis. Without the margin condition soft and hard trajectories can separate, and we keep the statements about objectives, maps and trajectories distinct throughout.

\begin{figure}[t]
    \centering
    \includegraphics[width=\linewidth]{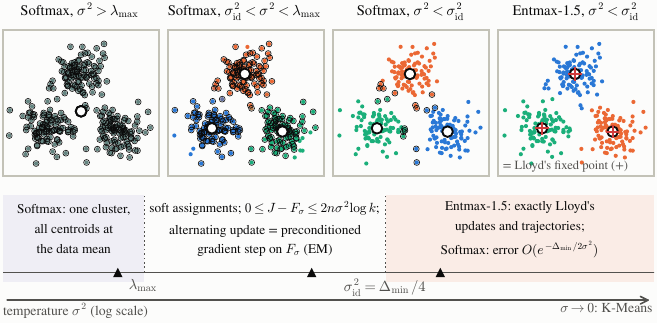}
    \caption{The results of the paper along the temperature axis, on three Gaussian
        clusters in the plane ($n=300$, $k=3$). Top: fixed points of the alternating
        updates at the temperatures marked by triangles, reached by deterministic
        annealing; each point is colored by its responsibilities, as in
        Figure~\ref{fig:explain-assignments}, and circled when they are not one-hot to
        within $10^{-12}$; white disks are the centroids. Above the critical temperature
        $\lambda_{\max}$, the largest eigenvalue of the data covariance, every Softmax
        centroid sits at the data mean (Proposition~\ref{prop:critical-temperature}).
        Between the two thresholds the assignments are soft, the reduced objective is
        within $2n\sigma^2\log k$ of the K-Means distortion
        (Theorem~\ref{thm:uniform-gamma-convergence}), and the alternating update is a
        preconditioned gradient step (Section~\ref{sec:optimization}). Below
        $\sigma_{\mathrm{id}}^2=\Delta_{\min}/4$, with $\Delta_{\min}$ the margin of the K-Means fixed
        point, Softmax remains soft near the cluster boundaries, whereas Entmax-1.5 is
        exactly hard and its fixed point is Lloyd's (crosses), as are its iterates
        (Section~\ref{sec:maps}).}
    \label{fig:overview}
\end{figure}

\paragraph{Main contributions.}
The entropic relaxation of K-Means and its interpretation as the EM algorithm for isotropic Gaussian mixtures are classical \citep{rose1998deterministic,neal1998view,xu1996convergence}, as are its small-variance limit \citep{kulis2012revisiting,broderick2013mad}, the reading of Lloyd's update as a Newton step \citep{bottou1995convergence}, the critical temperature of deterministic annealing \citep{rose1990statistical} and the margin property that makes Entmax outputs exactly sparse \citep{blondel2020learning}; Section~\ref{sec:related} separates this prior work from what is new. Our contribution is a quantitative account that keeps objective-level, optimization-level, map-level and trajectory-level statements apart, and extends it to sparse Tsallis-regularized assignments:
\begin{enumerate}
    \item[(i)] \emph{Objective.} The explicit uniform bound $0\le J-F_\sigma\le 2n\sigma^2\log k$ on the whole centroid space, with convergence of minimal values and global minimizers, and the critical temperature $\sigma^2=\lambda_{\max}$ of the data covariance below which the configuration with all centroids at the mean stops being a local minimizer (Section~\ref{subse:variational-convergence}).
    \item[(ii)] \emph{Optimization.} Alternating minimization as an exact block-preconditioned gradient method on $F_\sigma$, with sufficient decrease and convergence of the whole sequence of iterates; Lloyd's algorithm as its hard limit, and gradient descent with a common step size as a different procedure (Section~\ref{subse:updateequiv}).
    \item[(iii)] \emph{Maps, fixed points and trajectories.} Explicit exponential bounds for the responsibilities and the centroid map under a margin; a unique soft fixed point, a strict local minimizer of $F_\sigma$, next to every nondegenerate Lloyd fixed point; and the tracking of Lloyd's trajectory, with identical partitions, while Lloyd's iterates keep a margin (Section~\ref{subse:centroid-approximation}).
    \item[(iv)] \emph{Exact recovery with Entmax-1.5.} A Tsallis objective that lies strictly above $F_\sigma$ and below or at $J$, with a smaller uniform constant; exact identification of the hard assignments, of the Lloyd update and of Lloyd's whole trajectory for every $\sigma\le\sqrt{\Delta_{\min}}/2$, a sharp threshold that extends to the Entmax-$\alpha$ family; and, at a fixed temperature, identification of all observations except a fraction $O(\sigma^2)$ near the cluster boundaries (Sections~\ref{subsec:tsallis-relaxation}, \ref{sec:objective} and~\ref{subse:entmax-identification}).
    \item[(v)] \emph{Experiments.} Checks of every bound and rate on sixteen datasets and six draws of the data; tests of the trajectory statements at the edge of their hypotheses, with no violation in more than $12{,}000$ runs; and measurements of what the theory leaves open: the agreement of trajectories beyond the hypotheses, deterministic annealing, the gradients of the K-Means term, the dependence on the sample size, the cost at scale, and the K-Means term in deep clustering and in the codebook of a vector-quantized autoencoder (Sections~\ref{sec:experiments} and~\ref{sec:extension}).
\end{enumerate}
In this precise sense, K-Means is the zero-temperature limit of a normalized RBF assignment layer, and the reduced objective $F_\sigma$, with its closed-form centroid update, is the differentiable K-Means term that this limit singles out.

\begin{table}[t]
\centering
\footnotesize
\setlength{\tabcolsep}{4pt}
\renewcommand{\arraystretch}{1.25}
\begin{tabular}{p{0.13\linewidth}p{0.38\linewidth}p{0.38\linewidth}}
\toprule
Level & Shannon entropy (Softmax) & Tsallis entropy of order $3/2$ (Entmax-1.5) \\
\midrule
Objective & $0\leq J-F_\sigma\leq2n\sigma^2\log k$, hence $\Gamma$-convergence (Thm.~\ref{thm:uniform-gamma-convergence}); gap $O(n\sigma^2e^{-\beta})$ under a margin (Prop.~\ref{prop:value-gap-margin}) & $F_\sigma<F_\sigma^{(3/2)}\leq J$, gap at most $\tfrac83n\sigma^2(1-1/\sqrt k)$; $F_\sigma^{(3/2)}=J$ for $\sigma\leq\sigma_{\mathrm{id}}$ (Props.~\ref{prop:tsallis-uniform-approximation}, \ref{prop:shannon-tsallis-ordering}, \ref{prop:value-gap-margin}) \\
Optimization & alternating update $=$ preconditioned gradient step (EM); descent, whole-sequence convergence (Thm.~\ref{thm:alternating-full-convergence}) & descent with sufficient decrease (Prop.~\ref{prop:entmax-descent}) \\
One update & $\|T_\sigma-T_{\mathrm{Lloyd}}\|=O(e^{-\beta})$ with explicit constants (Thm.~\ref{thm:soft-centroid-exponential-bound}) & $T^{(3/2)}_\sigma=T_{\mathrm{Lloyd}}$ for $\sigma\leq\sigma_{\mathrm{id}}$, sharp (Cor.~\ref{cor:entmax-lloyd-identification}) \\
Fixed points & unique soft fixed point within $O(e^{-\beta})$ of each nondegenerate Lloyd fixed point (Thm.~\ref{thm:soft-fixed-point-branch}) & Lloyd fixed points with margin $\geq4\sigma^2$ are fixed points \\
Trajectories & same partitions as Lloyd while its iterates keep a margin (Thm.~\ref{thm:trajectory-tracking}) & identical iterates while Lloyd's keep a margin $\geq4\sigma^2$ (Cor.~\ref{cor:entmax-trajectory}) \\
\bottomrule
\end{tabular}
\caption{The statements of the paper, by level, with $\beta=\Delta_{\min}/(2\sigma^2)$ and $\sigma_{\mathrm{id}}=\sqrt{\Delta_{\min}}/2$. The trajectory statements require their margin condition along Lloyd's path; the others hold at any configuration with a positive margin, or everywhere for the objective bounds.}
\label{tab:summary}
\end{table}

\paragraph{Organization of the Paper.}
The remainder of the paper is organized as follows.
Section~\ref{sec:related} reviews related work and separates what is classical from what is new.
Section~\ref{sec:preliminaries} defines K-Means, its margin and the two relaxations, with the normalized RBF and Entmax-1.5 responsibilities as their optimal assignments.
The theory follows the levels of Table~\ref{tab:summary}: the objective, with the uniform approximation, $\Gamma$-convergence, the value gaps under a margin and the critical temperature (Section~\ref{sec:objective}); the optimization algorithms (Section~\ref{sec:optimization}); and one update, fixed points and trajectories, including exact identification with Entmax-1.5 (Section~\ref{sec:maps}). All proofs are in Appendix~\ref{app:proofs}.
Section~\ref{sec:experiments} tests these predictions on synthetic and real datasets, including the trajectory statements and deterministic annealing.
Section~\ref{sec:extension} turns to the reduced objective as a K-Means term in gradient-based models, with experiments on its gradients, on deep clustering and on vector quantization, and discusses the dependence on the data geometry and the computational cost; Section~\ref{sec:conclusions} concludes.

\section{Related Work}
\label{sec:related}

\paragraph{Entropic relaxation, deterministic annealing and EM.}
Replacing the hard assignments of K-Means by Gibbs distributions over the
clusters, governed by a temperature, is the starting point of deterministic
annealing \citep{rose1990statistical,rose1998deterministic}. The distortion is
minimized under a constraint on the Shannon entropy of the assignments, the
resulting free energy is minimized at each temperature, and the hard solution is
recovered at zero temperature; critical temperatures at which clusters split are
derived explicitly, and the approach extends to pairwise dissimilarity data
\citep{hofmann1997pairwise}. The same construction appears as soft K-means with a
stiffness parameter, which becomes hard K-Means as the stiffness grows, in
\citet[Ch.~20 and~22]{mackay2003information}. The joint energy
\eqref{eq:entropic-functional} is this free energy. Up to the factor $-2\sigma^2$
and an additive constant that depends only on $n$, $k$, $d$ and $\sigma$, its
reduced form $F_\sigma$ is the log-likelihood of a Gaussian mixture with equal
weights and common covariance $\sigma^2I_d$, and the alternating soft update of
Section~\ref{subse:updateequiv} is the corresponding EM iteration. The variational
identity of Lemma~\ref{lem:gibbs-variational-principle}, the optimality of the
Softmax responsibilities (Corollary~\ref{cor:softmax-responsibilities}) and the
descent property of Proposition~\ref{prop:alternating-descent} are therefore
classical: they are the free-energy view of EM \citep{neal1998view}, whose
monotonicity and convergence of limit points to stationary points go back to
\citet{dempster1977maximum} and \citet{wu1983convergence}.
\citet{xu1996convergence} showed that the EM step for Gaussian mixtures is the
gradient of the log-likelihood multiplied by a block-diagonal positive-definite
matrix whose block for the $j$-th mean is the $j$-th covariance divided by the
total responsibility of that component; with equal isotropic covariances this is
the preconditioner $P_\sigma$ of \eqref{eq:variable-metric-preconditioner}. We
include these results to keep the paper self-contained and to fix the notation
used by the new ones. Fuzzy c-means \citep{bezdek1981pattern} is a different soft
relaxation, with a fuzzifier exponent in place of an entropy;
\citet{banerjee2005clustering} extended the correspondence between hard and soft
clustering from the squared Euclidean distance to all Bregman divergences; and
\citet{kearns1997information} compared hard and soft assignment through an
information-theoretic decomposition of the expected distortion.

\paragraph{Small-variance asymptotics.}
That EM for isotropic Gaussian mixtures approaches K-Means as the common variance
tends to zero is made precise by small-variance asymptotics.
\citet{kulis2012revisiting} use it to derive a hard clustering algorithm (DP-means)
from Dirichlet-process mixtures, \citet{jiang2012small} extend it to
exponential-family mixtures through Bregman divergences, and
\citet{broderick2013mad} apply it directly to the posterior, showing that the
negative log-likelihood of a Gaussian mixture approaches the K-Means objective and
that EM approaches the K-Means algorithm as the covariances tend to zero. These
works establish the limits; they do not give uniform error bounds, convergence
rates or a statement at a positive temperature, which are the subject of
Sections~\ref{sec:objective}--\ref{sec:maps}. \citet{lucke2019k} take a different route and
obtain Lloyd's algorithm exactly as a truncated variational EM approximation of
isotropic Gaussian mixtures, without any small-variance limit.

\paragraph{Convergence of K-Means and $\Gamma$-convergence.}
Lloyd's algorithm \citep{lloyd1982least} and the sequential variant of
\citet{macqueen1967some} converge in finitely many steps, possibly to a point that
is not a local minimum \citep{selim1984kmeans}, and the number of steps can be
superpolynomial in the worst case \citep{arthur2006slow}.
\citet{bottou1995convergence} showed that batch K-Means performs Newton steps on
the quantization error, whose Hessian is diagonal with the cluster sizes on the
diagonal; Proposition~\ref{prop:lloyd-preconditioned-gradient} restates this
observation as the hard limit of the variable-metric step. Statistical consistency
of K-Means as $n\to\infty$ \citep{pollard1981strong} concerns a different limit from
the one studied here. For $\Gamma$-convergence we follow
\citet{dalmaso1993introduction} and \citet{braides2002gamma};
Theorem~\ref{thm:uniform-gamma-convergence} is the standard consequence of uniform
convergence to a continuous limit, stated with an explicit rate.

\paragraph{Sparse probability mappings and Tsallis regularization.}
Sparsemax \citep{martins2016softmax} and the $\alpha$-entmax family
\citep{peters2019sparse,correia2019adaptively} replace Softmax by regularized
prediction maps that can return exact zeros. They are gradients of smoothed max
operators defined by a regularizer on the simplex \citep{niculae2017regularized},
which for $\alpha$-entmax is the Tsallis entropy
\citep{tsallis1988possible,blondel2020learning}. \citet{blondel2020learning}
characterized when such a map returns a vertex of the simplex: the associated
Fenchel--Young loss has a separation margin, equal to $1/(\alpha-1)$ for the Tsallis
entropy of order $\alpha$, so that a score gap of at least $1/(\alpha-1)$ forces a
one-hot output; the existence of a margin is tied to the sparsity of the map, and
the Shannon entropy has none. Lemma~\ref{lem:entmax-one-hot-condition} and part (c)
of Theorem~\ref{thm:entmax-alpha-limit} are this property. What is new here is its
use inside the variational problem of K-Means. With the logits
$-d_{ij}/(2\sigma^2)$, the margin condition becomes a condition on the
squared-distance margin of the data, which yields the identification threshold
$\sigma_{\mathrm{id}}$ and its sharpness, the exact even split of tied
observations, the equality of the Entmax and Lloyd centroid maps at a finite
temperature, the exact vanishing of the value gap, and the identity of the two
optimization trajectories for as long as every iterate satisfies the margin
condition. In clustering, Tsallis (non-extensive) entropies have been used in fuzzy
c-means \citep{menard2003possibilistic,yasuda2017approximate,kanzawa2022generalization},
quadratic regularization of fuzzy memberships goes back to
\citet{miyamoto1998fuzzy}, and \citet{xu2016robust} regularize the memberships of
fuzzy K-means with a sparse quadratic term, which yields sparse memberships for
suitable values of its parameter. We are not aware of work that uses sparsemax or
entmax responsibilities in K-Means or EM, or that proves exact recovery of the hard
assignments and of Lloyd's update at a finite temperature.

\paragraph{Differentiable clustering in deep learning.}
Several deep clustering methods train a representation jointly with cluster
parameters. DEC \citep{xie2016unsupervised} uses Student's $t$ soft assignments and
a self-training Kullback--Leibler objective, IDEC \citep{guo2017improved} adds a
reconstruction loss to preserve local structure, and DCN \citep{yang2017towards}
optimizes an explicit K-Means loss with hard assignments by alternating stochastic
updates, again with a reconstruction term to avoid trivial solutions. Closest to
the present setting are the methods that relax K-Means with a Softmax over negative
distances. In Deep k-Means \citep{fard2020deep} the Softmax parameter plays the role
of an inverse temperature, annealed in one variant and fixed in another. DKM
\citep{cho2022dkm} casts K-Means as attention with a fixed temperature and
backpropagates through the K-Means iterations to compress networks. ClusterNet
\citep{wilder2019end} differentiates through the fixed point of a soft K-means layer
with a fixed inverse temperature, using a cosine similarity in its experiments. Soft-to-hard
vector quantization \citep{agustsson2017soft} anneals the hardness of a Softmax
over squared distances to infinity during training. Other approaches avoid the
Softmax relaxation. VQ-VAE \citep{vandenoord2017neural} uses hard nearest-neighbor
assignments with a straight-through gradient; SwAV \citep{caron2020unsupervised}
computes soft assignments by entropically regularized optimal transport with an
equipartition constraint; \citet{genevay2019differentiable} rewrite K-Means as an
entropically regularized optimal transport problem with cluster-size constraints;
and \citet{stewart2023differentiable} cluster through perturbed spanning forests.
The Softmax-based methods rely on the zero-temperature behavior that our results
quantify: Theorem~\ref{thm:soft-centroid-exponential-bound} and
Proposition~\ref{prop:value-gap-margin} give the error made at a given temperature,
and with Entmax-1.5 the temperature need not be sent to zero, since the hard
assignments, the Lloyd update and the hard objective are recovered exactly at a
finite temperature; annealing then only selects the fixed point that is reached
(Section~\ref{subsec:annealing-results}).

\paragraph{Radial basis function networks.}
Radial basis function networks \citep{broomhead1988multivariable} approximate
functions with radial units and are universal approximators
\citep{park1991universal}, and they remain competitive with kernel machines
\citep{que2016back}. In the hybrid scheme of \citet{moody1989fast} the
centers are placed by a self-organized clustering stage before the output weights
are fitted by supervised learning, which is the origin of the common use of
K-Means to initialize RBF centers. Normalizing the Gaussian responses across units
gives normalized Gaussian RBF networks \citep{bugmann1998normalized}; the normalized
responses are the Softmax responsibilities of
Section~\ref{subsec:rbf-responsibilities}, and the RBF assignment layer of this
paper is such a normalized layer without output weights. The connection with
K-Means established here concerns this normalized radial layer alone.

\paragraph{Positioning.}
The Gibbs variational identity, the Softmax responsibilities, the descent and
stationarity of the alternating update, its variable-metric form and the Newton
form of Lloyd's update are classical, as detailed above, and are included in
Sections~\ref{sec:preliminaries} and~\ref{sec:optimization} to keep the paper
self-contained. The new results are the quantitative and finite-temperature
statements listed in the contributions and collected in Table~\ref{tab:summary}.
Several of them are elementary; their value lies in the explicit constants, in
keeping statements about objectives, maps and trajectories apart, and in the finite
temperature at which Entmax-1.5 becomes exact.

\section{Setting: K-Means and Two Relaxations}
\label{sec:preliminaries}

This section defines every object used in the rest of the paper: the K-Means
problem and its margin (Section~\ref{subsec:variational-reparametrization}), the
entropic relaxation, whose optimal assignments are normalized Gaussian RBF
responsibilities (Section~\ref{subsec:rbf-responsibilities}), and the sparse Tsallis
relaxation, whose optimal assignments are Entmax-1.5 responsibilities
(Section~\ref{subsec:tsallis-relaxation}). The proofs of all results are in
Appendix~\ref{app:proofs}.

\subsection{K-Means}
\label{subsec:variational-reparametrization}

Let $X=\{x_1,\dots,x_n\}\subset\mathbb{R}^d$ be a dataset and let $k\geq2$ be the
number of clusters; with a single cluster there is nothing to assign, and several
statements below (the critical temperature, the margins) are vacuous or false for
$k=1$. A configuration of centroids is $\mu=(\mu_1,\ldots,\mu_k)\in\mathbb R^{k\times d}$,
and $d_{ij}(\mu):=\|x_i-\mu_j\|^2$ is the squared distance from the $i$-th observation
to the $j$-th centroid. Following \citet{blomer2016theoretical}, an assignment is a
matrix in
\[
    \mathcal R=\Bigl\{r\in\{0,1\}^{n\times k}:\textstyle\sum_{j=1}^k r_{ij}=1\text{ for every }i\Bigr\},
\]
and K-Means minimizes the joint distortion
\begin{equation}
    \mathcal E_0(\mu,r)
    =
    \sum_{i=1}^n\sum_{j=1}^k
    r_{ij}d_{ij}(\mu)
    \label{eq:joint-hard-distortion}
\end{equation}
over $(\mu,r)\in\mathbb R^{k\times d}\times\mathcal R$. Minimizing over the assignments
alone gives the reduced K-Means objective
\begin{equation}
    J(\mu)
    =
    \min_{r\in\mathcal R}
    \mathcal E_0(\mu,r)
    =
    \sum_{i=1}^n
    \min_{1\leq j\leq k}
    d_{ij}(\mu).
    \label{eq:variational-kmeans}
\end{equation}
The discreteness of $\mathcal R$ is what keeps this problem outside gradient-based
optimization, and both relaxations below replace it by a product of simplices.

For a configuration $\mu$, the set of nearest centroids of $x_i$ is
$A_i(\mu)=\operatorname*{arg\,min}_jd_{ij}(\mu)$. When it is a single index $j_i(\mu)$
for every $i$, the configuration induces the clusters $S_j(\mu)=\{i:j_i(\mu)=j\}$, and
the Lloyd map replaces the centroid of every nonempty cluster by its mean,
\begin{equation}
    T_{\mathrm{Lloyd},j}(\mu)
    =
    \frac{1}{|S_j(\mu)|}
    \sum_{i\in S_j(\mu)}x_i .
    \label{eq:lloyd-map}
\end{equation}
Lloyd's algorithm \citep{lloyd1982least} alternates the two exact block
minimizations of \eqref{eq:joint-hard-distortion}: the nearest-centroid assignment,
with ties broken by keeping the current assignment, and the mean update, with an
empty cluster keeping its centroid. Neither step increases $J$ and there are finitely
many partitions, so the iteration reaches a fixed partition after finitely many
steps \citep{selim1984kmeans}; in practice it is stopped when the centroids no longer
move appreciably or after a budget of $T$ iterations.

How decisively each observation prefers its nearest centroid is measured by its
squared-distance margin and by the smallest margin over the data,
\begin{equation}
    \Delta_i(\mu)
    :=
    \min_{\ell\neq j_i(\mu)}
    \bigl[d_{i\ell}(\mu)-d_{ij_i(\mu)}(\mu)\bigr],
    \qquad
    \Delta_{\min}(\mu)
    :=
    \min_{1\leq i\leq n}\Delta_i(\mu),
    \label{eq:squared-distance-margin}
\end{equation}
with $\Delta_i(\mu):=0$ when $x_i$ has several nearest centroids, so that every
observation has a unique nearest centroid exactly when $\Delta_{\min}(\mu)>0$. When the
configuration $\bar\mu$ is fixed, we drop the argument and write $j_i$, $S_j$,
$\Delta_i$ and $\Delta_{\min}$, together with the balance $\rho:=\min_j|S_j|/n$ of the
clusters. The margin is the constant that governs every bound of
Sections~\ref{sec:objective}--\ref{sec:maps}.

\subsection{Entropic Relaxation and Normalized RBF Responsibilities}
\label{subsec:rbf-responsibilities}
\label{sec:entropic-relaxation}

Replacing $\mathcal R$ by $\mathcal P:=(\Delta^{k-1})^n$, with
$\Delta^{k-1}=\{p\in\mathbb R^k:p_j\geq0,\ \sum_jp_j=1\}$, does not by itself produce
soft assignments: for fixed $\mu$ the distortion is linear in each row $r_i$, so its
minimizers are vertices of the simplex except at ties. Soft assignments appear once
each row is penalized by its negative Shannon entropy
$-H(r_i)=\sum_jr_{ij}\log r_{ij}$, with $0\log0:=0$, at a temperature $\tau>0$. We set
$\tau=2\sigma^2$, so that the optimal weights are Gaussian responses of bandwidth
$\sigma$, and define the entropy-regularized joint energy
\begin{equation}
    \mathcal E_\sigma(\mu,r)
    =
    \sum_{i=1}^n\sum_{j=1}^k
    r_{ij}d_{ij}(\mu)
    +
    2\sigma^2
    \sum_{i=1}^n\sum_{j=1}^k
    r_{ij}\log r_{ij},
    \qquad
    r\in\mathcal P.
    \label{eq:entropic-functional}
\end{equation}
Large temperatures favor diffuse assignments and small ones the nearest centroid.
Minimizing over the assignments rests on the Gibbs variational principle behind
deterministic annealing and the free-energy view of EM
\citep{rose1998deterministic,neal1998view}.

\begin{lemma}[Entropy-regularized soft minimum]
    \label{lem:gibbs-variational-principle}
    Let $c=(c_1,\ldots,c_k)\in\mathbb R^k$ and let $\tau>0$. Then
    \begin{equation}
        \min_{p\in\Delta^{k-1}}
        \left\{
        \sum_{j=1}^k p_jc_j
        +
        \tau\sum_{j=1}^k p_j\log p_j
        \right\}
        =
        -\tau
        \log
        \left(
        \sum_{j=1}^k e^{-c_j/\tau}
        \right).
        \label{eq:gibbs-variational-principle}
    \end{equation}
    Moreover, the minimizer is unique and is given by
    \begin{equation}
        p_j^\star
        =
        \frac{
            e^{-c_j/\tau}
        }{
            \displaystyle
            \sum_{\ell=1}^k e^{-c_\ell/\tau}
        },
        \qquad
        j=1,\ldots,k.
        \label{eq:gibbs-minimizer}
    \end{equation}
\end{lemma}

Applied to each observation, with $c_j=d_{ij}(\mu)$ and $\tau=2\sigma^2$, the lemma
gives the optimal relaxed assignments.

\begin{corollary}[Optimal entropy-regularized responsibilities]
    \label{cor:softmax-responsibilities}
    For every fixed centroid configuration $\mu$ and every $\sigma>0$, the
    problem
    \[
        \min_{r\in\mathcal P}
        \mathcal E_\sigma(\mu,r)
    \]
    has the unique solution
    \begin{equation}
        r_{ij}^{\sigma}(\mu)
        =
        \frac{\exp\!\bigl(-d_{ij}(\mu)/(2\sigma^2)\bigr)}
        {\sum_{\ell=1}^{k}\exp\!\bigl(-d_{i\ell}(\mu)/(2\sigma^2)\bigr)}.
        \label{eq:softmax-variational}
    \end{equation}
\end{corollary}

These are the outputs of a normalized Gaussian RBF layer. A radial unit with center
$\mu_j$ and bandwidth $\sigma$ responds to $x$ with
$a_j^\sigma(x;\mu)=\exp(-\|x-\mu_j\|^2/(2\sigma^2))$, and
$r_{ij}^\sigma(\mu)=a_j^\sigma(x_i;\mu)/\sum_\ell a_\ell^\sigma(x_i;\mu)$ is the Softmax of the
logits $-d_{i\cdot}(\mu)/(2\sigma^2)$. Our analysis concerns this assignment layer and
its trainable centers, without the output weights of an RBF network. Because the
Gaussian decreases with the distance, the largest responsibility goes to a nearest
centroid, and as the temperature vanishes
\begin{equation}
    r_{ij}^{\sigma}(\mu)
    \longrightarrow
    \begin{cases}
        \dfrac{1}{|A_i(\mu)|},
         &
        j\in A_i(\mu),
        \\[6pt]
        0,
         &
        j\notin A_i(\mu),
    \end{cases}
    \qquad
    \text{as }\sigma\to0,
    \label{eq:softmax-tie-limit}
\end{equation}
so the responsibilities become the hard assignments of K-Means, split evenly among
tied centroids. This is convergence of the assignments alone; the objective, the
updates, the fixed points and the trajectories are the subject of
Sections~\ref{sec:objective}--\ref{sec:maps}.

Eliminating the assignments defines the reduced objective
\begin{equation}
    F_\sigma(\mu)
    :=
    \min_{r\in\mathcal P}
    \mathcal E_\sigma(\mu,r)
    =
    -2\sigma^2
    \sum_{i=1}^n
    \log
    \left[
        \sum_{j=1}^k
        \exp\!\left(
        -\frac{d_{ij}(\mu)}{2\sigma^2}
        \right)
        \right].
    \label{eq:reduced-entropic-functional}
\end{equation}
With the expected soft distortion
\begin{equation}
    D_\sigma(\mu)
    :=
    \sum_{i=1}^n\sum_{j=1}^k
    r_{ij}^{\sigma}(\mu)d_{ij}(\mu),
    \label{eq:soft-distortion}
\end{equation}
it splits as
\begin{equation}
    F_\sigma(\mu)
    =
    D_\sigma(\mu)
    -
    2\sigma^2
    \sum_{i=1}^n
    H\!\left(r_i^\sigma(\mu)\right).
    \label{eq:value-distortion-relation}
\end{equation}
Both $D_\sigma$ and $F_\sigma$ approach $J$ uniformly: since
$\sum_jr_{ij}d_{ij}\geq\min_jd_{ij}$ and one-hot assignments are feasible,
\eqref{eq:value-distortion-relation} gives
$J\leq D_\sigma\leq F_\sigma+2n\sigma^2\log k\leq J+2n\sigma^2\log k$. What distinguishes
$F_\sigma$ is that, as a minimum over the assignments, its gradient involves the
responsibilities but not their derivatives, by Danskin's theorem:
\begin{equation}
    \nabla_{\mu_j}F_\sigma(\mu)
    =
    2\sum_{i=1}^n
    r_{ij}^{\sigma}(\mu)(\mu_j-x_i).
    \label{eq:reduced-functional-gradient}
\end{equation}
The gradient of $D_\sigma$, by contrast, also contains the derivatives of the
responsibilities. Each centroid is pulled
toward the observations in proportion to their responsibilities. With the masses
$m_j^\sigma(\mu):=\sum_ir_{ij}^\sigma(\mu)>0$ and the soft centroid map
\begin{equation}
    T_{\sigma,j}(\mu)
    :=
    \frac{
        \sum_{i=1}^n r_{ij}^\sigma(\mu)x_i
    }{
        m_j^\sigma(\mu)
    },
    \qquad
    T_\sigma(\mu)
    =
    \left(
    T_{\sigma,1}(\mu),\ldots,T_{\sigma,k}(\mu)
    \right),
    \label{eq:soft-centroid-map}
\end{equation}
the gradient reads
\begin{equation}
    \nabla_{\mu_j}F_\sigma(\mu)
    =
    2m_j^\sigma(\mu)
    \left(
    \mu_j-T_{\sigma,j}(\mu)
    \right),
    \label{eq:gradient-soft-centroid-map}
\end{equation}
so that, since every mass is positive,
\begin{equation}
    T_\sigma(\mu)=\mu
    \quad\Longleftrightarrow\quad
    \nabla F_\sigma(\mu)=0.
    \label{eq:fixed-point-stationarity}
\end{equation}
The stationary points of $F_\sigma$ are thus exactly the fixed points of $T_\sigma$,
which replaces every centroid by the responsibility-weighted mean of the data.

\subsection{Sparse Relaxation and Entmax-1.5 Responsibilities}
\label{subsec:tsallis-relaxation}

Softmax responsibilities are positive at every $\sigma>0$, so the soft map and the
reduced objective differ from their hard counterparts at every positive
temperature, by amounts that vanish only as $\sigma\to0$
(Section~\ref{sec:maps}). This is not a numerical problem: stable implementations
subtract the largest logit, and the gradients \eqref{eq:reduced-functional-gradient}
converge to those of the hard distortion instead of vanishing
(Section~\ref{subsec:gradient-results}). The motivation for a sparse relaxation is
exactness. It replaces the Shannon entropy by the Tsallis entropy of order $3/2$,
\begin{equation}
    H_{3/2}(p)
    :=
    \frac{4}{3}
    \left(
    1-\sum_{j=1}^k p_j^{3/2}
    \right),
    \label{eq:tsallis-three-halves-entropy}
\end{equation}
which gives the joint energy
\begin{equation}
    \mathcal E_{\sigma}^{(3/2)}(\mu,r)
    =
    \sum_{i=1}^n\sum_{j=1}^k
    r_{ij}d_{ij}(\mu)
    -
    2\sigma^2
    \sum_{i=1}^n
    H_{3/2}(r_i),
    \qquad
    r\in\mathcal P.
    \label{eq:tsallis-joint-functional}
\end{equation}
Unlike the Shannon entropy, this regularizer admits minimizers on the boundary of
the simplex. With the logits
\begin{equation}
    z_{ij}^{\sigma}(\mu)
    :=
    -\frac{d_{ij}(\mu)}{2\sigma^2},
    \label{eq:entmax-distance-logits}
\end{equation}
the optimal assignments are Entmax-1.5 responsibilities \citep{peters2019sparse}.

\begin{proposition}[Variational characterization of Entmax-$1.5$; \citealp{peters2019sparse}]
    \label{prop:entmax-variational-characterization}
    For every fixed centroid configuration $\mu$ and every $\sigma>0$, the
    minimizer of \eqref{eq:tsallis-joint-functional} with respect to the
    $i$-th assignment vector is
    \begin{equation}
        r_{i\cdot}^{\sigma,3/2}(\mu)
        =
        \operatorname{entmax}_{3/2}
        \left(
        z_{i\cdot}^{\sigma}(\mu)
        \right).
        \label{eq:entmax-responsibilities}
    \end{equation}
    Equivalently, its components satisfy
    \begin{equation}
        r_{ij}^{\sigma,3/2}(\mu)
        =
        \left[
            \frac{1}{2}
            \left(
            z_{ij}^{\sigma}(\mu)-\lambda_i^\sigma(\mu)
            \right)
            \right]_+^2,
        \label{eq:entmax-closed-form}
    \end{equation}
    where $[u]_+:=\max\{u,0\}$ and
    $\lambda_i^\sigma(\mu)$ is uniquely selected so that $\sum_{j=1}^k r_{ij}^{\sigma,3/2}(\mu) = 1$.
\end{proposition}

The inactive coordinates of \eqref{eq:entmax-closed-form} are exactly zero, which is
the source of the exact identification of Section~\ref{subse:entmax-identification}. Figure~\ref{fig:explain-assignments} contrasts the two kinds of
responsibilities in the plane.
Define the reduced Tsallis value function $F_\sigma^{(3/2)}(\mu) := \min_{r\in\mathcal P} \mathcal E_\sigma^{(3/2)}(\mu,r)$. By the uniqueness of the optimal responsibilities and Danskin's theorem,
\begin{equation}
    \nabla_{\mu_j}F_\sigma^{(3/2)}(\mu)
    =
    2\sum_{i=1}^n
    r_{ij}^{\sigma,3/2}(\mu)
    (\mu_j-x_i).
    \label{eq:tsallis-reduced-gradient}
\end{equation}
The Entmax centroid map replaces each centroid with mass
$m_j:=\sum_ir_{ij}^{\sigma,3/2}(\mu)>0$ by its responsibility-weighted mean,
\begin{equation}
    T_{\sigma,j}^{(3/2)}(\mu)
    =
    \frac{
        \sum_{i=1}^n
        r_{ij}^{\sigma,3/2}(\mu)x_i
    }{
        \sum_{i=1}^n
        r_{ij}^{\sigma,3/2}(\mu)
    },
    \label{eq:entmax-centroid-map}
\end{equation}
and keeps every centroid with $m_j=0$ in place. Unlike Softmax, Entmax can leave a
cluster without mass, so an implementation must specify this rule; the
identification results of Section~\ref{subse:entmax-identification} rule out empty
clusters at configurations whose hard clusters are nonempty, but not at arbitrary
intermediate iterates. Appendix~\ref{app:EntmaxHardLimit} treats the whole
Entmax-$\alpha$ family, built on the Tsallis entropy of any order $\alpha>1$.

\begin{figure}[t]
    \centering
    \includegraphics[width=0.86\linewidth]{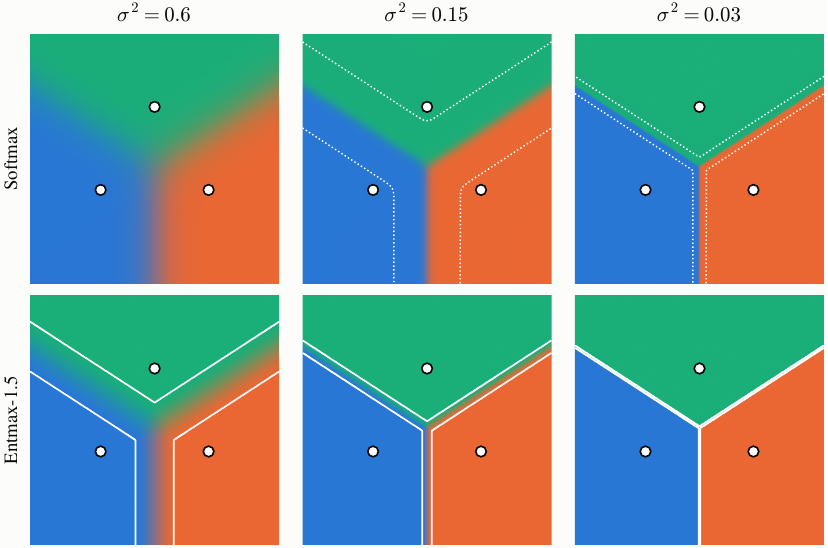}
    \caption{Softmax (top) and Entmax-1.5 (bottom) responsibilities in the plane for
        three fixed centroids (white circles), at three temperatures. Each point is
        colored by the responsibility-weighted mix of the three cluster colors, so pure
        colors are hard assignments and blends are soft ones. Top: outside the dotted
        contour the Softmax responsibilities are within $10^{-6}$ of one-hot, and no
        point is exactly one-hot; at $\sigma^2=0.6$ no point of the window is within
        $10^{-6}$. Bottom: outside the solid contours the Entmax-1.5 responsibilities
        are exactly one-hot. The soft region is the band of points closer than
        $2\sigma^2/\|\mu_j-\mu_\ell\|$ to a bisector
        (Proposition~\ref{prop:per-observation-identification}). Softmax responsibilities
        are within $\varepsilon$ of one-hot only outside a band wider by the factor
        $\tfrac12\log((k-1)/\varepsilon)$, about $7$ for $\varepsilon=10^{-6}$, and exactly
        one-hot nowhere.}
    \label{fig:explain-assignments}
\end{figure}

\section{The Objective: Uniform Approximation and a Critical Temperature}
\label{sec:objective}
\label{subse:variational-convergence}

This section compares the reduced objectives with the K-Means distortion $J$ of
\eqref{eq:variational-kmeans}, everywhere and at configurations with a positive
margin. The proofs are in Appendix~\ref{app:proofs-objective}.

\subsection{Uniform Approximation and \texorpdfstring{$\Gamma$}{Gamma}-Convergence}

Pointwise, $F_\sigma$ increases to $J$ as $\sigma$ decreases: with $\tau=2\sigma^2$, the
derivative of $-\tau\sum_i\log\sum_je^{-d_{ij}/\tau}$ with respect to $\tau$ is
$-\sum_iH(r_{i\cdot}^\sigma)\leq0$, and a monotonically increasing family of continuous
functions with a continuous limit $\Gamma$-converges to it
\citep[Ch.~5]{dalmaso1993introduction}. What this qualitative statement lacks is a
rate. The following theorem gives a uniform bound, whose error is an explicit
multiple of $\sigma^2$ that does not depend on the configuration; it yields the rate
of convergence of the optimal values and the convergence of the minimizers.

\begin{theorem}[Uniform approximation and $\Gamma$-convergence]
    \label{thm:uniform-gamma-convergence}
    Let $X=\{x_1,\ldots,x_n\}\subset\mathbb R^d$ be fixed. Then, for every $\sigma>0$ and every centroid configuration
    $\mu\in\mathbb R^{k\times d}$,
    \begin{equation}
        J(\mu)-2n\sigma^2\log k
        \leq
        F_\sigma(\mu)
        \leq
        J(\mu).
        \label{eq:uniform-softmin-bound}
    \end{equation}
    Consequently,
    \begin{equation}
        \sup_{\mu\in\mathbb R^{k\times d}}
        \left|
        F_\sigma(\mu)-J(\mu)
        \right|
        \leq
        2n\sigma^2\log k,
        \label{eq:uniform-convergence-rate}
    \end{equation}
    and therefore $F_\sigma\longrightarrow J$ uniformly as $\sigma\downarrow0$. In particular, for every nonempty compact set $\mathcal K\subset\mathbb R^{k\times d}$ we have $F_\sigma \xrightarrow{\Gamma} J$ on $\mathcal K$. Moreover, for any sequence $\sigma_h\downarrow0$, let $\mu_h^\star \in \operatorname*{arg\,min}_{\mu\in\mathcal K} F_{\sigma_h}(\mu)$ be a corresponding sequence of minimizers of the smoothed functionals. Then every accumulation point of $\{\mu_h^\star\}_{h\geq1}$ belongs to $\operatorname*{arg\,min}_{\mu\in\mathcal K}J(\mu)$, and furthermore
    \[
        \min_{\mu\in\mathcal K}F_\sigma(\mu)
        \longrightarrow
        \min_{\mu\in\mathcal K}J(\mu).
    \]
\end{theorem}

Neither $J$ nor $F_\sigma$ is coercive for $k\geq2$, since one centroid can be sent to
infinity while both stay bounded. Projecting every centroid onto the convex hull of
the data, however, decreases every distance, which removes the compactness
assumption.

\begin{corollary}[Extension to the whole centroid space]
    \label{cor:gamma-whole-space}
    Let $\mathcal K_X:=\operatorname{conv}(X)^k$. For every $\sigma>0$, every global minimizer of $F_\sigma$ over $\mathbb R^{k\times d}$ lies in $\mathcal K_X$, and
    \[
        \min_{\mu\in\mathbb R^{k\times d}}F_\sigma(\mu)=\min_{\mu\in\mathcal K_X}F_\sigma(\mu),
        \qquad
        \min_{\mu\in\mathbb R^{k\times d}}J(\mu)=\min_{\mu\in\mathcal K_X}J(\mu).
    \]
    Consequently,
    \[
        \left|
        \min_{\mu\in\mathbb R^{k\times d}}F_\sigma(\mu)
        -
        \min_{\mu\in\mathbb R^{k\times d}}J(\mu)
        \right|
        \leq
        2n\sigma^2\log k,
    \]
    and for every sequence $\sigma_h\downarrow0$ and every choice of global minimizers $\mu_h^\star\in\operatorname*{arg\,min}_{\mu\in\mathbb R^{k\times d}}F_{\sigma_h}(\mu)$, the sequence $\{\mu_h^\star\}$ is bounded and each of its accumulation points is a global minimizer of $J$ over $\mathbb R^{k\times d}$.
\end{corollary}

The Tsallis relaxation approximates $J$ uniformly as well, with a smaller constant,
because $H_{3/2}$ is bounded on the simplex by its value $\tfrac43(1-1/\sqrt k)$ at the
uniform distribution.

\begin{proposition}[Uniform Tsallis approximation]
    \label{prop:tsallis-uniform-approximation}
    For every $\mu\in\mathbb R^{k\times d}$ and every $\sigma>0$,
    \begin{equation}
        J(\mu)
        -
        \frac{8n\sigma^2}{3}
        \left(
        1-\frac{1}{\sqrt{k}}
        \right)
        \leq
        F_\sigma^{(3/2)}(\mu)
        \leq
        J(\mu).
        \label{eq:tsallis-uniform-bound}
    \end{equation}
    Consequently, $F_\sigma^{(3/2)} \longrightarrow J$ converges uniformly as $\sigma\to0$.
\end{proposition}

The smaller constant compares only the two worst cases. The value functions
themselves are ordered at every configuration, because the Tsallis entropy never
exceeds the Shannon entropy.

\begin{proposition}[Ordering of the Shannon and Tsallis value functions]
    \label{prop:shannon-tsallis-ordering}
    For every $p\in\Delta^{k-1}$,
    \begin{equation}
        H(p)\geq H_{3/2}(p),
        \label{eq:shannon-tsallis-entropy-ordering}
    \end{equation}
    with equality if and only if $p$ is a vertex of the simplex. Consequently, for every $\mu\in\mathbb R^{k\times d}$ and every $\sigma>0$,
    \begin{equation}
        F_\sigma(\mu)
        <
        F_\sigma^{(3/2)}(\mu)
        \leq
        J(\mu),
        \qquad\text{equivalently}\qquad
        0
        \leq
        J(\mu)-F_\sigma^{(3/2)}(\mu)
        <
        J(\mu)-F_\sigma(\mu).
        \label{eq:shannon-tsallis-value-ordering}
    \end{equation}
\end{proposition}

The Tsallis value function is therefore pointwise closer to the K-Means distortion
than the Shannon one, and not merely bounded by a smaller constant. The same holds
for the Tsallis entropy of any order $\alpha>1$, so the reduced objective of every
Entmax-$\alpha$ relaxation lies between $F_\sigma$ and $J$
(Appendix~\ref{app:proofs-objective}). Figure~\ref{fig:explain-objective} (left) shows the three objectives
along a family of configurations.

\begin{figure}[t]
    \centering
    \includegraphics[width=\linewidth]{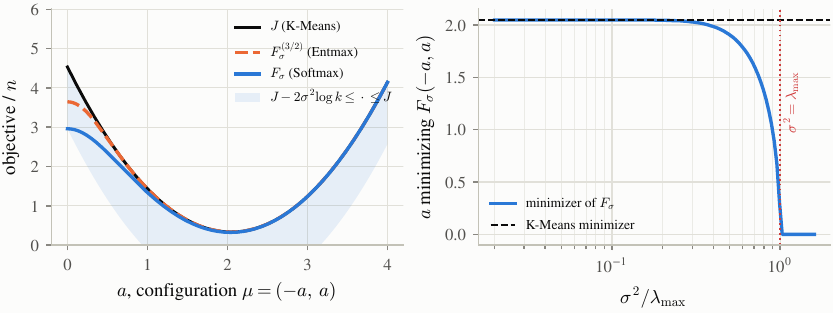}
    \caption{The objectives on one-dimensional data ($n=200$ points in two symmetric
        groups around $\pm2$) and the configurations $\mu=(-a,a)$, $k=2$. Left: per
        observation, $F_\sigma<F_\sigma^{(3/2)}\leq J$ at $\sigma^2=\lambda_{\max}/4$
        (Proposition~\ref{prop:shannon-tsallis-ordering}), with $F_\sigma$ inside the
        band $[J-2\sigma^2\log k,\,J]$ of Theorem~\ref{thm:uniform-gamma-convergence}.
        At $a=0$ the two centroids coincide and $F_\sigma$ attains the lower edge of the
        band; near the K-Means minimizer the three objectives almost coincide, as
        Proposition~\ref{prop:value-gap-margin} predicts. Right: the minimizer of
        $F_\sigma$ over these configurations against $\sigma^2/\lambda_{\max}$. Above the
        critical temperature $\sigma^2=\lambda_{\max}$ both centroids sit at the data
        mean (Proposition~\ref{prop:critical-temperature}); below it they split and
        approach the K-Means minimizer (dashed) as $\sigma\to0$.}
    \label{fig:explain-objective}
\end{figure}

\subsection{Value Gaps under a Positive Margin}

The uniform bounds of Theorem~\ref{thm:uniform-gamma-convergence} and Proposition~\ref{prop:tsallis-uniform-approximation} are worst-case statements over all configurations, and they are attained: when all centroids coincide, every $d_{ij}$ is independent of $j$, the optimal responsibilities are uniform, and $F_\sigma=J-2n\sigma^2\log k$ and $F_\sigma^{(3/2)}=J-\tfrac{8n\sigma^2}{3}(1-1/\sqrt k)$ exactly. At a configuration with positive margin, by contrast, the gaps are far smaller.

\begin{proposition}[Value gaps under a positive margin]
    \label{prop:value-gap-margin}
    Assume that $\bar\mu$ has a positive margin, $\Delta_{\min}>0$, and set $\beta:=\Delta_{\min}/(2\sigma^2)$. For every $\sigma>0$,
    \begin{equation}
        \sigma^2e^{-\beta}
        \leq
        J(\bar\mu)-F_\sigma(\bar\mu)
        \leq
        2n(k-1)\,\sigma^2e^{-\beta},
        \label{eq:value-gap-exponential}
    \end{equation}
    and, for every $0<\sigma\leq\sigma_{\mathrm{id}}=\sqrt{\Delta_{\min}}/2$,
    \begin{equation}
        J(\bar\mu)-F_\sigma^{(3/2)}(\bar\mu)=0 .
        \label{eq:value-gap-entmax-zero}
    \end{equation}
\end{proposition}

Proposition~\ref{prop:value-gap-margin} shows that the $\sigma^2$-linear rates of the uniform bounds describe only the worst case over configurations. Near a K-Means configuration with a positive margin, the Softmax gap decays like $\sigma^2e^{-\Delta_{\min}/2\sigma^2}$ (the two sides of \eqref{eq:value-gap-exponential} differ only by the factor $2n(k-1)$), and the Tsallis-regularized value function coincides with the hard distortion below the identification threshold.

\subsection{The Critical Temperature}

At the opposite end of the temperature range, $F_\sigma$ forgets the clusters
altogether. Let $\bar x$ be the mean of the data, $C:=\frac1n\sum_i(x_i-\bar x)(x_i-\bar x)^\top$
its covariance with largest eigenvalue $\lambda_{\max}$, $D_X:=\max_{i,i'}\|x_i-x_{i'}\|$
the diameter of the data, and $\bar{\boldsymbol x}:=(\bar x,\ldots,\bar x)\in\mathbb R^{k\times d}$
the configuration with every centroid at the mean.

\begin{proposition}[Critical temperature]
    \label{prop:critical-temperature}
    For every $\sigma>0$, $\bar{\boldsymbol x}$ is a stationary point of $F_\sigma$, and the
    Hessian of $F_\sigma$ at $\bar{\boldsymbol x}$ has the eigenvalue $2n/k$ (with
    multiplicity $d$, along common translations of all centroids) and the eigenvalues
    $(2n/k)(1-\lambda/\sigma^2)$, each with multiplicity $k-1$, for every eigenvalue
    $\lambda$ of $C$. Consequently:
    \begin{enumerate}
        \item[(a)] if $\sigma^2>\lambda_{\max}$, then $\bar{\boldsymbol x}$ is a strict local
        minimizer of $F_\sigma$;
        \item[(b)] if $\sigma^2<\lambda_{\max}$, then $\bar{\boldsymbol x}$ is not a local
        minimizer: moving two centroids apart along the leading eigenvector of $C$
        decreases $F_\sigma$;
        \item[(c)] if $\sigma^2>D_X^2$, then $F_\sigma$ is strictly convex on
        $\operatorname{conv}(X)^k$, and $\bar{\boldsymbol x}$ is the unique stationary point
        and the unique global minimizer of $F_\sigma$.
    \end{enumerate}
\end{proposition}

The threshold $\sigma^2=\lambda_{\max}$ and the Hessian computation behind it are
those of the first critical temperature of deterministic annealing
\citep{rose1990statistical}, written for $F_\sigma$ in the parametrization
$\tau=2\sigma^2$: above it the data mean is a stable configuration,
and below it the centroids split (Figure~\ref{fig:explain-objective}, right). Part (c) gives a cruder but global statement.
Together with Theorem~\ref{thm:uniform-gamma-convergence}, the proposition brackets
the behavior of $F_\sigma$ between a single cluster at high temperature and the
K-Means landscape at low temperature; the annealing paths of
Figure~\ref{fig:centroids_paths} start in the first regime.

\section{Optimization: Alternating Minimization and Gradient Descent}
\label{sec:optimization}
\label{subse:updateequiv}

Three mechanisms minimize $F_\sigma$: exact minimization over the centroids with the
responsibilities fixed, alternating minimization of the joint energy, and gradient
descent on $F_\sigma$. The first two produce the same step, an exact block
minimization; the third follows a different trajectory. Apart from the convergence
of the whole sequence and the Entmax counterpart, the results of this section are
the classical EM theory specialized to our setting, included to fix notation
(Section~\ref{sec:related}). The proofs are in
Appendix~\ref{app:proofs-optimization}.

\subsection{Centroids with Fixed Responsibilities}
\label{subsubsec:fixed-responsibilities}

With the responsibilities $r\in\mathcal P$ fixed, the entropy term is constant and the
centroids minimize the weighted quadratic
\begin{equation}
    Q(\mu;r)
    =
    \sum_{i=1}^n\sum_{j=1}^k
    r_{ij}\|x_i-\mu_j\|^2,
    \label{eq:weighted-quadratic-functional}
\end{equation}
with masses and weighted sums
\begin{equation}
    m_j(r)
    :=
    \sum_{i=1}^n r_{ij},
    \qquad
    b_j(r)
    :=
    \sum_{i=1}^n r_{ij}x_i.
    \label{eq:mass-weighted-sum}
\end{equation}
Its gradient and Hessian blocks are
\begin{equation}
    \nabla_{\mu_j}Q(\mu;r)
    =
    2m_j(r)\mu_j-2b_j(r),
    \label{eq:fixed-responsibility-gradient}
\end{equation}
and
\begin{equation}
    \nabla_{\mu_j}^2Q(\mu;r)
    =
    2m_j(r)I_d,
    \qquad
    \nabla_{\mu_j\mu_\ell}^2Q(\mu;r)
    =
    0
    \quad
    \text{for }j\neq\ell.
    \label{eq:fixed-responsibility-hessian}
\end{equation}

\begin{proposition}[Weighted-centroid minimizer]
    \label{prop:weighted-centroid-minimizer}
    Let $r\in\mathcal P$ be fixed and assume that $m_j(r)>0$ for every $j$. Then $Q(\cdot;r)$ is strictly convex and has the unique minimizer
    \begin{equation}
        \mu_j^\star(r)
        =
        \frac{b_j(r)}{m_j(r)}
        =
        \frac{
            \sum_{i=1}^n r_{ij}x_i
        }{
            \sum_{i=1}^n r_{ij}
        },
        \qquad
        j=1,\ldots,k.
        \label{eq:weighted-centroid-minimizer}
    \end{equation}
\end{proposition}

Gradient descent with a common learning rate $\eta$ on the same quadratic contracts
each centroid toward $\mu_j^\star(r)$ by the factor $1-2\eta m_j(r)$, so it converges
only if $\eta<1/m_j(r)$ for every cluster, that is,
\begin{equation}
    0<\eta<
    \frac{1}{
        \max_{1\leq j\leq k}m_j(r)
    }.
    \label{eq:common-learning-rate-condition}
\end{equation}
The step is thus limited by the heaviest cluster, whereas the cluster-dependent step
$\eta_j=1/(2m_j(r))$ reaches $\mu_j^\star(r)$ in one step: the exact
minimization is a blockwise preconditioned gradient step, not a gradient step with
one learning rate. Figure~\ref{fig:explain-optimization} shows the consequence on unbalanced
clusters.

\begin{figure}[t]
    \centering
    \includegraphics[width=\linewidth]{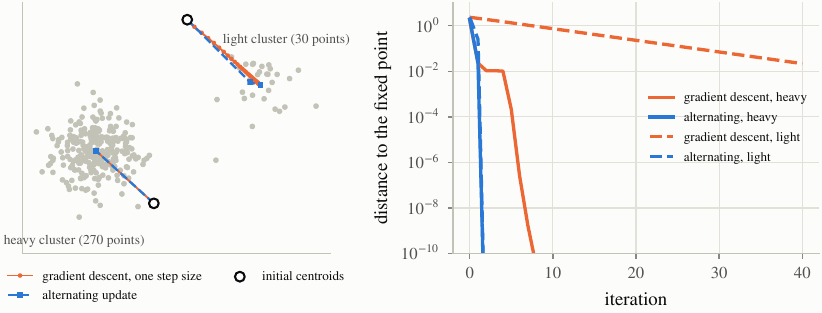}
    \caption{One step size against the preconditioned step, on two unbalanced
        clusters ($270$ and $30$ points, $k=2$, $\sigma^2=0.05$). Gradient descent on
        $F_\sigma$ uses the common step $1/(2\max_jm_j^\sigma)$, the largest that is safe
        for the heavy cluster; the alternating update uses the step $1/(2m_j^\sigma)$ of
        each cluster. Left: centroid paths from the same initialization. Right:
        distance of each centroid to the fixed point. The alternating update reaches
        it in two iterations; with one step size the light centroid contracts only by
        the factor $1-m_{\mathrm{light}}/m_{\mathrm{heavy}}\approx0.89$ per iteration.}
    \label{fig:explain-optimization}
\end{figure}

\subsection{Alternating Minimization}
\label{subsubsec:alternating-minimization}

Alternating minimization of $\mathcal E_\sigma$ sets the responsibilities to their
optimum \eqref{eq:softmax-variational} at the current centroids,
\begin{equation}
    r^{(t+1)}
    =
    \operatorname*{arg\,min}_{r\in\mathcal P}
    \mathcal E_\sigma(\mu^{(t)},r)
    =
    r^\sigma(\mu^{(t)}),
    \label{eq:alternating-responsibility-update}
\end{equation}
and then minimizes over the centroids exactly,
\begin{equation}
    \mu_j^{(t+1)}
    =
    \operatorname*{arg\,min}_{u\in\mathbb R^d}
    \sum_{i=1}^n
    r_{ij}^{(t+1)}
    \|x_i-u\|^2
    =
    \frac{
        \sum_{i=1}^n r_{ij}^{(t+1)}x_i
    }{
        \sum_{i=1}^n r_{ij}^{(t+1)}
    }.
    \label{eq:alternating-centroid-update}
\end{equation}
This is $\mu^{(t+1)}=T_\sigma(\mu^{(t)})$, the EM algorithm for a Gaussian
mixture with equal weights and common covariance $\sigma^2I_d$, written as coordinate
descent on the free energy \citep{neal1998view}; the next two propositions are its
classical monotonicity and stationarity properties
\citep{dempster1977maximum,wu1983convergence}.

\begin{proposition}[Descent under alternating minimization]
    \label{prop:alternating-descent}
    The alternating updates
    \eqref{eq:alternating-responsibility-update}--%
    \eqref{eq:alternating-centroid-update} satisfy
    \begin{equation}
        F_\sigma(\mu^{(t+1)})
        \leq
        F_\sigma(\mu^{(t)}).
        \label{eq:alternating-descent}
    \end{equation}
    Moreover, the inequality is strict whenever $\mu^{(t+1)}\neq\mu^{(t)}$.
\end{proposition}

The proof also gives a sufficient decrease. With $r=r^{(t+1)}$ fixed, $Q(\cdot\,;r)$ is
quadratic with Hessian blocks $2m_j(r)I_d$ and minimizer $\mu^{(t+1)}$, and
$m_j(r^{(t+1)})=m_j^\sigma(\mu^{(t)})$, so
\begin{equation}
    F_\sigma(\mu^{(t)})
    -
    F_\sigma(\mu^{(t+1)})
    \geq
    \sum_{j=1}^k
    m_j^\sigma(\mu^{(t)})
    \left\|
    \mu_j^{(t+1)}-\mu_j^{(t)}
    \right\|^2.
    \label{eq:alternating-sufficient-decrease}
\end{equation}

\begin{proposition}[Stationarity of accumulation points]
    \label{prop:alternating-accumulation-points}
    Fix $\sigma>0$ and let
    $\{\mu^{(t)}\}_{t\geq0}$ be generated by alternating minimization. Then,
    after the first centroid update,
    \[
        \mu^{(t)}
        \in
        \operatorname{conv}(X)^k.
    \]
    The objective sequence $\{F_\sigma(\mu^{(t)})\}_{t\geq0}$ converges, and every accumulation point of the centroid sequence is a stationary point of $F_\sigma$.
\end{proposition}

Because $F_\sigma$ is real-analytic, the \L{}ojasiewicz gradient inequality
strengthens this from accumulation points to the whole sequence.

\begin{theorem}[Convergence of the whole sequence]
    \label{thm:alternating-full-convergence}
    Fix $\sigma>0$. The iterates $\{\mu^{(t)}\}_{t\geq0}$ of alternating minimization
    converge to a single stationary point $\mu_\infty$ of $F_\sigma$.
\end{theorem}

The limit need not be a global minimizer, nor even a local one: the theorem rules
out oscillation between stationary points, not convergence to a saddle.
Section~\ref{subsubsec:fixed-points-trajectories} shows that, near a nondegenerate
Lloyd fixed point and at low temperature, the limit is a strict local minimizer.

\subsection{Variable Metric and the Hard Limit}
\label{subsubsec:variable-metric}
\label{subsubsec:lloyd-limit}

Block minimization and preconditioned gradient step are two descriptions of the same
update. With the block-diagonal operator
\begin{equation}
    P_\sigma(\mu)
    =
    \frac{1}{2}
    \operatorname{diag}
    \left(
    \frac{1}{m_1^\sigma(\mu)},
    \ldots,
    \frac{1}{m_k^\sigma(\mu)}
    \right)
    \otimes I_d,
    \label{eq:variable-metric-preconditioner}
\end{equation}
equation~\eqref{eq:gradient-soft-centroid-map} gives $T_\sigma(\mu)=\mu-P_\sigma(\mu)\nabla F_\sigma(\mu)$:
one alternating update is a variable-metric gradient step on $F_\sigma$, the
equal-weight, isotropic case of the projection-matrix form of EM for Gaussian
mixtures \citep{xu1996convergence}. Gradient descent on $F_\sigma$ with a common
learning rate $\eta_t$,
$\mu_j^{(t+1)}=\mu_j^{(t)}-\eta_t\nabla_{\mu_j}F_\sigma(\mu^{(t)})$, coincides with it, for every
block with a nonzero gradient, only when the step of block $j$ is
$1/(2m_j^\sigma(\mu^{(t)}))$. Because these steps depend on the current masses, the
alternating update is a variable-metric method, and gradient descent with one step
size is a different procedure.

As $\sigma\to0$ at a configuration with $\Delta_{\min}(\mu)>0$ and nonempty clusters,
$r_{ij}^\sigma(\mu)\to\mathbf 1_{\{i\in S_j(\mu)\}}$ by \eqref{eq:softmax-tie-limit}, so
$m_j^\sigma(\mu)\to|S_j(\mu)|$ and $T_\sigma(\mu)\to T_{\mathrm{Lloyd}}(\mu)$;
Section~\ref{sec:maps} quantifies this limit. The limiting preconditioner is the
inverse Hessian of the quantization error of fixed clusters, so Lloyd's update is a
Newton step \citep{bottou1995convergence}.

\begin{proposition}[Lloyd update as a preconditioned gradient step]
    \label{prop:lloyd-preconditioned-gradient}
    Let $S_1,\ldots,S_k$ be fixed nonempty clusters and define
    \[
        Q_S(\mu)
        =
        \sum_{j=1}^k
        \sum_{i\in S_j}
        \|x_i-\mu_j\|^2.
    \]
    Let $H_S = 2\operatorname{diag} \left( |S_1|,\ldots,|S_k| \right) \otimes I_d$. Then the update
    \begin{equation}
        \mu^+
        =
        \mu
        -
        H_S^{-1}
        \nabla Q_S(\mu)
        \label{eq:lloyd-preconditioned-gradient}
    \end{equation}
    satisfies $\mu_j^+ = \frac{1}{|S_j|} \sum_{i\in S_j}x_i$.
\end{proposition}

Lloyd's centroid update is thus at once an exact block minimization, a blockwise
Newton step and a block-diagonally preconditioned gradient step, but not, in
general, a gradient step with a common learning rate.

\subsection{Alternating Entmax}

The alternating scheme carries over to the Tsallis relaxation: given $\mu^{(t)}$, set
$r^{(t+1)}=r^{\sigma,3/2}(\mu^{(t)})$ and $m_j=\sum_ir^{(t+1)}_{ij}$, and apply the Entmax
centroid map \eqref{eq:entmax-centroid-map}, keeping every centroid with $m_j=0$ in
place.

\begin{proposition}[Descent under alternating Entmax]
    \label{prop:entmax-descent}
    The iterates of this scheme satisfy
    \[
        F_\sigma^{(3/2)}(\mu^{(t)})-F_\sigma^{(3/2)}(\mu^{(t+1)})
        \geq
        \sum_{j=1}^km_j\left\|\mu_j^{(t+1)}-\mu_j^{(t)}\right\|^2
        \geq0,
    \]
    with equality only if $\mu^{(t+1)}=\mu^{(t)}$, and the values
    $F_\sigma^{(3/2)}(\mu^{(t)})$ converge. Every fixed point of the scheme is a
    stationary point of $F_\sigma^{(3/2)}$.
\end{proposition}

Re-seeding an empty cluster elsewhere, as our implementation does
(Algorithm~\ref{alg:rbf-gradient}), moves a centroid and changes the responsibilities,
so it is not covered by the proposition; in our experiments no cluster ever lost
its mass at the temperatures used (Table~\ref{tab:methods}). Convergence of the
whole sequence is subtler than for $F_\sigma$. The Tsallis objective is not
analytic, but it is semialgebraic and continuously differentiable, so it still
satisfies the Łojasiewicz gradient inequality \citep{bolte2007lojasiewicz}, which is
all that the argument of Theorem~\ref{thm:alternating-full-convergence} uses of
analyticity. The obstacle is the strong descent condition: the decrease
$\sum_jm_j\|\mu_j^{(t+1)}-\mu_j^{(t)}\|^2$ controls $\|\nabla F_\sigma^{(3/2)}\|\,\|\mu^{(t+1)}-\mu^{(t)}\|$
only when the positive masses are bounded away from zero, and Entmax masses can
be arbitrarily small. If they stay above a constant along the sequence, the
argument applies and the whole sequence converges; in the identification regime
the scheme coincides with Lloyd's algorithm and stops after finitely many steps.
Whether the whole sequence converges without such a condition is left open.

\begin{remark}[The Entmax-$\alpha$ family]
    \label{rem:entmax-alpha-descent}
    Proposition~\ref{prop:entmax-descent} holds for every $\alpha>1$. With the Tsallis
    entropy $H_\alpha$ of Appendix~\ref{app:EntmaxHardLimit}, set
    $\mathcal E_\sigma^{(\alpha)}(\mu,r):=Q(\mu;r)-2\sigma^2\sum_{i=1}^nH_\alpha(r_{i\cdot})$ and
    $F_\sigma^{(\alpha)}(\mu):=\min_{r\in\mathcal P}\mathcal E_\sigma^{(\alpha)}(\mu,r)$. Since $H_\alpha$ is
    strictly concave, the minimizing responsibilities are unique, namely
    $r_{i\cdot}=\operatorname{entmax}_\alpha(-d_{i\cdot}/(2\sigma^2))$, and Danskin's theorem gives
    $\nabla_{\mu_j}F_\sigma^{(\alpha)}(\mu)=2\sum_ir_{ij}(\mu_j-x_i)$, as in
    \eqref{eq:tsallis-reduced-gradient}. The proof of
    Proposition~\ref{prop:entmax-descent} uses only that the responsibility step
    minimizes the joint energy exactly and that $Q(\cdot\,;r)$ is quadratic with
    Hessian blocks $2m_jI_d$, so it applies verbatim; the lower bound becomes
    $-2n\sigma^2\,(1-k^{1-\alpha})/(\alpha(\alpha-1))$, the maximum of $H_\alpha$ on the simplex
    being attained at the uniform distribution. Combined with
    Theorem~\ref{thm:entmax-alpha-limit}(c), the alternating Entmax-$\alpha$ update
    coincides with Lloyd's update at every configuration with nonempty clusters and
    $\Delta_{\min}\geq2\sigma^2/(\alpha-1)$, which for $\alpha=3/2$ is the condition
    $\Delta_{\min}\geq4\sigma^2$ of Corollary~\ref{cor:entmax-lloyd-identification}.
\end{remark}

\section{One Update, Fixed Points and Trajectories}
\label{sec:maps}
\label{subse:centroid-approximation}

Section~\ref{sec:optimization} showed that one soft update tends to one Lloyd update
as $\sigma\to0$. This section quantifies the limit under a positive margin, at a fixed
configuration for Softmax (Section~\ref{subsec:softmax-map}) and exactly for
Entmax-1.5 (Section~\ref{subse:entmax-identification}), and then propagates it to
fixed points and whole trajectories
(Section~\ref{subsubsec:fixed-points-trajectories}). The proofs are in
Appendix~\ref{app:proofs-maps}.

\subsection{Exponential Approximation of the Lloyd Map}
\label{subsec:softmax-map}

Fix a configuration $\bar\mu$ with $\Delta_{\min}>0$, with the notation of
Section~\ref{subsec:variational-reparametrization} at $\bar\mu$. Each responsibility is
controlled by the margin of its observation.

\begin{lemma}[Exponential concentration of responsibilities]
    \label{lem:responsibility-concentration}
    Let $\bar\mu$ have margin $\Delta_{\min}>0$. For every observation $x_i$, define
    \[
        q_i(\sigma)
        :=
        (k-1)
        \exp\!\left(
        -\frac{\Delta_i}{2\sigma^2}
        \right).
    \]
    Then $1-r_{i j_i}^{\sigma}(\bar\mu) \leq q_i(\sigma)$ and, for every $\ell\neq j_i$,
    \begin{equation}
        r_{i\ell}^{\sigma}(\bar\mu)
        \leq
        \exp\!\left(
        -\frac{\Delta_i}{2\sigma^2}
        \right).
        \label{eq:loser-responsibility-bound}
    \end{equation}
    Consequently,
    \begin{equation}
        \left|
        r_{ij}^{\sigma}(\bar\mu)
        -
        \mathbf 1_{\{j=j_i\}}
        \right|
        \leq
        q(\sigma),
        \qquad
        q(\sigma)
        :=
        (k-1)
        \exp\!\left(
        -\frac{\Delta_{\min}}{2\sigma^2}
        \right).
        \label{eq:uniform-responsibility-bound}
    \end{equation}
    Moreover,
    \begin{equation}
        \left\|
        r_{i\cdot}^{\sigma}(\bar\mu)-e_{j_i}
        \right\|_1
        \leq
        2q_i(\sigma).
        \label{eq:responsibility-l1-bound}
    \end{equation}
\end{lemma}

The error propagates through the weighted mean with two further constants of
$\bar\mu$: the balance $\rho$ of its clusters and their radii
$B_j:=\max_i\|x_i-T_{\mathrm{Lloyd},j}(\bar\mu)\|$, with $B:=\max_jB_j$.

\begin{theorem}[Exponential approximation of the Lloyd centroid map]
    \label{thm:soft-centroid-exponential-bound}
    Let $\bar\mu$ have margin $\Delta_{\min}>0$ and nonempty clusters $S_j$, with balance
    $\rho$ and radii $B_j$ as above. Let
    \[
        q(\sigma)
        =
        (k-1)
        \exp\!\left(
        -\frac{\Delta_{\min}}{2\sigma^2}
        \right).
    \]
    For every $\sigma>0$ satisfying $q(\sigma)<1$, the soft centroid map satisfies
    \begin{equation}
        \left\|
        T_{\sigma,j}(\bar\mu)
        -
        T_{\mathrm{Lloyd},j}(\bar\mu)
        \right\|
        \leq
        \frac{B_j}{\rho}
        \frac{q(\sigma)}{1-q(\sigma)}
        \label{eq:exact-centroid-approximation-bound}
    \end{equation}
    for every $j=1,\ldots,k$. In particular, whenever $q(\sigma)\leq 1/2$,
    \begin{equation}
        \left\|
        T_{\sigma,j}(\bar\mu)
        -
        T_{\mathrm{Lloyd},j}(\bar\mu)
        \right\|
        \leq
        \frac{2B_j}{\rho}
        (k-1)
        \exp\!\left(
        -\frac{\Delta_{\min}}{2\sigma^2}
        \right).
        \label{eq:simplified-centroid-approximation-bound}
    \end{equation}
\end{theorem}

At a Lloyd fixed point $\bar\mu$, where $T_{\mathrm{Lloyd}}(\bar\mu)=\bar\mu$, the theorem
says that the soft map moves $\bar\mu$ by at most $(B_j/\rho)\,q(\sigma)/(1-q(\sigma))$ in
each centroid, an amount exponentially small in $1/\sigma^2$.

\subsection{Exact Identification with Entmax-1.5}
\label{subse:entmax-identification}

We first characterize when Entmax-$1.5$ produces an exactly one-hot output. The
condition is the separation margin of the Tsallis-$3/2$ Fenchel--Young loss
identified by \citet{blondel2020learning}; we include a short proof in the notation
of this paper.

\begin{lemma}[One-hot identification condition]
    \label{lem:entmax-one-hot-condition}
    Let $z\in\mathbb R^k$ have a unique maximizer
    \[
        j^\star
        =
        \operatorname*{arg\,max}_{1\leq j\leq k}z_j.
    \]
    Then $\operatorname{entmax}_{3/2}(z) = e_{j^\star}$ if and only if

    \begin{equation}
        z_{j^\star}
        -
        \max_{\ell\neq j^\star}z_\ell
        \geq
        2.
        \label{eq:entmax-logit-threshold}
    \end{equation}
\end{lemma}

At a configuration $\bar\mu$ with $\Delta_{\min}>0$, the logit gaps of every observation
are its margins divided by $2\sigma^2$, which gives a finite threshold.

\begin{theorem}[Finite identification of hard assignments]
    \label{thm:entmax-finite-identification}
    Let $\bar\mu$ have margin $\Delta_{\min}>0$, and define
    \begin{equation}
        \sigma_{\mathrm{id}}
        :=
        \frac{\sqrt{\Delta_{\min}}}{2}.
        \label{eq:entmax-identification-threshold}
    \end{equation}
    Then, for every $0<\sigma\leq\sigma_{\mathrm{id}}$, the Entmax-$1.5$ responsibilities evaluated at $\bar\mu$ satisfy
    \begin{equation}
        r_{ij}^{\sigma,3/2}(\bar\mu)
        =
        \mathbf 1_{\{j=j_i\}}
        \qquad
        \text{for every }i,j.
        \label{eq:exact-entmax-hard-assignments}
    \end{equation}
\end{theorem}

\begin{corollary}[Finite identification of the Lloyd update]
    \label{cor:entmax-lloyd-identification}
    Assume additionally that the hard clusters $S_j = \{i:j_i=j\}$ are nonempty. Then, for every $0<\sigma\leq\sigma_{\mathrm{id}}$,
    \begin{equation}
        T_{\sigma,j}^{(3/2)}(\bar\mu)
        =
        T_{\mathrm{Lloyd},j}(\bar\mu)
        =
        \frac{1}{|S_j|}
        \sum_{i\in S_j}x_i
        \label{eq:exact-entmax-lloyd-map}
    \end{equation}
    for all $j$. Consequently,
    \begin{equation}
        \left\|
        T_\sigma^{(3/2)}(\bar\mu)
        -
        T_{\mathrm{Lloyd}}(\bar\mu)
        \right\|
        =
        0
        \qquad
        \text{for }
        0<\sigma\leq\sigma_{\mathrm{id}}.
        \label{eq:zero-entmax-centroid-error}
    \end{equation}
\end{corollary}

Figure~\ref{fig:explain-one-update} contrasts this finite threshold with the
exponential approach of Softmax for a single observation.

\begin{figure}[t]
    \centering
    \includegraphics[width=\linewidth]{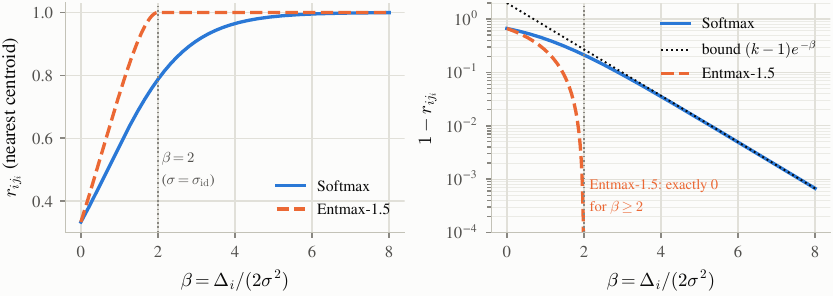}
    \caption{The responsibility of the nearest centroid for one observation with $k=3$
        whose two competitors both lie at its margin $\Delta_i$, against
        $\beta=\Delta_i/(2\sigma^2)$. Left: Softmax approaches the hard assignment only
        asymptotically, while Entmax-1.5 reaches it exactly at $\beta=2$, that is, at
        $\sigma=\sqrt{\Delta_i}/2$ (Lemma~\ref{lem:entmax-one-hot-condition}). Right: the
        same on a logarithmic scale; the Softmax error follows the bound $(k-1)e^{-\beta}$
        of Lemma~\ref{lem:responsibility-concentration}, which is tight for this
        observation, and the Entmax-1.5 error is exactly zero for every $\beta\geq2$.}
    \label{fig:explain-one-update}
\end{figure}

The threshold extends to the whole Entmax-$\alpha$ family:
Theorem~\ref{thm:entmax-alpha-limit} in Appendix~\ref{app:EntmaxHardLimit} shows that
it is sharp, gives its value
$\sqrt{(\alpha-1)\Delta_{\min}/2}$ for every $\alpha>1$, and shows that observations tied
between several centroids, for which $\Delta_{\min}=0$ and the theorem above does not
apply, are split exactly evenly among them at a finite temperature.

The threshold $\sigma_{\mathrm{id}}=\sqrt{\Delta_{\min}}/2$ is a worst case over the
observations, and $\Delta_{\min}$ is set by the single observation closest to a
Voronoi boundary. Since Theorem~\ref{thm:entmax-alpha-limit} applies observation by
observation, a positive temperature leaves unidentified only the observations
near the boundaries. For a configuration $\mu$ write $\Pi_{j\ell}$ for the bisector
hyperplane of $\mu_j$ and $\mu_\ell$, $s_{\min}(\mu):=\min_{j\neq\ell}\|\mu_j-\mu_\ell\|$, and
\[
    N_\sigma(\mu):=\#\{i:\Delta_i(\mu)<4\sigma^2\}
\]
for the number of observations whose Entmax-1.5 responsibilities are not one-hot.

\begin{proposition}[Identification observation by observation]
    \label{prop:per-observation-identification}
    Let $\sigma>0$ and let $j_i$ be a nearest centroid of $x_i$ under $\mu$, chosen
    arbitrarily in case of ties (a tied observation has $\Delta_i(\mu)=0$).
    \begin{enumerate}
        \item[(a)] $r_{i\cdot}^{\sigma,3/2}(\mu)=e_{j_i}$ if and only if
        $\Delta_i(\mu)\geq4\sigma^2$, and
        $\Delta_i(\mu)=\min_{\ell\neq j_i}2\|\mu_{j_i}-\mu_\ell\|\operatorname{dist}(x_i,\Pi_{j_i\ell})$.
        The unidentified observations are therefore exactly those at distance less than
        $2\sigma^2/\|\mu_{j_i}-\mu_\ell\|$ from a bisector $\Pi_{j_i\ell}$.
        \item[(b)] $0\leq J(\mu)-F_\sigma^{(3/2)}(\mu)\leq\tfrac{8\sigma^2}{3}\bigl(1-\tfrac{1}{\sqrt k}\bigr)N_\sigma(\mu)$.
        \item[(c)] If the clusters $S_j$ of $\mu$ are nonempty, then for every $j$ with
        $N_\sigma(\mu)<|S_j|$,
        $\|T_{\sigma,j}^{(3/2)}(\mu)-T_{\mathrm{Lloyd},j}(\mu)\|\leq B_j\,N_\sigma(\mu)/\bigl(|S_j|-N_\sigma(\mu)\bigr)$,
        with $B_j=\max_i\|x_i-T_{\mathrm{Lloyd},j}(\mu)\|$.
        \item[(d)] If $\mu$ is fixed and $x_1,\ldots,x_n$ are identically distributed with a
        law whose one-dimensional projections $\langle x,u\rangle$, $\|u\|=1$, have
        densities bounded by $\kappa$, then
        $\mathbb E[N_\sigma(\mu)]/n\leq2k(k-1)\kappa\sigma^2/s_{\min}(\mu)$.
    \end{enumerate}
\end{proposition}

At a fixed temperature, then, only a fraction $O(\sigma^2)$ of the observations
remains soft, and by (b) and (d) the expected Tsallis gap is $O(n\sigma^4)$ rather than
the $O(n\sigma^2)$ of the uniform bound. The global threshold $\sigma_{\mathrm{id}}$, by
contrast, is governed by the smallest margin, which for data with a positive
density near the boundaries decreases as $n$ grows (Section~\ref{subsec:sample-size}).
The bound in (d) assumes a configuration that does not depend on the data; for a
data-dependent configuration such as a Lloyd fixed point, (a)--(c) still hold
deterministically.

\subsection{Fixed Points and Trajectories}
\label{subsubsec:fixed-points-trajectories}

Theorem~\ref{thm:soft-centroid-exponential-bound} compares one soft update with one
Lloyd update at a fixed configuration. Because the Lloyd map is locally constant
wherever the margin is positive, this comparison propagates to fixed points and to
complete trajectories without any accumulation of error. We write
$\|\mu-\nu\|_{2,\infty}:=\max_{j}\|\mu_j-\nu_j\|$ and
$\mathfrak B_{r_\star}(\bar\mu):=\{\mu:\|\mu-\bar\mu\|_{2,\infty}\leq r_\star\}$.

\begin{lemma}[Partition stability]
    \label{lem:partition-stability}
    Let $\bar\mu$ satisfy $\Delta_{\min}(\bar\mu)>0$, let
    $R:=\max_{i,j}\|x_i-\bar\mu_j\|$, and let
    $r_\star:=\sqrt{R^2+\Delta_{\min}(\bar\mu)/4}-R$, so that
    $2r_\star(2R+r_\star)=\Delta_{\min}(\bar\mu)/2$. Then every $\mu\in\mathfrak B_{r_\star}(\bar\mu)$
    assigns every observation to the same unique nearest centroid as $\bar\mu$ and
    satisfies $\Delta_{\min}(\mu)\geq\Delta_{\min}(\bar\mu)/2$. If, in addition, the
    clusters of $\bar\mu$ are nonempty, then
    $T_{\mathrm{Lloyd}}(\mu)=T_{\mathrm{Lloyd}}(\bar\mu)$.
\end{lemma}

In words, the first consequence is that, for every sufficiently small $\sigma$,
each K-Means fixed point with a positive margin and nonempty clusters carries a
unique soft fixed point within $O(e^{-\Delta_{\min}/(4\sigma^2)})$ of it, which is a
strict local minimizer of $F_\sigma$ and attracts alternating minimization linearly;
conversely, limits of soft stationary points with a positive margin are K-Means
fixed points. The theorem makes the constants explicit.

\begin{theorem}[Soft fixed points near Lloyd fixed points]
    \label{thm:soft-fixed-point-branch}
    Let $\bar\mu$ be a fixed point of Lloyd's map with $\Delta_{\min}:=\Delta_{\min}(\bar\mu)>0$
    and nonempty clusters, let $\rho:=\min_j|S_j|/n$, and let $R$ and $r_\star$ be as in
    Lemma~\ref{lem:partition-stability}. For $\sigma>0$ set
    \[
        q'(\sigma):=(k-1)e^{-\Delta_{\min}/(4\sigma^2)},
        \qquad
        \varepsilon(\sigma):=\frac{R}{\rho}\,\frac{q'(\sigma)}{1-q'(\sigma)},
        \qquad
        L(\sigma):=\frac{k(R+r_\star)^2}{\rho\,\sigma^2}\,\frac{q'(\sigma)}{1-q'(\sigma)},
    \]
    where $q'(\sigma)$ is the quantity $q(\sigma)$ of
    Theorem~\ref{thm:soft-centroid-exponential-bound} evaluated at the margin
    $\Delta_{\min}/2$ that Lemma~\ref{lem:partition-stability} guarantees on the ball,
    and assume that $q'(\sigma)<1$, $\varepsilon(\sigma)\leq r_\star$ and $L(\sigma)<1$, which
    holds for every sufficiently small $\sigma$. Then:
    \begin{enumerate}
        \item[(a)] $T_\sigma$ has a unique fixed point $\mu_\sigma$ in $\mathfrak B_{r_\star}(\bar\mu)$,
        and $\|\mu_\sigma-\bar\mu\|_{2,\infty}\leq\varepsilon(\sigma)$;
        \item[(b)] $F_\sigma$ is strongly convex on $\mathfrak B_{r_\star}(\bar\mu)$, with
        $\nabla^2F_\sigma\succeq2n\rho\,(1-q'(\sigma))(1-L(\sigma))\,I$, so $\mu_\sigma$ is the
        unique minimizer of $F_\sigma$ on $\mathfrak B_{r_\star}(\bar\mu)$ and a strict local
        minimizer of $F_\sigma$;
        \item[(c)] alternating minimization started at any $\mu^{(0)}\in\mathfrak B_{r_\star}(\bar\mu)$
        stays in $\mathfrak B_{r_\star}(\bar\mu)$ and converges linearly to $\mu_\sigma$:
        $\|\mu^{(t)}-\mu_\sigma\|_{2,\infty}\leq L(\sigma)^t\|\mu^{(0)}-\mu_\sigma\|_{2,\infty}$.
    \end{enumerate}
    Conversely, if $\sigma_h\downarrow0$ and $\mu_h$ are stationary points of
    $F_{\sigma_h}$ converging to a configuration $\mu^\star$ with
    $\Delta_{\min}(\mu^\star)>0$ and nonempty clusters, then $\mu^\star$ is a fixed
    point of Lloyd's map.
\end{theorem}

The exponent $\Delta_{\min}/(4\sigma^2)$ in $\varepsilon(\sigma)$ reflects the
worst-case halving of the margin in Lemma~\ref{lem:partition-stability}. Once
$\|\mu_\sigma-\bar\mu\|_{2,\infty}\leq\varepsilon(\sigma)$ is known, the margin of
$\mu_\sigma$ is at least $\Delta_{\min}-2\varepsilon(\sigma)(2R+\varepsilon(\sigma))$, and
applying Theorem~\ref{thm:soft-centroid-exponential-bound} at $\mu_\sigma$ improves
the exponent to $\Delta_{\min}(1-o(1))/(2\sigma^2)$, and the linear rate $L(\sigma)$ of
part (c) tends to zero as $\sigma\to0$.

The same mechanism controls whole trajectories. In words: as long as Lloyd's
iterates keep a positive margin $\delta$ and nonempty clusters, alternating
minimization started at the same configuration induces the same partition at
every iteration and stays within $O(e^{-\delta/(4\sigma^2)})$ of Lloyd's iterates.

\begin{theorem}[Tracking of Lloyd's trajectory]
    \label{thm:trajectory-tracking}
    Let $\mu^{(0)}$ be an initial configuration and let
    $\mu_{\mathrm L}^{(t)}=T_{\mathrm{Lloyd}}^t(\mu^{(0)})$, $0\leq t\leq T$, be Lloyd's
    iterates, with $T\leq\infty$. Assume that
    $\delta:=\inf_{0\leq t<T}\Delta_{\min}(\mu_{\mathrm L}^{(t)})>0$ and that
    $\rho:=\inf_{0\leq t<T}\min_j|S_j(\mu_{\mathrm L}^{(t)})|/n>0$. Let
    $R:=\sup_{0\leq t\leq T}\max_{i,j}\|x_i-\mu_{\mathrm L,j}^{(t)}\|$,
    $r_\star:=\sqrt{R^2+\delta/4}-R$,
    $q_\delta(\sigma):=(k-1)e^{-\delta/(4\sigma^2)}$ (the quantity $q'$ of
    Theorem~\ref{thm:soft-fixed-point-branch} with $\delta$ in place of $\Delta_{\min}$) and
    $\varepsilon(\sigma):=(R/\rho)\,q_\delta(\sigma)/(1-q_\delta(\sigma))$. If
    $q_\delta(\sigma)<1$ and $\varepsilon(\sigma)\leq r_\star$, the iterates
    $\mu_{\mathrm S}^{(t)}$ of alternating minimization started at the same
    $\mu^{(0)}$ satisfy, for every $0\leq t\leq T$,
    \[
        \left\|\mu_{\mathrm S}^{(t)}-\mu_{\mathrm L}^{(t)}\right\|_{2,\infty}
        \leq
        \varepsilon(\sigma),
    \]
    and $\mu_{\mathrm S}^{(t)}$ induces the same hard partition as $\mu_{\mathrm L}^{(t)}$
    for every $0\leq t<T$.
\end{theorem}

The error does not accumulate along the trajectory, because the Lloyd map
absorbs perturbations smaller than $r_\star$: $\varepsilon(\sigma)$ bounds the
discrepancy at every iteration, not only at the first. Lloyd's algorithm reaches a
fixed point after finitely many steps, so if that fixed point also satisfies the
margin and balance conditions, the hypotheses hold with $T=\infty$. If $\sigma$
also satisfies the hypotheses of Theorem~\ref{thm:soft-fixed-point-branch} at
Lloyd's limit $\bar\mu$, the soft trajectory eventually enters the ball around
$\bar\mu$ and then converges linearly to the soft fixed point $\mu_\sigma$; no further
condition is needed, because $r_\star$ decreases in $R$ and increases in the margin,
so the radius of Theorem~\ref{thm:trajectory-tracking} is never larger than that of
Theorem~\ref{thm:soft-fixed-point-branch}. Under these conditions the soft and hard
trajectories induce identical partitions at every iteration and their centroids
differ by $O(e^{-\delta/(4\sigma^2)})$ throughout, a deliberately crude rate:
repeating the one-step bound of Theorem~\ref{thm:soft-centroid-exponential-bound}
at the soft iterates improves the exponent to $\delta(1-o(1))/(2\sigma^2)$. The conditions fail
when Lloyd's path passes close to a configuration with a tied observation, which
is exactly where the two algorithms can separate (Section~\ref{subsec:optimization-results}).
Gradient descent with a common step size is not covered: its iterates are not
images of $T_\sigma$, so nothing forces them to track Lloyd's, although they often
do (Section~\ref{subsec:trajectory-results}).

Applied at every iterate, Corollary~\ref{cor:entmax-lloyd-identification} turns into an exact statement about
trajectories, the Entmax counterpart of Theorem~\ref{thm:trajectory-tracking}.

\begin{corollary}[Identical trajectories]
    \label{cor:entmax-trajectory}
    Let $\mu_{\mathrm L}^{(t)}$, $0\leq t\leq T$, be Lloyd's iterates from
    $\mu^{(0)}$, and assume that every $\mu_{\mathrm L}^{(t)}$ with $t<T$ has nonempty
    clusters and $\Delta_{\min}(\mu_{\mathrm L}^{(t)})\geq4\sigma^2$. Then the iterates of
    alternating Entmax-$1.5$ started at $\mu^{(0)}$ coincide with Lloyd's:
    $\mu_{\mathrm E}^{(t)}=\mu_{\mathrm L}^{(t)}$ for every $0\leq t\leq T$.
\end{corollary}

Figure~\ref{fig:explain-trajectories} illustrates Theorem~\ref{thm:trajectory-tracking}
and Corollary~\ref{cor:entmax-trajectory} on a small example.

\begin{figure}[t]
    \centering
    \includegraphics[width=\linewidth]{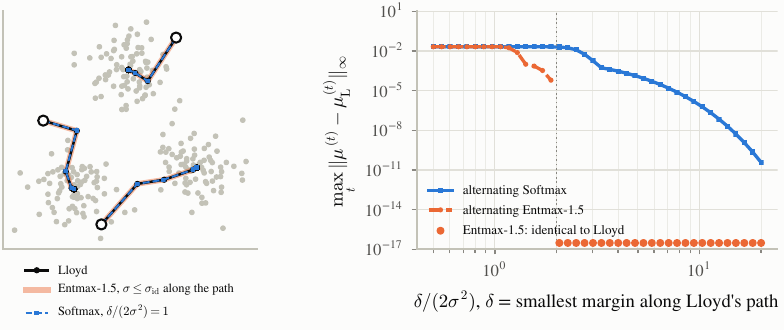}
    \caption{Trajectories from one initialization (open circles) on three Gaussian
        clusters ($n=240$, $k=3$), with $\delta$ the smallest margin along Lloyd's path.
        Left: Lloyd's iterates (black), alternating Entmax-1.5 at a temperature with
        $\sigma\leq\sigma_{\mathrm{id}}$ at every Lloyd iterate, whose iterates coincide
        with Lloyd's (Corollary~\ref{cor:entmax-trajectory}), and alternating Softmax at
        $\delta/(2\sigma^2)=1$, outside the hypotheses of
        Theorem~\ref{thm:trajectory-tracking}, which still follows Lloyd's path closely.
        Right: the largest discrepancy from Lloyd's iterates over the whole trajectory,
        against $\delta/(2\sigma^2)$. It decays exponentially for Softmax and is exactly
        zero for Entmax-1.5 from $\delta/(2\sigma^2)=2$ on (markers at the bottom).}
    \label{fig:explain-trajectories}
\end{figure}

\section{Experimental Validation}
\label{sec:experiments}

This section evaluates separately the theoretical approximation results and the behavior of the corresponding optimization algorithms. This separation is necessary because the results of Sections~\ref{sec:objective}--\ref{sec:maps} concern different mathematical objects.

After the protocol (Section~\ref{subsec:evaluation-setting}) and the datasets
(Section~\ref{subsec:datasets}), the evaluation follows the levels of
Table~\ref{tab:summary}. At fixed reference configurations we first evaluate the
\emph{assignment maps}: the exponential bounds for Softmax
(Theorem~\ref{thm:soft-centroid-exponential-bound}) and the finite-identification
threshold for Entmax (Theorems~\ref{thm:entmax-finite-identification}
and~\ref{thm:entmax-alpha-limit}) in Section~\ref{subsec:convresults}, and how
that threshold depends on the sample size in Section~\ref{subsec:sample-size}.
Still at fixed configurations, Section~\ref{subsec:value-function-results}
evaluates the \emph{value functions}, testing the uniform bound of
Theorem~\ref{thm:uniform-gamma-convergence} and its Tsallis counterpart in
Proposition~\ref{prop:tsallis-uniform-approximation}. We then run the
\emph{algorithms}, Lloyd's algorithm, alternating Softmax, gradient descent on
$F_\sigma$ and alternating Entmax, from paired initializations:
Section~\ref{subsec:optimization-results} tests the descent and convergence
statements and compares the fixed points reached, and
Section~\ref{subsec:trajectory-results} tests the fixed-point and trajectory
theorems of Section~\ref{subsubsec:fixed-points-trajectories}, which go beyond a
single update of the centroid map. Section~\ref{subsec:annealing-results} measures
what the theory leaves open: which fixed point an annealing schedule
reaches.\footnote{The code that reproduces every experiment is available at \codeurl.}

\subsection{Evaluation Protocol}
\label{subsec:evaluation-setting}

\paragraph{Reference configurations.}
\label{subsubsec:reference-configurations}
The map and value statements of Sections~\ref{sec:objective} and~\ref{sec:maps}
concern a fixed configuration, and the natural one is a K-Means solution. For each
dataset we draw $M$ k-means++ initializations \citep{arthur2006k} once and run Lloyd's
algorithm to convergence from each, which gives reference configurations
$\bar\mu^{(m)}$, $m=1,\ldots,M$, with their assignments $j_i^{(m)}$, clusters
$S_j^{(m)}$, margins $\Delta_i^{(m)}$ and smallest margin
\begin{equation}
    \Delta_{\min}^{(m)}
    =
    \min_{1\leq i\leq n}
    \Delta_i^{(m)}.
    \label{eq:experimental-minimum-margin}
\end{equation}
A configuration is admissible when $\Delta_{\min}^{(m)}>0$ and no cluster is empty, the
hypotheses of the bounds. Inadmissible runs are counted and reported, and excluded
only from the experiments whose hypotheses they violate.

\paragraph{Temperatures.}
\label{subsubsec:margin-normalized-temperature}
A raw temperature is not comparable across datasets. The bounds depend on the
dimensionless ratio
\begin{equation}
    \beta^{(m)}(\sigma)
    =
    \frac{
        \Delta_{\min}^{(m)}
    }{
        2\sigma^2
    }
    \label{eq:normalized-inverse-temperature}
\end{equation}
of margin to temperature, in which the Softmax factor of
Lemma~\ref{lem:responsibility-concentration} is $(k-1)e^{-\beta}$ and the Entmax
identification condition of Theorem~\ref{thm:entmax-finite-identification} is
$\beta\geq2$. Every reference configuration is therefore probed at the same $L=24$
values $\beta_1<\dots<\beta_L$, spaced geometrically between $0.25$ and $64$, through
the temperatures
\begin{equation}
    \sigma_\ell^{(m)}
    =
    \sqrt{
        \frac{
            \Delta_{\min}^{(m)}
        }{
            2\beta_\ell
        }
    }.
    \label{eq:margin-adapted-temperature-grid}
\end{equation}
Four temperature scales appear in the experiments, and we keep them apart. The
map and value experiments use $\beta$, normalized by the margin of each reference
configuration. Experiments that run several methods from many initializations of
a dataset need one temperature per dataset; they use
$\widetilde\beta=\widetilde\Delta_{\min}/(2\sigma^2)$, normalized by the median
$\widetilde\Delta_{\min}$ of the margins of the dataset's admissible reference
configurations (Sections~\ref{subsec:optimization-results}
and~\ref{subsec:trajectory-results}). The annealing schedules are expressed
relative to the largest eigenvalue $\lambda_{\max}$ of the data covariance, which sets the
critical temperature (Section~\ref{subsec:annealing-results}), and the learned
representations of Section~\ref{sec:extension} use $\sigma^2$ relative to the mean
squared distance to the nearest centroid, because their scale changes during
training. In the theory, $\tau=2\sigma^2$ is the temperature of the entropy penalty
and $\sigma$ the corresponding Gaussian bandwidth.

\paragraph{Map and value errors.}
\label{subsubsec:softmax-validation-metrics}
\label{subsubsec:entmax-validation-metrics}
\label{subsubsec:value-function-metrics}
At each admissible reference configuration and temperature we measure the uniform
error of the Softmax responsibilities against the hard assignments
$r_{ij}^{0,(m)}=\mathbf 1_{\{j=j_i^{(m)}\}}$ and the error of one soft update against
one Lloyd update, in the norm $\|\nu\|_{2,1}=\sum_j\|\nu_j\|$,
\begin{equation}
    E_r^{(m)}(\sigma)
    =
    \max_{1\leq i\leq n}
    \max_{1\leq j\leq k}
    \left|
    r_{ij}^{\sigma}
    \left(
    \bar\mu^{(m)}
    \right)
    -
    r_{ij}^{0,(m)}
    \right|,
    \label{eq:softmax-responsibility-error}
\end{equation}
\begin{equation}
    E_T^{(m)}(\sigma)
    =
    \left\|
    T_\sigma
    \left(
    \bar\mu^{(m)}
    \right)
    -
    T_{\mathrm{Lloyd}}
    \left(
    \bar\mu^{(m)}
    \right)
    \right\|_{2,1}.
    \label{eq:softmax-centroid-map-error}
\end{equation}
and compare them with the bounds of Lemma~\ref{lem:responsibility-concentration} and
Theorem~\ref{thm:soft-centroid-exponential-bound},
\begin{equation}
    E_r^{(m)}(\sigma)
    \leq
    (k-1)
    \exp\!\left(
    -\frac{
        \Delta_{\min}^{(m)}
    }{
        2\sigma^2
    }
    \right)
    =
    (k-1)e^{-\beta^{(m)}(\sigma)},
    \label{eq:experimental-responsibility-bound}
\end{equation}
\begin{equation}
    E_T^{(m)}(\sigma)
    \leq
    \frac{
        kB^{(m)}
    }{
        \rho^{(m)}
    }
    \frac{
        q^{(m)}(\sigma)
    }{
        1-q^{(m)}(\sigma)
    },
    \qquad q^{(m)}(\sigma)<1,
    \label{eq:experimental-centroid-map-bound}
\end{equation}
where $\rho^{(m)}$ is the balance of the clusters of $\bar\mu^{(m)}$, $B^{(m)}$ the largest
radius $\max_{j,i}\|x_i-T_{\mathrm{Lloyd},j}(\bar\mu^{(m)})\|$, and
$q^{(m)}(\sigma)=(k-1)e^{-\beta^{(m)}(\sigma)}$. The Entmax-1.5 errors
$E_{r,\mathrm{ent}}^{(m)}$ and $E_{T,\mathrm{ent}}^{(m)}$ are defined in the same way with
$r^{\sigma,3/2}$ and $T_\sigma^{(3/2)}$. They are predicted to vanish exactly for
$\beta^{(m)}\geq2$, so the diagnostic is a step rather than a slope, and
identification is declared when $E_{r,\mathrm{ent}}^{(m)}\leq\varepsilon_{\mathrm{id}}=10^{-9}$.
For the objectives we measure the value gaps
\begin{equation}
    G_\sigma^{(m)}
    :=
    J\!\left(\bar\mu^{(m)}\right)
    -
    F_\sigma\!\left(\bar\mu^{(m)}\right),
    \qquad
    G_{\sigma,\mathrm{ent}}^{(m)}
    :=
    J\!\left(\bar\mu^{(m)}\right)
    -
    F_\sigma^{(3/2)}\!\left(\bar\mu^{(m)}\right),
    \label{eq:value-function-gap}
\end{equation}
whose uniform bounds (Theorem~\ref{thm:uniform-gamma-convergence} and
Proposition~\ref{prop:tsallis-uniform-approximation}) are attained only at degenerate
configurations. The informative predictions are those of
Proposition~\ref{prop:value-gap-margin}: $\log(G_\sigma^{(m)}/\sigma^2)$ is asymptotically
linear in $\beta$ with slope $-1$, $G_{\sigma,\mathrm{ent}}^{(m)}$ vanishes for $\beta\geq2$, and
$G_{\sigma,\mathrm{ent}}^{(m)}<G_\sigma^{(m)}$ in every run
(Proposition~\ref{prop:shannon-tsallis-ordering}). A bound counts as satisfied up to
a declared floating-point tolerance, and values below the precision floor are
treated as censored (Appendix~\ref{app:protocol}).

\paragraph{Algorithms.}
\label{subsubsec:optimization-algorithms}
The optimization experiments run four procedures from the same initialization
within each paired run: Lloyd's algorithm, alternating Softmax, alternating
Entmax-1.5 with an explicit empty-cluster policy (Algorithm~\ref{alg:rbf-gradient}),
and gradient descent on $F_\sigma$ with one step size common to all centroids. Its step
is chosen by Armijo backtracking from $1/\bigl(2\max_jm_j^\sigma(\mu^{(t)})\bigr)$, the
inverse of the bound
$\nabla^2F_\sigma(\mu^{(t)})\preceq2\operatorname{diag}\bigl(m^\sigma(\mu^{(t)})\bigr)\otimes I_d$ on the
local curvature, halved until
$F_\sigma\bigl(\mu^{(t)}-\eta_t\nabla F_\sigma(\mu^{(t)})\bigr)\leq
F_\sigma(\mu^{(t)})-\tfrac12\eta_t\|\nabla F_\sigma(\mu^{(t)})\|^2$; even its first trial coincides
with the block step $1/(2m_j^\sigma)$ only for the heaviest cluster. Every method
stops when
\begin{equation}
    \frac{
        \|\mu^{(t+1)}-\mu^{(t)}\|_F
    }{
        \max\{1,\|\mu^{(t)}\|_F\}
    }
    \leq
    \varepsilon_\mu,
    \label{eq:centroid-stopping-criterion}
\end{equation}
or after $T$ iterations, and for gradient descent we also monitor
\begin{equation}
    \left\|
    \nabla F_\sigma(\mu^{(t)})
    \right\|_F.
    \label{eq:gradient-norm-diagnostic}
\end{equation}
Terminal configurations of two methods are compared after aligning their labels,
\begin{equation}
    d_{\mathrm{match}}(\mu,\nu)
    =
    \min_{\pi\in\mathfrak S_k}
    \left[
        \frac{1}{k}
        \sum_{j=1}^k
        \left\|
        \mu_j-\nu_{\pi(j)}
        \right\|_2^2
        \right]^{1/2},
    \label{eq:permutation-invariant-centroid-distance}
\end{equation}
where $\mathfrak S_k$ is the permutation group on $k$ elements.

\paragraph{Statistics.}
\label{subsubsec:paired-initialization}
The $M$ initializations are shared by every method and every temperature, so that
differences between methods are not due to the initialization. The tables report
medians and interquartile ranges over the admissible runs, and percentages for
events. With at least $196$ admissible runs per dataset, a percentage has a 95\%
Wilson interval of half-width at most $6.9$ points, and at most $1.7$ points when
pooled over the $3{,}196$ runs of all datasets; we state differences of percentages
only when they exceed these widths. Appendix~\ref{app:statistics} gives confidence
intervals for the median slopes and sign tests for the paired comparisons,
Appendix~\ref{app:protocol} the tolerances and further details, and
Algorithm~\ref{alg:evaluation} summarizes the protocol.

\subsection{Benchmark Datasets}
\label{subsec:datasets}

The evaluation uses four dataset families, sixteen datasets in total. The
\emph{geometry family} spans distinct topological regimes (Euclidean
separation, nonlinear manifolds, polar distortions and radial structures) and
probes whether the approximation results hold irrespective of the ambient
geometry. The \emph{margin-controlled family} varies a single quantity, the
gap between three isotropic blobs, and with it the assignment margin
$\Delta_{\min}$, which is the constant governing every bound in
Sections~\ref{subse:centroid-approximation}
and~\ref{subse:entmax-identification}; testing the bounds as a function of
their governing constant is not possible with fixed geometries alone. The
\emph{high-dimensional family} places three isotropic blobs with unit standard
deviation in $d=5$, $10$ and $50$ dimensions, with centers drawn from
$\mathcal N(0,6^2I_d)$, since the bounds depend on the data only through $\Delta_{\min}$,
$\rho$ and $B$ and not on $d$. Finally, four standard \emph{real datasets}
with $d$ between $4$ and $64$ and up to $k=10$ clusters test the predictions
outside synthetic generators.

\begin{table}[ht]
    \centering
    \small
    \begin{tabular}{lll}
        \toprule
        Dataset & Regime & Parameters \\
        \midrule
        Gaussian blobs & Separated Euclidean clusters & std $0.6$ \\
        Two moons      & Nonlinear manifold           & noise $0.05$ \\
        Spiral         & Nonconvex polar structure    & radius $3$, noise $0.05$ \\
        Circles        & Radial separation            & factor $0.5$, noise $0.05$ \\
        Separation     & Isotropic blobs, set gap     & std $0.35$, $\delta\in\{0.5,1,2,4,8\}$ \\
        Blobs, $d$-dim.& Isotropic blobs in $d$ dim.  & std $1$, centers $\mathcal N(0,6^2I)$, $d\in\{5,10,50\}$ \\
        Iris, Wine, Breast cancer, Digits & Real data & see Table~\ref{tab:constants} \\
        \bottomrule
    \end{tabular}
    \caption{Datasets. The synthetic datasets have $n=300$ points; $n$, $d$ and $k$
        for every dataset are listed in Table~\ref{tab:constants} (Appendix~\ref{app:per-dataset}). All datasets except
        the separation sweep are standardized to zero mean and unit variance per
        coordinate; the separation sweep is left unscaled so that the gap $\delta$
        controls the margin directly.}
    \label{tab:datasets}
\end{table}

The number of centroids is matched to each dataset rather than fixed
globally: the map-level results
(Section~\ref{sec:maps})
hold for \emph{any} configuration with $\Delta_{\min}>0$, so the geometry
family is used to stress the governing constants, not as a claim that
K-Means is the correct model for non-convex manifolds. For the
separation sweep, the inter-blob gap $\delta$ is increased so that
$\Delta_{\min}$ ranges from near-degenerate to well-separated, allowing the
predicted slope $-\Delta_{\min}$ (Section~\ref{subsubsec:softmax-validation-metrics})
and threshold $\sigma_{\mathrm{id}}=\sqrt{\Delta_{\min}}/2$
(Section~\ref{subsubsec:entmax-validation-metrics}) to be tracked directly.

For every dataset, Table~\ref{tab:constants} reports the governing constants
$\Delta_{\min}$, the balance ratio $\rho=\min_j|S_j|/n$, and the radius
$B=\max_{j,i}\|x_i-T_{\mathrm{Lloyd},j}(\bar\mu)\|$, together with the number
of initializations that satisfy the admissibility condition
($\Delta_{\min}>0$ and no empty cluster). Runs are never silently removed:
their frequency is reported and they are excluded only from the experiments
whose hypotheses they fail. Each configuration is repeated over $M=200$
paired initializations with $T=150$ iterations. Admissibility fails in only
four of the $3{,}200$ reference runs, all on the 50-dimensional blobs. The
median margin spans more than four orders of magnitude across datasets, from
$3.9\times10^{-3}$ to $106$, and the smallest margins observed are far smaller
still, which is what makes a margin-normalized analysis necessary.

\subsection{Rate Analysis: Exponential Concentration and Finite Identification}
\label{subsec:convresults}

Lemma~\ref{lem:responsibility-concentration} and
Theorem~\ref{thm:soft-centroid-exponential-bound} predict a decay exponential in
$1/\sigma^2$, not polynomial in $\sigma$, so we fit $\log E_r$ and $\log E_T$ linearly
against $1/(2\sigma^2)$ rather than against $\log\sigma$, which would impose a
polynomial ansatz. The predicted slope, $-\Delta_{\min}^{(m)}$, differs between
reference configurations, so the fit is made separately for each run, over the
regime $\beta\geq2$ and on the values above the declared tolerance, and its ratio to
$\Delta_{\min}^{(m)}$, the slope on the $\beta$ axis, is predicted to equal $-1$. For
$E_r$ this slope is close to $-1$ by construction, because the observation attaining
$\Delta_{\min}$ has error about $e^{-\beta}/(1+e^{-\beta})$ at every temperature: it checks the
implementation and the constant rather than the theory. The informative slopes are
those of $E_T$ and of the value gaps, which aggregate observations with different
margins. For the same reason, the collapse of all datasets onto one line on the
$\beta$ axis (Figure~\ref{fig:wiseconvergence}) shows that nothing but the margin
enters the error, rather than testing independence from the geometry.

\begin{figure}[ht]
    \centering
    \begin{subfigure}{0.45\textwidth}
        \centering
        \includegraphics[width=\linewidth]{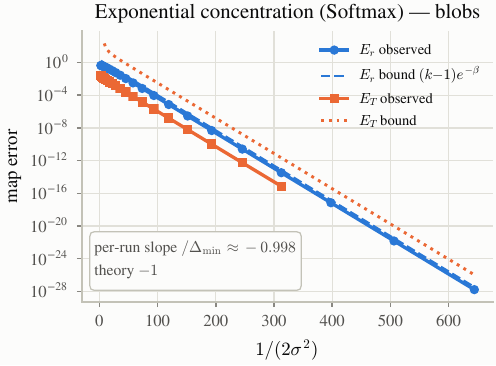}
        \caption{Exponential concentration: $\log E_r$, $\log E_T$ vs.\ $1/(2\sigma^2)$}
        \label{fig:logconvergence}
    \end{subfigure}
    \hfill
    \begin{subfigure}{0.45\textwidth}
        \centering
        \includegraphics[width=\linewidth]{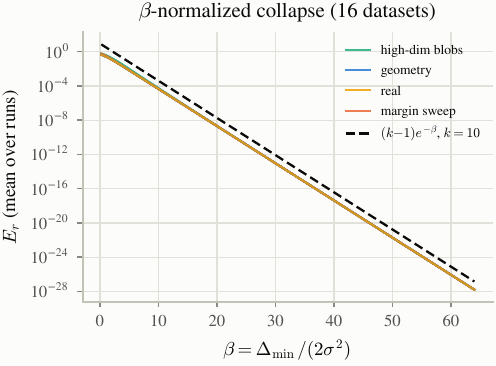}
        \caption{$\beta$-normalized collapse across datasets}
        \label{fig:wiseconvergence}
    \end{subfigure}
    \caption{Softmax rate diagnostics on the theoretically dictated axes.
        (\subref{fig:logconvergence})~Gaussian blobs: responsibility and
        centroid-map errors, averaged over the admissible reference runs at each
        $\beta$, on a $\log$ scale against $1/(2\sigma^2)=\beta/\widetilde\Delta_{\min}$,
        with $\widetilde\Delta_{\min}$ the median margin over runs, together with the
        analytic bounds
        \eqref{eq:experimental-responsibility-bound}--\eqref{eq:experimental-centroid-map-bound}.
        The theory predicts straight lines of slope $-\Delta_{\min}$; the panel reports
        the median over runs of the per-run fitted slope divided by that run's
        $\Delta_{\min}$, which is predicted to equal $-1$.
        At the coldest temperatures $E_T$ is exactly zero in floating point and
        cannot be drawn on the logarithmic axis.
        (\subref{fig:wiseconvergence})~Mean responsibility error against the
        margin-normalized inverse temperature $\beta=\Delta_{\min}/(2\sigma^2)$ for all
        16 datasets, one line per dataset, colored by family; the dashed line is the
        bound $(k-1)e^{-\beta}$ for the largest $k$ used ($k=10$). All curves collapse
        onto a single line of slope $-1$, as the margin normalization implies.}
    \label{fig:comparacion}
\end{figure}

For Entmax-$1.5$ the prediction is a step, not a decay:
Theorem~\ref{thm:entmax-finite-identification} and
Corollary~\ref{cor:entmax-lloyd-identification} state that both errors vanish
exactly for $\beta\geq2$. We therefore report the fraction of runs identified at each
$\beta$, the number of violations below the threshold, and the empty Entmax clusters
outside the identification regime, rather than a fitted exponent.

\begin{figure}[ht]
    \centering
    \includegraphics[width=0.47\linewidth]{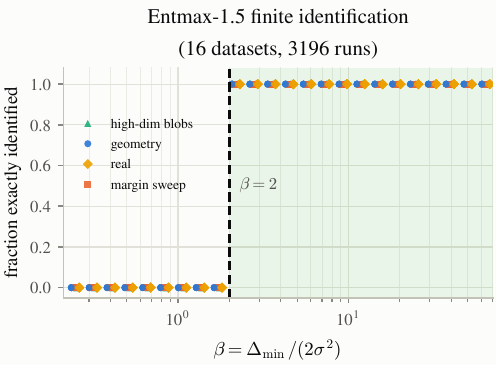}
    \caption{Entmax-$1.5$ finite identification: fraction of admissible reference
        runs whose Entmax responsibilities are exactly one-hot (error at most
        $\varepsilon_{\mathrm{id}}$), against $\beta$, for all 16 datasets
        (3{,}196 runs), colored by family. The responsibility and centroid-map
        errors are predicted to vanish exactly for $\beta\geq2$, i.e.\
        $\sigma\leq\sigma_{\mathrm{id}}=\sqrt{\Delta_{\min}}/2$, and, by the
        sharpness in Remark~\ref{rem:entmax-alpha-consequences}, not for any
        $\beta<2$: a step rather than a decay. The dashed line marks $\beta=2$ and
        the shaded region the identification regime. The fraction is observed only
        on the $\beta$ grid, so it is drawn as markers; families are offset
        horizontally by a few percent so that coincident markers remain visible.}
    \label{fig:entmax-threshold}
\end{figure}

At high temperature the centroid-map error and its bound reach a plateau, the
saturation of the factor $q(\sigma)/(1-q(\sigma))$ in
\eqref{eq:experimental-centroid-map-bound}, before the linear descent.

\paragraph{Results.} Table~\ref{tab:rates} (Appendix~\ref{app:per-dataset}) summarizes the map-level
experiments. On all sixteen datasets the per-run slope of $\log E_r$ against
$\beta$ has a median between $-1.017$ and $-0.998$, with an interquartile range
below $10^{-3}$: the rate $e^{-\Delta_{\min}/2\sigma^2}$ is recovered to within
two percent, and once the temperature is normalized by the margin it is the same
on every geometry, dimension and real dataset
(Figure~\ref{fig:wiseconvergence}). The responsibility error is governed by the
single observation with the smallest margin, which is why its slope is sharp.
The centroid-map error is instead a sum of such terms, one for each observation
near a Voronoi boundary, with rates $\Delta_i/\Delta_{\min}\geq1$. At finite
$\beta$ its effective slope is a weighted average of these rates, hence steeper
than $-1$, and it approaches $-1$ only as $\beta\to\infty$. This is visible on
the well-separated datasets (separation $\delta=4$ and $8$ and the
high-dimensional blobs, with slopes between $-1.25$ and $-1.14$), where the
declared tolerance censors $E_T$ before the asymptotic regime is reached; on the
other datasets the slope of $E_T$ lies within $0.02$ of $-1$. Every analytic
bound holds in every record; the centroid-map bound is infinite, and holds
trivially, at the records with $q(\sigma)\geq1$, that is $\beta\leq\log(k-1)$. The
Entmax-$1.5$ responsibilities are exactly
one-hot at every grid point with $\beta\geq2$ and at none with $\beta<2$, as the
sharpness in Remark~\ref{rem:entmax-alpha-consequences} predicts, so the
empirical identification threshold coincides with $\sigma_{\mathrm{id}}$ up to
the resolution of the grid (Figure~\ref{fig:entmax-threshold}). These checks
confirm deterministic inequalities, so they validate the implementation and the
sharpness of the constants rather than test a statistical hypothesis; the
informative quantities are the rates and the threshold, not the pass rate.

The two map-level bounds differ in tightness. Near $\beta=8$ the responsibility
bound $(k-1)e^{-\beta}$ is within a factor of $1$ to $9$ of the observed error, because
the observation closest to a Voronoi boundary attains it up to the factor $k-1$.
The centroid-map bound of Theorem~\ref{thm:soft-centroid-exponential-bound} is
looser, by a factor of $15$ to $4.6\times10^3$ on fifteen datasets, because it charges
every observation with the smallest margin, whereas only the few observations
near a boundary remain soft (Proposition~\ref{prop:per-observation-identification}).
It falls below the radius $B$ of the data, and thus becomes informative, from
$\beta\approx1.7$ to $3.6$ on these datasets. On digits, whose smallest clusters contain
a handful of points (median $\rho=2.2\times10^{-3}$), the factor $1/\rho$ makes it
vacuous until $\beta\approx12$, and at $\beta=8$ it exceeds the observed error by a
factor of $3\times10^7$. The rate, which the bound predicts correctly, is unaffected.

The temperature grid resolves the Entmax threshold only to a factor of about
$1.27$ in $\beta$, and Entmax-$1.5$ is one member of a family. We therefore evaluate
the identification threshold $\sigma_{\mathrm{id}}(\alpha)=\sqrt{(\alpha-1)\Delta_{\min}/2}$ of
Theorem~\ref{thm:entmax-alpha-limit}(c) directly, for $\alpha\in\{1.25,1.5,2,3\}$, at the
$415$ distinct admissible reference configurations of the sixteen datasets, at
$\sigma/\sigma_{\mathrm{id}}(\alpha)\in\{0.5,0.9,0.99,0.999,1.001,1.01,1.1,2\}$, with an exact
test of the support size. The prediction holds in every case: every configuration
is identified at $0.999\,\sigma_{\mathrm{id}}(\alpha)$ and none at $1.001\,\sigma_{\mathrm{id}}(\alpha)$, so
the threshold is sharp to $0.1\%$ for every $\alpha$, and, observation by observation,
$r_{i\cdot}$ is one-hot exactly when $\Delta_i\geq2\sigma^2/(\alpha-1)$. No cluster is left
without mass at any of these temperatures. Above the threshold only a few
observations are soft: $0.2\%$ on average at $1.001\,\sigma_{\mathrm{id}}$ and $2.2\%$ at
twice it. The choice $\alpha=3/2$ is therefore not essential: a larger $\alpha$
identifies the assignments at a higher temperature.

\subsection{Margin and Sample Size}
\label{subsec:sample-size}

The identification threshold $\sigma_{\mathrm{id}}$ is set by the margin, which depends
strongly on the data (Section~\ref{subsec:datasets}): its minimum over reference
runs goes down to about $2\times10^{-5}$ on the spiral
and on the ten-dimensional blobs (Table~\ref{tab:constants}). By
Proposition~\ref{prop:per-observation-identification}(a) the margin of an observation
is proportional to its distance to the nearest bisector, so $\Delta_{\min}$ is set by
the observation closest to a Voronoi boundary. For data with a positive density near the
boundaries it therefore shrinks as $n$ grows, and the identification threshold
$\sigma_{\mathrm{id}}=\sqrt{\Delta_{\min}}/2$ shrinks with it. A temperature that
identifies every assignment of a large dataset may thus have to be very small,
even though the observations far from every boundary are identified at much
larger temperatures: by Proposition~\ref{prop:per-observation-identification} the
threshold of each observation is $\sqrt{\Delta_i}/2$, and at a fixed temperature only a
fraction $O(\sigma^2)$ of the observations remains soft.

Figure~\ref{fig:sample-size} measures both effects. We sample two Gaussian
mixtures with unit variance, $k=3$ components on a triangle of side $4$ in the
plane and $k=5$ components at pairwise distance $4$ in $\mathbb R^{10}$, at sizes from
$10^2$ to $10^6$ ($50$ samples per size, $20$ above $10^5$), and MNIST projected on
its first $50$ principal components, subsampled from $10^2$ to $7\times10^4$ with
$k=10$. For each sample we compute the margin at the population configuration
$\mu^\star$, the Lloyd fixed point of a reference sample of $10^6$ points or of all of
MNIST, and at the sample's own Lloyd fixed point reached from $\mu^\star$. The
median of $\Delta_{\min}$ decays as $n^{-1}$: the least-squares slopes of $\log\Delta_{\min}$
against $\log n$ are $-1.03$ and $-1.00$ for the two mixtures at $\mu^\star$, $-1.05$
and $-1.04$ at the sample's fixed point, and $-1.08$ for MNIST at its fixed point.
At $\mu^\star$, MNIST gives $-0.75$, a noisier estimate because its subsamples are
drawn from a finite pool. The identification threshold therefore decays as
$n^{-1/2}$, from $\sigma_{\mathrm{id}}\approx0.3$ at $n=10^2$ to $0.0025$ at $n=10^6$ for the planar
mixture. The fraction of observations that are not identified at a fixed
temperature is, in contrast, independent of $n$ (right panel), between $0.3\%$ and
$6\%$ for the two temperatures shown, and below the bound of
Proposition~\ref{prop:per-observation-identification}(d), which for the mixtures we
evaluate with the exact density bound $\kappa=1/\sqrt{2\pi}$. The bound is loose by a
factor of $12$ to $24$, the price of its union bound over bisectors and of the
worst-case density. In practice, Entmax-$1.5$ at a fixed temperature identifies the
same fraction of a large dataset as of a small one, but exact identification of
every observation requires a temperature that vanishes as $n^{-1/2}$.

\begin{figure}[ht]
    \centering
    \includegraphics[width=\linewidth]{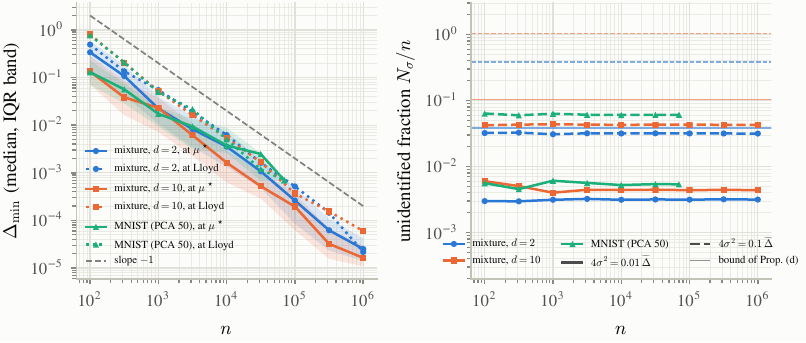}
    \caption{Dependence on the sample size. Left: median and interquartile band of
        $\Delta_{\min}$ against $n$, at the population configuration $\mu^\star$ (solid) and at
        the sample's Lloyd fixed point (dotted); the dashed line has slope $-1$. Right:
        mean fraction $N_\sigma/n$ of observations not identified by Entmax-$1.5$ at
        $\mu^\star$, for $4\sigma^2$ equal to $0.01$ and $0.1$ times the median observation margin
        $\widetilde\Delta$; thin horizontal lines: the bound of
        Proposition~\ref{prop:per-observation-identification}(d) for the mixtures.}
    \label{fig:sample-size}
\end{figure}

\subsection{Value-Function \texorpdfstring{$\Gamma$}{Gamma}-Convergence}
\label{subsec:value-function-results}

We test directly the objective-level statements of
Theorem~\ref{thm:uniform-gamma-convergence} and
Proposition~\ref{prop:tsallis-uniform-approximation}, which carry the
variational part of the paper and are not captured by the map-level metrics
above. Using the gaps
$G_\sigma=J-F_\sigma$ and $G_{\sigma,\mathrm{ent}}=J-F_\sigma^{(3/2)}$ of
\eqref{eq:value-function-gap}, both quantities are plotted against
$\sigma^2$ on logarithmic axes, together with the uniform ceilings
$2n\sigma^2\log k$ (Softmax) and $\tfrac{8n\sigma^2}{3}(1-1/\sqrt k)$ (Entmax).
The uniform results predict nonnegative gaps below these ceilings, and
Proposition~\ref{prop:value-gap-margin} sharpens the prediction at the
reference configurations: the Softmax gap decays like
$\sigma^2e^{-\Delta_{\min}/2\sigma^2}$, and the Entmax gap is exactly zero for
$\sigma\leq\sigma_{\mathrm{id}}$. A wide separation between the observed gaps
and the ceilings is therefore expected rather than a sign of looseness in the
experiment, since the ceilings are attained only at degenerate configurations.
Proposition~\ref{prop:shannon-tsallis-ordering} further predicts that the Entmax
gap lies strictly below the Softmax gap at every temperature. We report the
fraction of runs respecting the uniform inequalities, the two-sided estimate
\eqref{eq:value-gap-exponential} and this ordering, and the per-run slope of
$\log(G_\sigma/\sigma^2)$ against $\beta$, predicted to be $-1$.

These four value-function statements, which are four of the seven checks of
Table~\ref{tab:rates}, hold in every record of every dataset (column ``Checks''), and the per-run slope of
$\log(G_\sigma/\sigma^2)$ (column ``Slope $G_\sigma$'') is close to $-1$, with the
same finite-$\beta$ steepening as $E_T$ on the well-separated datasets, because
$G_\sigma$ is likewise a sum over observations. On the blobs of
Figure~\ref{fig:value-gap} the Softmax gap falls by about thirty orders of
magnitude across the temperature grid, while the uniform ceiling decreases only
in proportion to $\sigma^2$.

\begin{figure}[ht]
    \centering
    \includegraphics[width=0.47\linewidth]{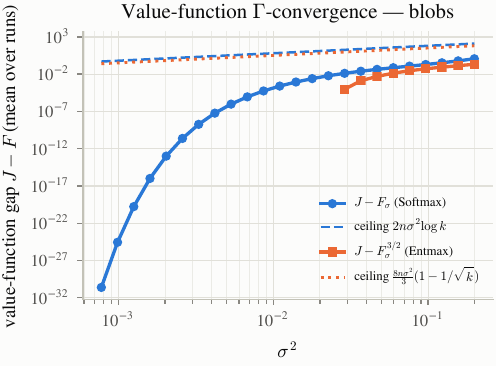}
    \caption{Value-function $\Gamma$-convergence on Gaussian blobs. The gaps
        $J-F_\sigma$ and $J-F_\sigma^{(3/2)}$ at the Lloyd reference configurations,
        averaged over runs at each temperature, against
        $\sigma^2$, on logarithmic axes, overlaid with the uniform ceilings
        $2n\sigma^2\log k$ and $\tfrac{8n\sigma^2}{3}(1-1/\sqrt k)$. The ceilings
        grow linearly in $\sigma^2$, whereas the observed Softmax gap decays
        exponentially in $\beta=\Delta_{\min}/(2\sigma^2)$
        (Proposition~\ref{prop:value-gap-margin}). The Entmax gap lies strictly
        below the Softmax gap at every temperature
        (Proposition~\ref{prop:shannon-tsallis-ordering}), and it is exactly zero
        for $\sigma\leq\sigma_{\mathrm{id}}$, so it is not drawn there.}
    \label{fig:value-gap}
\end{figure}

\subsection{Optimization Dynamics}
\label{subsec:optimization-results}

The preceding experiments evaluate fixed-reference maps and objective values.
They do not, by themselves, characterize the optimizers, and
Section~\ref{subse:updateequiv} is explicit that gradient descent on
$F_\sigma$, alternating Softmax, alternating Entmax, and Lloyd's algorithm are
\emph{distinct} procedures that coincide only under specific preconditioning
and step-size choices. This experiment tests those optimization statements;
Section~\ref{subsec:trajectory-results} then tests the trajectory theorems.

Three predictions are evaluated. First, the descent and sufficient-decrease
property of Proposition~\ref{prop:alternating-descent}: we verify that each method's
own objective is monotone non-increasing ($J$ for Lloyd, $F_\sigma$ for
alternating Softmax and for gradient descent, and $F_\sigma^{(3/2)}$ for
alternating Entmax, for which the same argument applies to the Tsallis energy),
and, for the two alternating methods, that the per-step decrease is at least
$\sum_j m_j^\sigma(\mu^{(t)})\|\mu_j^{(t+1)}-\mu_j^{(t)}\|^2$, as in
\eqref{eq:alternating-sufficient-decrease}. Second, the variable-metric
identity \eqref{eq:fixed-point-stationarity}: at convergence,
$\|T_\sigma(\mu)-\mu\|$ and $\|\nabla F_\sigma(\mu)\|$ are checked to vanish
together. Third, the differences between the methods: all four methods
are run from the same paired initializations and compared by iterations to
tolerance, empty-cluster events, wall-clock time, the terminal
permutation-invariant discrepancy $d_{\mathrm{match}}$ to Lloyd's solution, and
the terminal hard distortion. In particular, the
adaptive step $\eta_{t,j}=1/(2m_j^\sigma(\mu^{(t)}))$ is predicted to make one
alternating-Softmax step coincide with the exact weighted-centroid update,
whereas a common scalar rate does not. Because a Lloyd update, an exact
weighted-centroid update, and an ordinary gradient step perform different
amounts of optimization, a shared iteration budget is not used as a controlled
comparison; convergence is measured by the stopping criterion
\eqref{eq:centroid-stopping-criterion} and the gradient diagnostic
\eqref{eq:gradient-norm-diagnostic}.

The temperature is fixed at $\widetilde\beta=4$, relative to the median margin
$\widetilde\Delta_{\min}$ of each dataset's admissible reference runs, that is
$\sigma=(\widetilde\Delta_{\min}/8)^{1/2}$. For a typical reference configuration
this lies inside the Entmax identification regime, but not necessarily for every
iterate along a trajectory. All methods use the stopping criterion
\eqref{eq:centroid-stopping-criterion} with $\varepsilon_\mu=10^{-6}$ and the cap
$T=150$, and the wall-clock times are measured with one run at a time. Every
per-run quantity is stored with the code; the tables report medians and
percentages over runs.

\begin{table}[ht]
\centering
\small
\setlength{\tabcolsep}{4pt}
\begin{tabular}{lrrrrrrrr}
\toprule
Method & Iter. & Conv.\ \% & ms/iter. & Mono.\ \% & Suff.\ \% & $\|\nabla F_\sigma\|$ & $J<J_{\mathrm{L}}$ \% & $J>J_{\mathrm{L}}$ \% \\
\midrule
Lloyd & $5$ & $100.0$ & $0.0323$ & $100.0$ & -- & -- & -- & -- \\
Alt.\ Softmax & $8$ & $100.0$ & $0.0426$ & $100.0$ & $100.0$ & $7.06\times10^{-6}$ & $23.3$ & $1.9$ \\
GD on $F_\sigma$ & $11$ & $93.8$ & $0.0927$ & $100.0$ & -- & $2.64\times10^{-5}$ & $23.0$ & $9.9$ \\
Alt.\ Entmax & $6$ & $100.0$ & $0.668$ & $100.0$ & $100.0$ & -- & $16.6$ & $2.3$ \\
\bottomrule
\end{tabular}
\caption{Per-method summary pooled over all datasets (3,196 paired runs per method). Iter.: median iterations. Conv.: percentage of runs that met the stopping criterion before the cap. ms/iter.: median wall-clock time per iteration, counting only each method's own update step, with the runs executed one at a time on an otherwise idle laptop CPU. Mono.: percentage of runs whose own objective ($J$ for Lloyd, $F_\sigma$ for alternating Softmax and gradient descent, $F_\sigma^{(3/2)}$ for alternating Entmax) never increased. Suff.: percentage of steps satisfying \eqref{eq:alternating-sufficient-decrease}. $\|\nabla F_\sigma\|$: median terminal gradient norm, for the two methods that minimize $F_\sigma$. $J\lessgtr J_{\mathrm{L}}$: percentage of runs whose hard distortion, after a Lloyd polish of the terminal centroids, is lower (respectively higher) than that of Lloyd's algorithm from the same initialization by more than a relative $10^{-6}$; the polish removes the $O(q(\sigma))$ offset of soft fixed points, so a difference means that another basin was reached. Clusters left without mass in a step: 0 in total.}
\label{tab:methods}
\end{table}

\paragraph{Results.} \emph{Descent.} Every method decreased its own objective
monotonically in all $3{,}196$ paired runs, the two alternating methods
satisfied the sufficient decrease \eqref{eq:alternating-sufficient-decrease} at
every step, and no cluster ever lost all its mass (Table~\ref{tab:methods}). At
termination the gradient of $F_\sigma$ is small for the two methods that
minimize it (median $7.1\times10^{-6}$ for alternating Softmax and
$2.6\times10^{-5}$ for gradient descent), as the identity
\eqref{eq:fixed-point-stationarity} requires at a fixed point.

\emph{Efficiency.} Lloyd's algorithm needs a median of $5$ iterations,
alternating Entmax $6$, alternating Softmax $8$ and gradient descent $11$.
Gradient descent stops before the cap in $93.8\%$ of the runs overall but in
only $15\%$ on digits, whose clusters are the most unbalanced (median
$\rho=2.2\times10^{-3}$). There a single step size, limited by the heaviest
cluster as in \eqref{eq:common-learning-rate-condition}, is far too short for the
light ones, which is exactly the cost of not preconditioning by the masses. The
cost of an iteration is discussed in Section~\ref{subsec:cost}.

\emph{Trajectories.} Alternating Entmax ends at exactly Lloyd's centroids
(median $d_{\mathrm{match}}=0$) on fourteen of the sixteen datasets
(Table~\ref{tab:optimization} in Appendix~\ref{app:per-dataset}), as Corollary~\ref{cor:entmax-trajectory}
predicts: the two trajectories coincide for as long as every Lloyd iterate keeps
a margin of at least $4\sigma^2$. On spiral and digits, whose iterates fall below
this margin more often, the trajectories separate. Alternating Softmax and gradient descent end instead at
fixed points of $T_\sigma$ near Lloyd's, at a median distance between
$1.3\times10^{-4}$ and $0.19$ (except gradient descent on digits, $2.4$, where
$85\%$ of the runs stop at the iteration cap), set by the radius, the balance
and $q(\sigma)=(k-1)e^{-4}$, as Theorems~\ref{thm:soft-fixed-point-branch}
and~\ref{thm:trajectory-tracking} predict through the bound of
Theorem~\ref{thm:soft-centroid-exponential-bound}. These are soft fixed points
close to the hard ones, not the hard ones themselves.
Section~\ref{subsec:trajectory-results} compares the whole trajectories, iteration
by iteration.

\emph{Terminal distortion.} The hard distortion at the terminal point of a soft
method mixes two effects: the $O(q(\sigma))$ offset of a soft fixed point from the
nearby Lloyd fixed point, and a possible change of basin. To isolate the second,
we run Lloyd's algorithm from each terminal point (a ``Lloyd polish'', not
included in the timings) and compare the resulting K-Means fixed point with the
one Lloyd's algorithm reaches from the same initialization. Without the polish,
alternating Softmax and gradient descent would appear worse than Lloyd's
algorithm in $32.7\%$ and $41.7\%$ of the runs, almost entirely because of the
offset; after it, alternating Softmax reaches a lower distortion than Lloyd's
algorithm in $23.3\%$ of the runs and a higher one in $1.9\%$, gradient descent
in $23.0\%$ and $9.9\%$, and alternating Entmax in $16.6\%$ and $2.3\%$
(Table~\ref{tab:methods}); across
five further draws of the data the improvement rates range from $19.6\%$ to $28.1\%$
for alternating Softmax and from $11.3\%$ to $18.8\%$ for alternating Entmax
(Table~\ref{tab:replication} in Appendix~\ref{app:sensitivity}). The changes of basin concentrate on the datasets with
ambiguous cluster structure. On circles, spiral, the separation sweep with
$\delta\leq1$, iris, wine and digits, the soft methods reach a better K-Means fixed
point than Lloyd's algorithm in between $26\%$ and $65\%$ of the runs, and a worse
one mainly on spiral and digits (up to $18.5\%$ of the runs for the alternating
methods, and $56.5\%$ for gradient descent on digits, where most of its runs stop
at the iteration cap). On the well-separated datasets they reach a different fixed
point in at most $1.4\%$ of the runs. A positive temperature therefore acts as a
mechanism to leave poor Lloyd fixed points, as the annealing paths of
Figure~\ref{fig:centroids_paths} also show. Section~\ref{subsec:annealing-results}
quantifies this effect under an annealing schedule.

\subsection{Trajectory Agreement with Lloyd's Algorithm}
\label{subsec:trajectory-results}

Theorem~\ref{thm:trajectory-tracking} and Corollary~\ref{cor:entmax-trajectory}
are statements about whole trajectories, so we test them iteration by iteration.
From each of the $3{,}196$ admissible paired initializations and at seven
temperatures $\widetilde\beta\in\{1/2,1,2,4,8,16,32\}$, again relative to the median
margin $\widetilde\Delta_{\min}$ of the dataset, we run Lloyd's algorithm,
alternating Softmax, alternating Entmax-$1.5$ and gradient descent on $F_\sigma$
from the same $\mu^{(0)}$, with the stopping rule and the cap of
Section~\ref{subsec:optimization-results}. The hard partition induced by the
$t$-th iterate of each soft method is compared with that of Lloyd's $t$-th
iterate; once a method stops, its last partition is kept, so the comparison runs
up to the longer of the two horizons. A run \emph{follows Lloyd's path} when the
partitions agree at every iteration. Along Lloyd's path we also compute the
quantities $\delta$, $\rho$ and $R$ of Theorem~\ref{thm:trajectory-tracking} and
record, for each temperature, whether its hypothesis $\varepsilon(\sigma)\leq r_\star$
holds and whether the hypothesis of Corollary~\ref{cor:entmax-trajectory}, a
margin of at least $4\sigma^2$ with nonempty clusters at every Lloyd iterate,
holds.

\paragraph{Results.} Both predictions hold without exception. The hypothesis of
Theorem~\ref{thm:trajectory-tracking} is satisfied in $1{,}663$ of the $22{,}372$
alternating-Softmax runs, and every one of them follows Lloyd's path; the
hypothesis of Corollary~\ref{cor:entmax-trajectory} is satisfied in $5{,}644$
alternating-Entmax runs, and every one of them follows Lloyd's path as well
(Table~\ref{tab:trajectories} in Appendix~\ref{app:per-dataset}). The check is stricter than the statements, since
it also continues past the end of Lloyd's trajectory. As a test of the theorem it
is weak, however. On this grid, which is common to all runs of a dataset, the
hypothesis of Theorem~\ref{thm:trajectory-tracking} holds in no run with
$\widetilde\beta\leq4$, $92\%$ of the runs that satisfy it come from the five best-separated
datasets, and in the same dataset and temperature the runs that do not satisfy it
follow Lloyd's path $79\%$ of the time; Corollary~\ref{cor:entmax-trajectory}, for
its part, follows almost directly from the one-hot condition.

We therefore test the theorem quantitatively and at the edge of its hypothesis.
For every admissible initialization, the hypothesis $\varepsilon(\sigma)\leq r_\star$ of
Theorem~\ref{thm:trajectory-tracking} holds exactly for $\sigma\leq\sigma^\star$, where
$\sigma^{\star2}=\delta/\bigl(4\log((1+a)(k-1)/a)\bigr)$ with $a=r_\star\rho/R$ and $\delta$,
$\rho$, $R$ computed along that initialization's Lloyd path. We run alternating
Softmax at $\sigma=c\,\sigma^\star$ for $c$ from $0.5$ to $10$ and record the largest
distance $\max_t\|\mu_{\mathrm S}^{(t)}-\mu_{\mathrm L}^{(t)}\|_{2,\infty}$ between the two
trajectories. Every initialization of every dataset is thus tested inside and
outside the hypothesis. In the $12{,}784$ runs with $c\leq1$ the prediction holds
without exception: the trajectories induce the same partitions at every iteration
and their distance never exceeds the bound $\varepsilon(\sigma)$. The bound is loose:
the largest observed distance is $7\times10^{-4}$ times $\varepsilon(\sigma)$ and the median
$3\times10^{-9}$ times at $c=1$, because $\varepsilon(\sigma)$ carries the exponent
$\delta/(4\sigma^2)$ instead of the $\delta/(2\sigma^2)$ that the soft iterates actually
attain (Section~\ref{subsubsec:fixed-points-trajectories}). The hypothesis is also
conservative in temperature: the trajectories still coincide in every run up to
$c=1.5$, and in $98\%$, $92\%$, $65\%$ and $38\%$ of the runs at $c=2$, $3$, $5$ and
$10$.

Both hypotheses are sufficient conditions, and agreement is much more frequent
than they guarantee (Figure~\ref{fig:trajectories}). Over all datasets and
temperatures, alternating Entmax follows Lloyd's path in $61\%$ of the runs and
alternating Softmax in $53\%$, and the fractions rise with $\beta$, to $87\%$ and
$81\%$ at $\widetilde\beta=32$. They do not reach one because $\widetilde\beta$ is measured against the
margin of the converged reference configurations, whereas the early iterates
from a k-means++ initialization can have much smaller margins: the hypothesis of
Corollary~\ref{cor:entmax-trajectory} holds along the whole path in only $58\%$ of
the runs at $\widetilde\beta=32$. Entmax agrees with Lloyd more often than Softmax at every
temperature, as finite identification suggests; the gap is largest on iris,
where at $\widetilde\beta=4$ Entmax follows Lloyd's path in $68\%$ of the runs and Softmax
in $34\%$. Gradient descent behaves differently. Its fraction levels off
at $40\%$ instead of approaching one, because at low temperature it becomes a
gradient method on $J$ with a step shared by all centroids, not the Newton step of
Lloyd's algorithm. It often reaches the same terminal partition by another path
($61\%$ of the runs overall).

\begin{figure}[ht]
    \centering
    \includegraphics[width=0.47\linewidth]{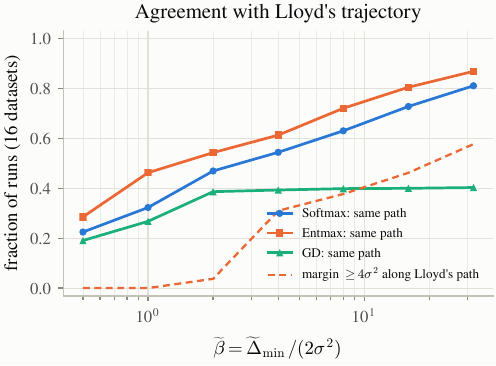}
    \caption{Fraction of runs, over the sixteen datasets, in which each soft method
        induces the same hard partition as Lloyd's algorithm at every iteration,
        against the temperature $\beta=\widetilde\Delta_{\min}/(2\sigma^2)$. Dashed: fraction
        of runs in which every Lloyd iterate has margin at least $4\sigma^2$, the
        hypothesis of Corollary~\ref{cor:entmax-trajectory}; in all of them
        alternating Entmax follows Lloyd's path.}
    \label{fig:trajectories}
\end{figure}

\subsection{Annealing and the Quality of the Fixed Point}
\label{subsec:annealing-results}
\label{subsec:qualitative}

Figure~\ref{fig:centroids_paths} visualizes the soft-to-hard transition along a
deterministic-annealing schedule \citep{rose1998deterministic}. On each of the
four planar geometries, alternating Softmax is run to convergence at 22
temperatures decreasing geometrically from $\sigma=2$ to $\sigma=0.02$, each
warm-started at the solution of the previous one, beginning from a k-means++
initialization. At the highest temperature all $k$ centroids coincide, to within
$10^{-7}$, at the mean of the standardized data, as in the phase-transition
picture of deterministic annealing, where the data mean is the unique minimizer
of $F_\sigma$ above a critical temperature. As $\sigma$ decreases the centroids
separate and travel to the end of the schedule. In every panel this endpoint is
an exact fixed point of Lloyd's map: a Lloyd step from it leaves it unchanged.

\begin{figure}[ht]
    \centering
    \begin{minipage}{0.40\textwidth}
        \centering
        \includegraphics[width=\textwidth]{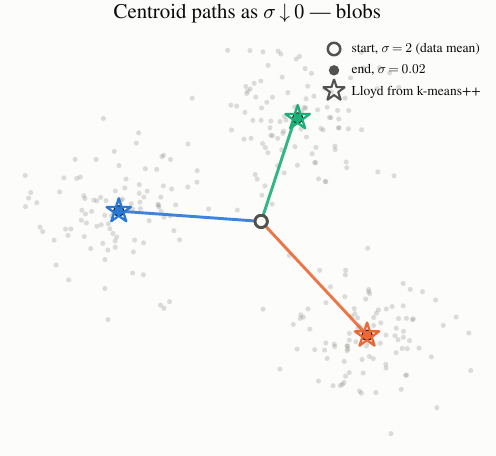}
    \end{minipage}
    \begin{minipage}{0.40\textwidth}
        \centering
        \includegraphics[width=\textwidth]{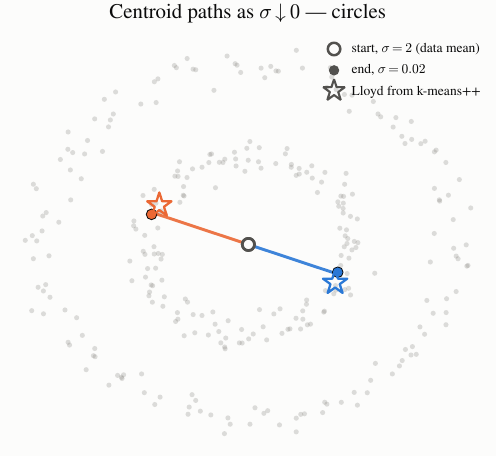}
    \end{minipage}
    \begin{minipage}{0.40\textwidth}
        \centering
        \includegraphics[width=\textwidth]{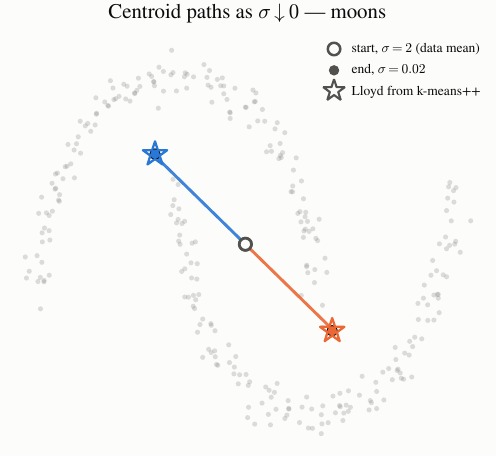}
    \end{minipage}
    \begin{minipage}{0.40\textwidth}
        \centering
        \includegraphics[width=\textwidth]{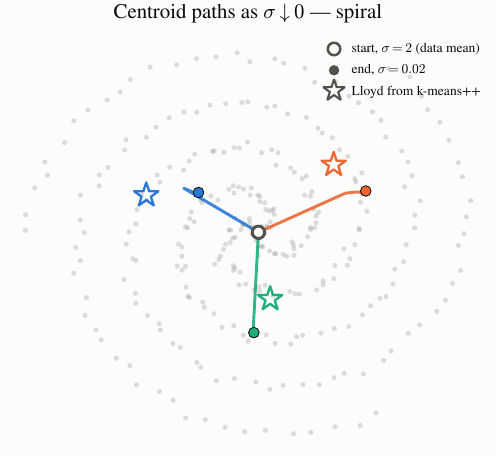}
    \end{minipage}
    \caption{Deterministic-annealing paths of alternating Softmax on four planar
        geometries ($k=3$ for blobs and spiral, $k=2$ for moons and circles). Open
        circle: all centroids at $\sigma=2$, where they coincide at the data mean.
        Lines: centroid positions as $\sigma$ decreases geometrically to $0.02$, each
        temperature warm-started from the previous one. Filled circles: end of the
        schedule, an exact fixed point of Lloyd's map in every panel. Hollow stars:
        Lloyd's solution from the k-means++ initialization that starts the schedule.
        On blobs and moons the two coincide; on spiral and circles annealing reaches
        a different K-Means fixed point, with lower distortion.}
    \label{fig:centroids_paths}
\end{figure}

The stars mark Lloyd's solution from the same k-means++ initialization. On
blobs and moons the annealed endpoint coincides with it. On spiral and circles
annealing ends at a different K-Means fixed point with a lower distortion:
$J=279.66$ against $283.97$ on spiral, and $J=381.93$ against $382.38$ on
circles. This does not contradict the trajectory theorems: an annealing path
starts far above the identification regime, where the margin hypothesis of
Theorem~\ref{thm:trajectory-tracking} fails, so nothing forces it to follow Lloyd's
iterates from the same initialization. What the figure shows is that the
low-temperature end of the annealed dynamics lands on K-Means fixed points, and
that the path taken to reach them can select a better one. We now measure how
often.

To measure the effect over many initializations, every admissible paired
k-means++ initialization is used as the starting point of four procedures:
Lloyd's algorithm; deterministic annealing with alternating Softmax (DA-Softmax)
or with alternating Entmax-$1.5$ (DA-Entmax); and EM for a spherical Gaussian
mixture with one variance and one weight per component (EM-GMM), the model of
which alternating Softmax is the equal-weight, common-variance case. EM-GMM is
the scikit-learn implementation with its default tolerance ($10^{-3}$); it shares
the initial means, while its weights and variances are initialized by its own
default procedure. The
annealing schedule has $40$ temperatures decreasing geometrically from
$\sigma^2=1.2\,\lambda_{\max}$, above the critical temperature of
Proposition~\ref{prop:critical-temperature}, to
$\sigma^2=\min\{\widetilde\Delta_{\min}/64,\,0.12\,\lambda_{\max}\}$, that is,
$\widetilde\beta=32$ or lower. At each temperature the method is run to the stopping rule,
warm-started from the previous temperature after a random perturbation of the
centroids of size $10^{-3}\sqrt{\lambda_{\max}}$, which breaks the symmetry of the
configuration in which centroids coincide. Every procedure ends with Lloyd's
algorithm, so all of them return a K-Means fixed point, and they are compared by
the hard distortion $J$ of that point.

\paragraph{Results.} In the reported draw, both annealing schedules are never
worse than Lloyd's algorithm from the same initialization on fifteen of the
sixteen datasets, and are better whenever the cluster structure is ambiguous: in
$84\%$ to $100\%$ of the runs on spiral, circles, the separation sweep with
$\delta=0.5$, iris and breast cancer, and in $49\%$ and $63\%$ of the runs on the
separation sweep with $\delta=1$ and on wine (Table~\ref{tab:annealing} in Appendix~\ref{app:per-dataset}). On these
fifteen datasets both schedules reach, from every initialization, the lowest
distortion found by any method, so the result of annealing no longer depends on
the initialization. Across five further draws of the data
(Appendix~\ref{app:sensitivity}), DA-Softmax is better than Lloyd's algorithm in
$43\%$ to $51\%$ of all runs, but on circles and spiral the one fixed point it
selects is not always the best: it is then worse than Lloyd's algorithm in up to
$27\%$ of the runs of a dataset, by at most $0.24\%$ of the distortion.
This is the behavior that motivates deterministic annealing
\citep{rose1998deterministic}: above the critical temperature the configuration
in which every centroid sits at the mean is a strict local minimizer of
$F_\sigma$, and the clusters split off from it as the temperature decreases. The
improvements are rare on the well-separated data, where Lloyd's algorithm
already reaches the best fixed point in $96\%$ or more of the runs, but they can
be large: on the separation sweep with $\delta=8$ the distortion is reduced by up
to $98\%$, in a run where Lloyd's algorithm places two centroids in one cluster.
EM-GMM, started
at the same centroids, is less reliable. It improves on Lloyd's algorithm less
often and makes it worse in up to $39\%$ of the runs, because the learned weights
and variances let a component spread over several clusters.

Digits is the exception. Its distortion landscape has many fixed points of
similar value, and the first temperature of the schedule does not bring the
centroids together (see below): the $200$ Lloyd runs end at $200$ different ones, DA-Softmax and
DA-Entmax at $49$ and $81$, and no method reaches the best distortion found more
than once. DA-Softmax and
DA-Entmax still improve on Lloyd's algorithm in $81\%$ and $92\%$ of the runs,
with a median distortion of $70{,}487$ and $69{,}914$ against $71{,}163$, but the
annealed endpoint now depends on the initialization and is worse than Lloyd's in
the remaining runs.

The improvement has a cost. With $40$ temperatures, DA-Softmax takes a median of
$856$ iterations and DA-Entmax $616$, against $5$ for Lloyd's algorithm. In
wall-clock time one annealed run costs between $42$ and $535$ Lloyd runs for
DA-Softmax and between $242$ and $3{,}532$ for DA-Entmax, depending on the dataset;
the cost is smallest on digits, where Lloyd's algorithm itself needs the most
iterations, and largest on the well-separated datasets, where it converges in a
few. Restarting Lloyd's algorithm as many times as the time of one annealed run
allows, and keeping the best result, reaches the distortion of the median
annealed run on every dataset except spiral, where the annealed endpoint
($J=279.66$) is lower than the best of the $200$ Lloyd runs ($279.69$). Annealing is
therefore a way to remove the dependence on the initialization, not a cheaper
alternative to restarts at equal computation.

The independence comes from the first temperature of the schedule, and it holds
for different reasons for the two annealed methods, as we measured by running the
first temperature alone from the first $20$ admissible initializations of each
dataset. For DA-Softmax, the alternating
update at $\sigma^2=1.2\,\lambda_{\max}$ brings the centroids from every initialization
to the data mean, to within $2\times10^{-5}\sqrt{\lambda_{\max}}$, on every dataset
except digits, so the initialization only enters through the small perturbations
that break the symmetry. Proposition~\ref{prop:critical-temperature} concerns
$F_\sigma$ only, and DA-Entmax does not collapse at this temperature: its centroids
stay apart, at a distance of about $\sqrt{\lambda_{\max}}$ or more, but they reach the
same configuration from every initialization, to within
$6\times10^{-3}\sqrt{\lambda_{\max}}$, on every dataset except circles and digits. On
digits neither method brings the centroids together: the data mean is only a
local minimizer of $F_\sigma$ above the critical temperature, and from most
initializations the first temperature ends elsewhere, which is why the annealed
endpoint depends on the initialization there.

\section{Extensions and Discussion}

\label{sec:extension}

\subsection{Embedding K-Means in Gradient-Based Models}
\label{subsec:embedding}

A common practice is to run K-Means on the features of a network and use the
resulting assignments as targets \citep{caron2018deep}; the methods reviewed in
Section~\ref{sec:related} instead train the representation and the clusters
jointly. The results of Sections~\ref{sec:objective}--\ref{sec:maps} indicate which
differentiable term to use for this purpose. Let $h_i=\phi_\theta(x_i)$ be the
representations produced by a network with parameters $\theta$ and let $\mu$ be
trainable centroids. The K-Means term is the reduced objective evaluated on the
representations,
\begin{equation}
    \mathcal L_{\mathrm{km}}(\theta,\mu)
    =
    F_\sigma(\mu;h_1,\ldots,h_n)
    =
    -2\sigma^2\sum_{i=1}^n\log\sum_{j=1}^k
    \exp\!\left(-\frac{\|h_i-\mu_j\|^2}{2\sigma^2}\right),
    \label{eq:kmeans-term}
\end{equation}
or its Tsallis counterpart $F_\sigma^{(3/2)}$, and the training problem is
$\min_{\theta,\mu}\ \mathcal L_{\mathrm{task}}(\theta)+\gamma\,\mathcal L_{\mathrm{km}}(\theta,\mu)$.
Its gradients are
\[
    \nabla_{\mu_j}\mathcal L_{\mathrm{km}}=2\sum_{i=1}^n r_{ij}(\mu_j-h_i),
    \qquad
    \nabla_{h_i}\mathcal L_{\mathrm{km}}=2\sum_{j=1}^k r_{ij}(h_i-\mu_j),
\]
with the Softmax (respectively Entmax-1.5) responsibilities evaluated at the
current point, by \eqref{eq:reduced-functional-gradient} and
\eqref{eq:tsallis-reduced-gradient}; the gradient with respect to $\theta$ follows
by the chain rule through the $h_i$. Three consequences of the preceding results
are relevant in practice.

First, $\mathcal L_{\mathrm{km}}$ is a controlled approximation of the K-Means
distortion of the representations: it lies below it, within $2n\sigma^2\log k$
uniformly (Theorem~\ref{thm:uniform-gamma-convergence}), and in its Tsallis form it
equals the hard distortion at every configuration whose margin is at least $4\sigma^2$
(Proposition~\ref{prop:value-gap-margin}). Second, its gradients involve the
responsibilities but not their derivatives, so they do not vanish at low
temperature (Section~\ref{subsec:gradient-results}); the expected distortion
$D_\sigma=\sum_{i,j}r_{ij}\|h_i-\mu_j\|^2$ does not share this property, and the
expressions above are not its gradient. Third, the centroids can be updated by
the exact weighted means $\mu_j\leftarrow\sum_ir_{ij}h_i/\sum_ir_{ij}$, the
preconditioned step \eqref{eq:variable-metric-preconditioner}, which is available
in closed form and needs no learning rate; a plain gradient step with a learning
rate shared by all centroids is limited by the heaviest cluster and leaves the
light ones nearly fixed (Section~\ref{subsec:optimization-results}).

The K-Means term cannot be used alone when the representation is learned. If
$\phi_\theta$ maps every input to the same point and all centroids sit at that
point, the hard distortion is zero and $F_\sigma$ attains its lower bound
$-2n\sigma^2\log k$; more generally, scaling every representation by $c<1$ scales
the optimal hard distortion by $c^2$. The term must therefore be combined with a
mechanism that prevents this collapse, such as a reconstruction loss
\citep{yang2017towards,guo2017improved,fard2020deep}, a normalization of the
representations, or a constraint on the cluster sizes
\citep{caron2020unsupervised,genevay2019differentiable}.

The differentiable formulation does not extend the representational capacity of
K-Means. It remains confined to Euclidean Voronoi partitions with roughly
isotropic clusters and inherits the poor local minima of K-Means
\citep{jain2010data,raykov2016k}, so it is inadequate for elongated, nonlinear or
manifold-structured clusters in the space where it is applied, for which
density-based \citep{ester1996density}, kernel or spectral methods
\citep{hastie2005elements,dhillon2004kernel,ng2001spectral,vonluxburg2007tutorial}
are better suited. Its contribution is
computational rather than geometric: it makes K-Means a differentiable component
with known behavior at every temperature.

\subsection{Gradients at Low Temperature}
\label{subsec:gradient-results}

The second consequence above can be checked directly. At every admissible Lloyd
reference configuration of the sixteen datasets ($3{,}196$ configurations) and at
$13$ temperatures from $\beta=1/4$ to $\beta=64$, we compute, for every
observation, the gradient with respect to $x_i$ of the hard distortion $J$, of
$F_\sigma$, of $F_\sigma^{(3/2)}$, and of the expected distortion $D_\sigma$
differentiated through the Softmax responsibilities, together with the Frobenius
norm of the Softmax Jacobian $\partial r_i/\partial z_i=\operatorname{diag}(r_i)-r_ir_i^\top$,
the factor that saturates as $\sigma\to0$.

\begin{figure}[ht]
    \centering
    \includegraphics[width=0.47\linewidth]{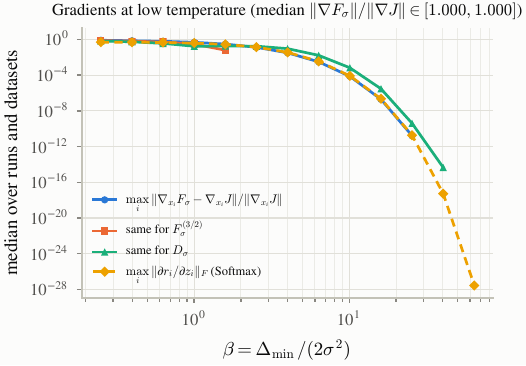}
    \caption{Relative deviation of the per-observation gradients of $F_\sigma$,
        $F_\sigma^{(3/2)}$ and $D_\sigma$ from that of the hard distortion $J$ (the
        maximum over observations, then the median over the $3{,}196$ reference
        configurations), and the largest Frobenius norm of the Softmax Jacobian,
        against $\beta=\Delta_{\min}/(2\sigma^2)$. Missing points are exact zeros.}
    \label{fig:gradients}
\end{figure}

The Jacobian vanishes exponentially with $\beta$, down to $10^{-28}$ at $\beta=64$
(Figure~\ref{fig:gradients}), but the gradients of the reduced objectives do not:
they converge to the gradient of $J$. The median deviation of $\nabla F_\sigma$ from
$\nabla J$ decays at the rate of the Jacobian and is exactly zero in floating
point from $\beta\approx40$; that of $\nabla F_\sigma^{(3/2)}$ is exactly zero
from the first grid point above $\beta=2$ ($\beta\approx2.5$), as finite
identification predicts. Across all
configurations and temperatures, no observation with $\nabla_{x_i}J\neq0$ has a
vanishing gradient of $F_\sigma$ or of $F_\sigma^{(3/2)}$, and the median ratio
$\|\nabla_{x_i}F_\sigma\|/\|\nabla_{x_i}J\|$ is $1.000$ at every temperature. The
gradient of $D_\sigma$ also converges to that of $J$, but more slowly: its extra
term $\sum_jd_{ij}\nabla_{x_i}r_{ij}$ multiplies the vanishing Jacobian by the
distances and by the factor $1/\sigma^2$ of the logits, so it is of order
$\beta e^{-\beta}$ rather than $e^{-\beta}$. A saturating Softmax therefore does not
by itself make the gradients of a K-Means term vanish. They do vanish for losses
whose derivative with respect to the responsibilities stays bounded, such as a
squared error on $r_i$, since their gradient is the saturated Jacobian applied
to a bounded vector; a cross-entropy with $r_i$ as the prediction is not of this
kind, because its derivative $-y/r_i$ is unbounded.

\subsection{End-to-End Deep Clustering}
\label{subsec:deep-results}

We now use the K-Means term \eqref{eq:kmeans-term} inside a network, to see which
of the effects predicted above appear in an end-to-end setting. The question is
not whether the reduced objectives outperform dedicated deep-clustering losses,
which are reported only as reference points, but whether their behavior is the
one the theory predicts. We use the fully connected autoencoder
$d$-500-500-2000-10 of DEC \citep{xie2016unsupervised} on five datasets with
$k=10$ classes: the pixels of MNIST \citep{lecun1998gradient} and Fashion-MNIST
\citep{xiao2017fashion} ($n=70{,}000$) and of USPS \citep{hull1994database}
($n=9{,}298$), and the $2048$-dimensional features of an ImageNet-pretrained
ResNet-50 \citep{he2016deep} for CIFAR-10 \citep{krizhevsky2009learning}
($n=60{,}000$) and STL-10 \citep{coates2011analysis} ($n=13{,}000$), standardized.
For each of five seeds the autoencoder is pretrained by reconstruction for $50$
epochs (Adam, learning rate $10^{-3}$, batch $256$), K-Means with ten restarts is
run on the embeddings, and every method is fine-tuned for $30$ epochs (Adam,
learning rate $10^{-4}$) from that same network and those same centroids. The
reduced objectives $F_\sigma$ and $F_\sigma^{(3/2)}$ update the centroids in closed
form, as responsibility-weighted means, once per epoch, and set the temperature
every epoch to $\sigma^2=0.1$ times the mean squared distance to the nearest
centroid, so that it follows the scale of the representation. As reference points
we train DEC and IDEC \citep{xie2016unsupervised,guo2017improved}, a DCN-type
hard K-Means loss with online centroid updates \citep{yang2017towards}, and the
expected distortion $D_\sigma$ with responsibilities differentiated through, the
objective type of Deep k-Means \citep{fard2020deep}. Every method except DEC adds
the reconstruction loss, with weight $\gamma=0.1$ on the clustering term. Four
ablations of $F_\sigma$ train the centroids with Adam instead of the closed form,
remove the reconstruction loss, fix the temperature at its value at the start of
fine-tuning, and anneal the relative temperature from $1$ to $0.01$ over the
epochs. The pretraining is shorter than in the original papers, so the absolute
accuracies are below the published ones. Appendix~\ref{app:sensitivity} varies the
weight $\gamma$, the relative temperature and the learning rate of the centroids.

\begin{table}[ht]
\centering
\scriptsize
\setlength{\tabcolsep}{4pt}
\begin{tabular}{lcccccr}
\toprule
Method & MNIST & Fashion-MNIST & USPS & CIFAR-10 & STL-10 & Var. \\
\midrule
AE + K-Means & $0.783\pm0.013$ & $0.498\pm0.020$ & $0.712\pm0.006$ & $0.596\pm0.028$ & $0.581\pm0.030$ & $1$ \\
DEC & $0.790\pm0.015$ & $0.527\pm0.015$ & $0.780\pm0.008$ & $0.623\pm0.045$ & $0.628\pm0.035$ & $1.2$ \\
IDEC & $0.818\pm0.013$ & $\mathbf{0.535\pm0.020}$ & $\mathbf{0.789\pm0.004}$ & $\mathbf{0.648\pm0.049}$ & $\mathbf{0.638\pm0.036}$ & $1.5$ \\
DCN (hard) & $0.814\pm0.024$ & $0.475\pm0.010$ & $0.725\pm0.013$ & $0.524\pm0.060$ & $0.470\pm0.043$ & $0.14$ \\
$D_\sigma$ & $0.828\pm0.023$ & $0.462\pm0.017$ & $0.679\pm0.042$ & $0.494\pm0.069$ & $0.467\pm0.052$ & $0.26$ \\
$F_\sigma$ (ours) & $0.830\pm0.014$ & $0.454\pm0.020$ & $0.668\pm0.023$ & $0.574\pm0.054$ & $0.580\pm0.041$ & $0.047$ \\
$F_\sigma^{(3/2)}$ (ours) & $\mathbf{0.831\pm0.012}$ & $0.463\pm0.017$ & $0.681\pm0.024$ & $0.569\pm0.028$ & $0.573\pm0.035$ & $0.05$ \\
\midrule
$F_\sigma$, Adam on $\mu$ & $0.824\pm0.028$ & $0.448\pm0.025$ & $0.656\pm0.072$ & $0.360\pm0.127$ & $0.430\pm0.081$ & $0.26$ \\
$F_\sigma$, no recon. & $0.781\pm0.050$ & $0.267\pm0.060$ & $0.428\pm0.119$ & $0.248\pm0.040$ & $0.132\pm0.012$ & $0.052$ \\
$F_\sigma$, fixed $\sigma$ & $0.129\pm0.037$ & $0.200\pm0.000$ & $0.167\pm0.000$ & $0.100\pm0.000$ & $0.100\pm0.000$ & $0.0035$ \\
$F_\sigma$, annealed $\sigma$ & $0.529\pm0.032$ & $0.535\pm0.014$ & $0.589\pm0.073$ & $0.478\pm0.028$ & $0.703\pm0.083$ & $0.0033$ \\
\bottomrule
\end{tabular}
\caption{End-to-end clustering accuracy (ACC) with a $d$-500-500-2000-10 autoencoder, mean $\pm$ standard deviation over 5 seeds; best non-ablation method in bold. MNIST, Fashion-MNIST and USPS use pixels, CIFAR-10 and STL-10 the features of an ImageNet-pretrained ResNet-50. All methods start from the same pretrained network and the same K-Means centroids. $D_\sigma$: expected soft distortion, the objective type of Deep k-Means. Below the rule, ablations of $F_\sigma$: centroids trained by Adam instead of the closed-form update, no reconstruction loss, temperature fixed at its initial value, and temperature annealed. Var.: total variance of the embeddings after fine-tuning divided by its value before (median over datasets and seeds). NMI and ARI are in Table~\ref{tab:deep-nmi-ari}.}
\label{tab:deep}
\end{table}

\paragraph{Results.} Three predictions of the theory are visible, and each is
isolated by an ablation (Table~\ref{tab:deep}).

\emph{The critical temperature explains when the clustering collapses.} With the
temperature fixed at its initial value, the clustering collapses on every
dataset: the observations end in one or two clusters, with an accuracy between
$0.10$ and $0.20$ and an NMI of at most $0.04$, except on Fashion-MNIST ($0.40$).
The mechanism is the transition of Proposition~\ref{prop:critical-temperature}. A
K-Means term contracts the representation, here to less than $1\%$ of its initial
variance on four datasets, and we record $\lambda_{\max}(C_Z)/\sigma^2$, with $C_Z$ the
covariance of the embeddings, along fine-tuning. In every run with a fixed
temperature on those four datasets this ratio falls below one, to about $0.01$,
and the clustering collapses to a single cluster, or two in one run; in every run
with the temperature tied to the scale of the representation it stays above $2.3$
and no run collapses. On Fashion-MNIST the ratio stays between $5$ and $7$ and the
clustering keeps two clusters, consistent with the cascade of critical
temperatures of deterministic annealing, in which each cluster splits only below
the largest eigenvalue of its own covariance \citep{rose1990statistical}. Figure~\ref{fig:deep-critical} shows every run.

\begin{figure}[t]
    \centering
    \includegraphics[width=\linewidth]{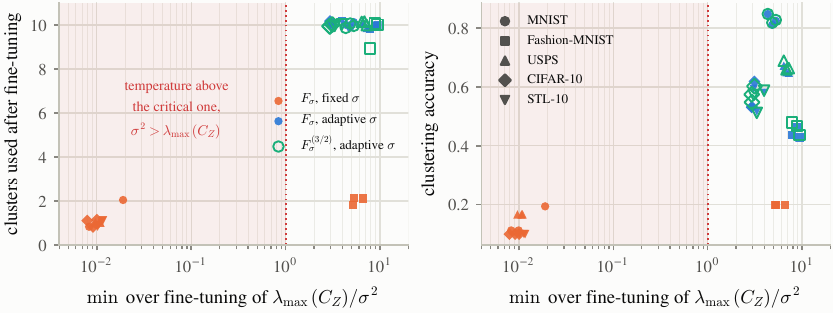}
    \caption{Every run of the critical-temperature check (five datasets, three seeds):
        the smallest value of $\lambda_{\max}(C_Z)/\sigma^2$ during fine-tuning against the
        number of clusters used at the end (left, jittered vertically) and the accuracy
        (right). Shaded: $\sigma^2>\lambda_{\max}(C_Z)$, where the configuration with all
        centroids at the mean is a strict local minimizer of $F_\sigma$
        (Proposition~\ref{prop:critical-temperature}). With a fixed temperature the
        contracting representation crosses this threshold on four datasets and the
        clustering collapses to one or two clusters; with the temperature tied to the
        scale of the representation no run crosses it. The Fashion-MNIST runs with a
        fixed temperature (squares at $5$ to $7$) keep two clusters without crossing it.
        Hollow markers: $F_\sigma^{(3/2)}$, whose smallest ratio coincides with that of
        $F_\sigma$ in every run.}
    \label{fig:deep-critical}
\end{figure}

\emph{The closed-form centroid update is the preconditioned step and needs no
learning rate.} Training the centroids of $F_\sigma$ with Adam at the network's
learning rate gives a lower and much more variable accuracy, down to
$0.360\pm0.127$ against $0.574\pm0.054$ on CIFAR-10, because a step shared by all
centroids is limited by the heaviest cluster
(Section~\ref{subsec:optimization-results}). This gap is largely a matter of the
learning rate: with the rate of the centroids tuned
(Appendix~\ref{app:sensitivity}), Adam matches the closed form on MNIST and USPS
and exceeds it on CIFAR-10 ($0.641$ against $0.574$), where the centroids benefit
from following the representation within an epoch, which our once-per-epoch
closed form does not do. The closed form thus provides, without a learning rate,
what a gradient optimizer of the centroids provides once its rate is tuned.

\emph{A K-Means term contracts the representation.} After $30$ epochs the total
variance of the embeddings is between $0.9\%$ and $46\%$ of its initial value for
every K-Means-type loss (DCN, $D_\sigma$, $F_\sigma$, $F_\sigma^{(3/2)}$), whereas the
Student-$t$ objectives of DEC and IDEC change it by a factor between $0.7$ and $4.7$.
A K-Means term lowers its value by shrinking the representation, as noted in
Section~\ref{subsec:embedding}; the reconstruction loss limits the shrinkage but
does not remove it, and without it the accuracy of $F_\sigma$ drops on every dataset,
to $0.13$ on STL-10. Annealing the temperature contracts the representation
further and gives unstable results across datasets, from $0.53$ on MNIST to $0.70$
on STL-10.

As for the clustering accuracy itself, the reduced objectives are within one
standard deviation of the best reference method on MNIST, where the Tsallis and
Softmax forms reach $0.831$ and $0.830$, and below IDEC on the other four
datasets, where they also do not improve the accuracy of K-Means on the
pretrained embeddings although they raise its NMI on Fashion-MNIST, CIFAR-10 and
STL-10 (Table~\ref{tab:deep-nmi-ari}); with five seeds, the differences among the
best methods on each dataset are mostly within one standard deviation. The gap to
IDEC depends on the weight $\gamma$: with $\gamma=0.01$ the accuracy of $F_\sigma$ on
CIFAR-10 matches that of IDEC at its best weight, while on USPS IDEC remains better
at every weight (Appendix~\ref{app:sensitivity}). The reduced objectives behave as differentiable K-Means terms are predicted to
behave, and controlling the scale of the representation is the question that
their use in representation learning leaves open.

\subsection{Codebook Learning in a Vector-Quantized Autoencoder}
\label{subsec:vq-results}

The codebook of a vector-quantized autoencoder (VQ-VAE) \citep{vandenoord2017neural}
is a K-Means problem on the encoder outputs, solved online while the encoder
changes, and it is known to leave many codes unused. It is therefore a setting in
which the update rules of Section~\ref{sec:optimization} can be compared
directly. We train a small convolutional VQ-VAE on CIFAR-10
\citep{krizhevsky2009learning}, with an $8\times8$ grid of latent vectors of dimension
$64$ and $K=512$ codes initialized at random encoder outputs, for $30$ epochs (Adam,
learning rate $2\times10^{-4}$, batch $128$, three seeds). The forward pass always
quantizes each latent vector to its nearest code, with the straight-through
estimator and the same commitment loss for the encoder; only the codebook update
changes, in a $2\times3$ design. The update is either a gradient step on a codebook
loss or the closed-form update of the codes as responsibility-weighted means,
maintained by exponential moving averages; the responsibilities are hard,
Softmax or Entmax-$1.5$. Hard assignments with a gradient step give the original
VQ-VAE loss, and with the closed form they give its moving-average variant
\citep[Appendix~A.1]{vandenoord2017neural}, an online Lloyd algorithm. The Softmax
and Entmax rows use $F_\sigma$ and $F_\sigma^{(3/2)}$ as codebook losses, or the
corresponding online alternating minimization, with $\sigma^2$ equal to $0.1$ times
the mean squared distance to the nearest code in the batch. As a reference we add
the standard remedy for unused codes, which moves every code whose moving-average
usage falls below one assignment per step to a random encoder output every $100$
steps \citep{dhariwal2020jukebox}, applied to both hard rules.

\begin{table}[ht]
\centering
\footnotesize
\setlength{\tabcolsep}{4pt}
\begin{tabular}{llrrr}
\toprule
Assignment & Update & MSE & Perplexity & Used \\
\midrule
Hard & Gradient & $4.16\pm0.09$ & $323\pm47$ & $372$ \\
Hard + reset & Gradient & $3.96\pm0.16$ & $436\pm4$ & $512$ \\
Hard & Closed form & $4.98\pm0.17$ & $92\pm18$ & $104$ \\
Hard + reset & Closed form & $3.98\pm0.10$ & $439\pm5$ & $512$ \\
Softmax & Gradient & $3.93\pm0.19$ & $448\pm4$ & $511$ \\
Softmax & Closed form & $4.06\pm0.16$ & $446\pm13$ & $504$ \\
Entmax-1.5 & Gradient & $3.99\pm0.09$ & $396\pm27$ & $455$ \\
Entmax-1.5 & Closed form & $4.49\pm0.12$ & $185\pm24$ & $211$ \\
\bottomrule
\end{tabular}
\caption{Codebook learning in a VQ-VAE on CIFAR-10 ($K=512$ codes of dimension 64), mean $\pm$ standard deviation over 3 seeds. Assignment: responsibilities used to update the codebook (the forward pass always quantizes to the nearest code). Update: gradient step on the codebook loss, or closed-form update of the codes as responsibility-weighted means maintained by exponential moving averages. Hard + reset: the standard dead-code remedy, which moves every code whose moving-average usage falls below one assignment per step to a random encoder output every 100 steps \citep{dhariwal2020jukebox}. MSE: test reconstruction error ($\times10^{-3}$). Perplexity and Used: exponential of the entropy of the code frequencies on the test set and number of codes used.}
\label{tab:vq}
\end{table}

\paragraph{Results.} Soft assignments keep the codebook in use without any
resetting heuristic (Table~\ref{tab:vq}). With Softmax responsibilities, $504$ to
$511$ of the $512$ codes are used on the test set, against $372$ for the original
loss and $104$ for its moving-average variant, and the reconstruction error,
$3.93\times10^{-3}$ with the gradient step and $4.06\times10^{-3}$ with the closed form,
is as low as that of the hard rules with code resetting ($3.96\times10^{-3}$ and
$3.98\times10^{-3}$, which use all $512$ codes); the differences among these four are
within one standard deviation over seeds. Softmax responsibilities therefore do
what the resetting heuristic does, as a consequence of the objective rather than
of an extra rule, but they do not do better. The effect is largest for the
closed-form update: replacing the hard assignments of the online Lloyd algorithm
by Softmax responsibilities, which turns it into online alternating
minimization of $F_\sigma$, multiplies the number of used codes by five and reduces
the reconstruction error from $4.98\times10^{-3}$ to $4.06\times10^{-3}$. The mechanism
is the one described by the gradient formula of
Section~\ref{subsec:embedding}: every code receives a force from every latent
vector in proportion to its responsibility, so a code that is never the nearest
one still moves toward the data instead of staying where it was initialized.
Entmax-$1.5$ lies in between, because its responsibilities vanish outside a
margin of $4\sigma^2$, which recovers part of the hard behavior; with the gradient
step it is close to Softmax, with the closed form it uses about twice as many
codes as hard assignments. The temperature and codebook size used here were not
tuned; Appendix~\ref{app:sensitivity} reports the same comparison at other values.

\subsection{Dependence on the Data Geometry}
\label{subsec:data-geometry}

The experiments separate two ways in which the data enter. At the level of
maps and value functions, the geometry enters only through the margin
$\Delta_{\min}$ and, in the constants, through $\rho$ and $B$. Once the
temperature is measured in units of the margin, the sixteen datasets collapse
onto a single curve (Figure~\ref{fig:wiseconvergence}), as the normalization
implies. The slopes of the value gap, which aggregate many observations, lie
between $-1.00$ and $-1.09$ on every dataset, and those of the centroid-map error
are steeper only on the well-separated datasets, for the finite-$\beta$ reason
given in Section~\ref{subsec:convresults} (Table~\ref{tab:rates}), whatever the
geometry, the dimension or the origin of the data.

The margin itself, however, depends strongly on the data, and it shrinks as
the sample grows (Section~\ref{subsec:sample-size}), so the temperature at
which every assignment is identified does too.

At the level of optimization trajectories the geometry matters in a different
way. On well-separated data the four methods reach the same fixed point. On data
with ambiguous cluster structure a positive temperature changes the basin that is
reached in a substantial fraction of runs, usually toward a lower distortion
(Section~\ref{subsec:optimization-results} and
Figure~\ref{fig:centroids_paths}), and an annealing schedule removes the
dependence on the initialization on every dataset but digits
(Section~\ref{subsec:annealing-results}). The fixed point that is reached, and its
distortion, are therefore geometry-dependent even though the soft-to-hard limits
are not.

\subsection{Computational Cost}
\label{subsec:cost}

Per iteration, all four methods of Section~\ref{subsec:optimization-results}
compute the $n\times k$ matrix of squared distances, which takes $\Theta(nkd)$
operations and dominates the cost. Lloyd's algorithm then needs $\Theta(nk)$
comparisons for the assignments and $\Theta(nd)$ operations for the centroids,
since each observation contributes to a single centroid. Alternating Softmax
needs $\Theta(nk)$ exponentials and a dense accumulation $\sum_ir_{ij}x_i$ in
$\Theta(nkd)$, of the same order as the distances. Alternating Entmax-1.5 needs the
threshold $\lambda_i^\sigma$ of each observation, in $O(k\log k)$ operations by sorting
\citep{peters2019sparse} or $O(kL)$ by $L$ bisection steps as in our
implementation, followed by an accumulation in $\Theta(nsd)$, where $s\leq k$ is
the number of active clusters per observation and $s=1$ in the identification
regime. Lloyd's algorithm, alternating Softmax and gradient descent therefore
share the leading order $\Theta(nkd)$ per iteration, and alternating Entmax costs
$\Theta(nk(d+\log k))$ with a sorting solver, or $\Theta(nk(d+L))$ with bisection;
all use memory $\Theta(nk+nd)$. The sparsity of Entmax reduces the accumulation
relative to Softmax, but it cannot make an iteration cheaper than Lloyd's, which
is already maximally sparse. In our measurements (Table~\ref{tab:methods}) an
iteration takes $0.032$\,ms for Lloyd, $0.043$\,ms for alternating Softmax,
$0.093$\,ms for gradient descent, the extra time going to the objective evaluations
of its line search, and $0.67$\,ms for alternating Entmax, dominated by the sixty
bisection steps per observation of our threshold solver rather than by an
intrinsic cost.
Our Lloyd implementation accumulates the centroids with the same dense product as
the soft methods, with one-hot weights, so these timings understate the advantage
that a sparse Lloyd update would have.
Inside a network, the backward pass of \eqref{eq:kmeans-term} costs the same order
as the forward pass, since its gradients reuse the responsibilities computed in
the forward pass.

Table~\ref{tab:scale} measures these costs at scale, up to $n=10^6$ observations
and $k=1{,}000$ clusters in $d=64$, in single precision, against the CPU
implementations of Lloyd's algorithm in scikit-learn \citep{pedregosa2011scikit}
and FAISS \citep{johnson2019billion}. Our GPU implementations of
the three alternating updates process the data in blocks of rows, and
alternating Entmax-$1.5$ computes the exact threshold by sorting
\citep{peters2019sparse} instead of by bisection. All methods start from the same
centroids, one observation of each mixture component, so that no cluster starts
empty and the three implementations of Lloyd's algorithm produce the same
iterates: the final distortion of our GPU Lloyd update agrees with that of
scikit-learn to a relative $4\times10^{-11}$ or better. On these mixtures Lloyd's
algorithm reaches its fixed point in three or four iterations, after which
scikit-learn and FAISS stop because the assignments no longer change, so the
three are compared by their total time to that fixed point, with the same number
of iterations, and the soft updates by their time per iteration. The
measurements confirm the analysis above. At fixed $k$ the time per iteration on
the GPU grows by a factor of $7$ to $10$ from $n=10^5$ to $n=10^6$, and for
$n\geq10^5$ alternating Softmax costs between $0.8$ and $1.8$ times Lloyd's update,
the extra cost being the exponentials and the dense accumulation; at $n=10^4$ the
times are dominated by launching GPU kernels. Alternating Entmax-$1.5$ costs
between $7$ and $8$ times Lloyd's update at $n=10^6$, which is the price of sorting
$k$ values per observation; a partial sort restricted to the candidates within
$4\sigma^2$ of the nearest centroid could remove much of it in the identification
regime. On the same laptop, the GPU Lloyd update reaches the fixed point $2.5$ to
$10$ times faster than scikit-learn and $12$ to $15$ times faster than FAISS at
$n=10^6$. The soft updates therefore scale like K-Means itself, and Softmax
responsibilities in particular add less than a factor of two to the cost of a
Lloyd iteration.

\begin{table}[ht]
\centering
\footnotesize
\setlength{\tabcolsep}{4pt}
\begin{tabular}{rrrrrrrrrr}
\toprule
 & & & \multicolumn{3}{c}{Time to the fixed point} & \multicolumn{3}{c}{Time per iteration (GPU)} & \\
\cmidrule(lr){4-6}\cmidrule(lr){7-9}
$n$ & $k$ & Iter. & scikit-learn & FAISS & Lloyd (GPU) & Lloyd & Softmax & Entmax & $|\Delta J|$ \\
\midrule
$10{,}000$ & $10$ & $3$ & $11.7$ & $2.78$ & $1.79$ & $0.596$ & $0.530$ & $1.10$ & $2.6\times10^{-13}$ \\
$10{,}000$ & $100$ & $3$ & $8.09$ & $6.93$ & $3.23$ & $1.08$ & $0.610$ & $1.06$ & $2.0\times10^{-14}$ \\
$10{,}000$ & $1{,}000$ & $4$ & $35.2$ & $51.3$ & $4.32$ & $1.08$ & $2.01$ & $12.3$ & $1.2\times10^{-15}$ \\
$100{,}000$ & $10$ & $3$ & $34.7$ & $26.5$ & $3.33$ & $1.11$ & $0.911$ & $5.18$ & $2.6\times10^{-12}$ \\
$100{,}000$ & $100$ & $3$ & $31.4$ & $58.3$ & $5.59$ & $1.86$ & $2.66$ & $15.0$ & $2.7\times10^{-13}$ \\
$100{,}000$ & $1{,}000$ & $3$ & $78.3$ & $369$ & $28.8$ & $9.61$ & $16.9$ & $74.6$ & $2.0\times10^{-14}$ \\
$1{,}000{,}000$ & $10$ & $3$ & $223$ & $274$ & $22.7$ & $7.56$ & $7.37$ & $52.1$ & $3.2\times10^{-11}$ \\
$1{,}000{,}000$ & $100$ & $3$ & $280$ & $658$ & $48.6$ & $16.2$ & $22.4$ & $127$ & $2.9\times10^{-12}$ \\
$1{,}000{,}000$ & $1{,}000$ & $4$ & $989$ & $5749$ & $390$ & $97.6$ & $171$ & $733$ & $2.7\times10^{-13}$ \\
\bottomrule
\end{tabular}
\caption{Cost at scale on mixtures of $k$ Gaussians in $d=64$, in float32, from the same initial centroids; median over 3 calls after one warm-up call. Left: total time (ms) for three implementations of Lloyd's algorithm to reach the same fixed point, in the number of iterations (Iter.) after which scikit-learn stops because the assignments no longer change; FAISS and our GPU Lloyd update run exactly that number, and all three produce the same iterates. Right: time per iteration (ms) on the GPU of Lloyd's update and of the alternating Softmax and Entmax-1.5 updates, the soft updates over $T=20$ iterations. Entmax-1.5 uses the exact sorting-based threshold. CPU: scikit-learn and FAISS with all threads of a 24-thread laptop CPU; GPU: our implementation on an RTX 5070 Ti Laptop GPU, with the data resident on the device. $|\Delta J|$: relative difference between the final distortion of the GPU Lloyd update and that of scikit-learn.}
\label{tab:scale}
\end{table}

\section{Conclusions}
\label{sec:conclusions}

This paper made precise the sense in which K-Means is the zero-temperature limit
of a normalized Gaussian RBF assignment layer, keeping four levels of statement
apart. For the objective, the entropic relaxation yields a reduced objective
$F_\sigma$ within $2n\sigma^2\log k$ of the K-Means distortion everywhere, hence
$\Gamma$-convergent to it, and above the critical temperature $\sigma^2=\lambda_{\max}$
the configuration with all centroids at the mean is a local minimizer. For
optimization, the alternating soft update is a preconditioned gradient step on
$F_\sigma$ whose hard limit is Lloyd's Newton step, its whole sequence of iterates
converges, and gradient descent with a common step size is a different procedure.
For one update, the responsibilities and the centroid map approach their hard
counterparts exponentially in $\Delta_{\min}/(2\sigma^2)$. For fixed points and
trajectories, every nondegenerate K-Means fixed point has a unique soft fixed
point nearby, and as long as Lloyd's iterates keep a positive margin the soft
iterates follow them with identical partitions.

With the Tsallis entropy of order $3/2$ these limits become exact. The
Entmax-1.5 responsibilities recover the hard assignments, the Lloyd update and
the hard objective for every $\sigma\leq\sqrt{\Delta_{\min}}/2$, a sharp threshold
that extends to the whole Entmax-$\alpha$ family, and they reproduce Lloyd's
trajectory under the same condition along the path; at a fixed temperature, only
a fraction $O(\sigma^2)$ of the observations, those near the cluster boundaries,
remain soft.

The experiments confirm every bound, rate and threshold on sixteen datasets and
six draws of the data, and the trajectory statements hold in every run that
satisfies their hypotheses, including runs placed at the edge of those
hypotheses. They also measure what the theory leaves open. Beyond the
hypotheses, soft and hard trajectories still agree in most runs; when they do
not, a positive temperature tends to lead to a better K-Means fixed point, and
deterministic annealing makes the result independent of the initialization on
all datasets but one, although restarting Lloyd's algorithm for the same computation reaches the same
distortion on all datasets but one.

For learning systems, the results single out the reduced objective, rather than
the expected soft distortion, as the differentiable K-Means term, and the
experiments show its predicted effects. Its gradients are responsibility-weighted
centroid forces that do not vanish at low temperature. Its closed-form centroid
update is the preconditioned step of the theory and needs no learning rate,
whereas a gradient optimizer of the centroids reaches comparable results only
once its rate is tuned for each dataset. A
K-Means term contracts a learned representation, so a fixed temperature crosses
the critical one and the clustering collapses, as the theory predicts, whereas a
temperature tied to the scale of the representation avoids it. In the
codebook of a vector-quantized autoencoder, soft responsibilities keep the codes
in use, as the standard code-resetting heuristic does, as a consequence of the
objective rather than of an extra rule.

Three questions remain open: the behavior of soft and hard trajectories when
Lloyd's path comes close to a tied observation, where the margin condition of
Theorem~\ref{thm:trajectory-tracking} fails; a temperature schedule that keeps up
with the identification threshold, which decays as $n^{-1/2}$ while the fraction of
unidentified observations at a fixed temperature does not grow with $n$; and how
to control the scale of a representation trained with a K-Means term.

\ifdeanonymized
\section*{Acknowledgments}
We thank Gustavo de los R\'ios Alatorre for computational support, Luis Alberto
Mu\~noz for discussions that contributed key ideas to this work, Nezih Nieto for
his encouragement, and Sofia Alvarez for proofreading the manuscript.
\fi

\newpage

\appendix

\section{Proofs}
\label{app:proofs}

This appendix collects the proofs of the results of
Sections~\ref{sec:preliminaries}--\ref{sec:maps}, in the order in which the results
appear there.

\subsection{Proofs for Section~\ref{sec:preliminaries}}
\label{app:proofs-setting}

\begin{proof}[Proof of Lemma~\ref{lem:gibbs-variational-principle}]
    Define the normalizing constant $Z_\tau(c) := \sum_{\ell=1}^k e^{-c_\ell/\tau}$ and the Gibbs probability vector
    \[
        q_j
        :=
        \frac{e^{-c_j/\tau}}{Z_\tau(c)}.
    \]

    Since $\log q_j = -\frac{c_j}{\tau} - \log Z_\tau(c)$, we obtain
    \[
        c_j
        =
        -\tau\log q_j
        -
        \tau\log Z_\tau(c).
    \]
    Therefore, for every $p\in\Delta^{k-1}$,
    \begin{align}
        \sum_{j=1}^k p_jc_j
        +
        \tau\sum_{j=1}^k p_j\log p_j
         & =
        -\tau\sum_{j=1}^k p_j\log q_j
        -\tau\log Z_\tau(c)
        +
        \tau\sum_{j=1}^k p_j\log p_j
        \\
         & =
        -\tau\log Z_\tau(c)
        +
        \tau\sum_{j=1}^k
        p_j\log\frac{p_j}{q_j}
        \\
         & =
        -\tau\log Z_\tau(c)
        +
        \tau\,\mathrm{KL}(p\|q).
    \end{align}

    Because $\mathrm{KL}(p\|q)\geq 0$, with equality if and only if $p=q$, the minimum is attained uniquely at $p^\star=q$. This proves \eqref{eq:gibbs-variational-principle} and \eqref{eq:gibbs-minimizer}.
\end{proof}

\begin{proof}[Proof of Corollary~\ref{cor:softmax-responsibilities}]
    For fixed $\mu$, the functional separates over the observations:
    \[
        \mathcal E_\sigma(\mu,r)=\sum_{i=1}^n\left[\sum_{j=1}^k r_{ij}d_{ij}(\mu)+2\sigma^2\sum_{j=1}^k r_{ij}\log r_{ij}\right].
    \]
    For every fixed $i$, apply
    Lemma~\ref{lem:gibbs-variational-principle} with $c_j=d_{ij}(\mu)$, and $\tau=2\sigma^2$.
    This gives \eqref{eq:softmax-variational}. Since each row has a unique
    minimizer, the complete assignment matrix is also uniquely determined.
\end{proof}

\begin{proof}[Proof of Proposition~\ref{prop:entmax-variational-characterization}]
    For fixed $i$, minimizing the corresponding term in
    \eqref{eq:tsallis-joint-functional} is equivalent, after dividing by
    $2\sigma^2$ and flipping the sign, to the concave maximization
    \[
        \max_{p\in\Delta^{k-1}}
        \left\{
        p^\top z_i^\sigma(\mu)
        +
        H_{3/2}(p)
        \right\}.
    \]
    Because $H_{3/2}$ is strictly concave on the simplex, this problem has a
    unique maximizer, and since the constraints are linear the
    Karush--Kuhn--Tucker conditions are necessary and sufficient for it.
    Introducing a multiplier $\lambda_i^\sigma$ for the constraint $\sum_j p_j=1$ and
    multipliers $\nu_j\geq0$ for $p_j\geq0$, and using
    $\partial H_{3/2}(p)/\partial p_j = -2\sqrt{p_j}$, these conditions read
    \[
        z_{ij}^\sigma - 2\sqrt{p_j} - \lambda_i^\sigma + \nu_j = 0,
        \qquad
        \nu_j\geq0,
        \qquad
        \nu_jp_j=0 .
    \]
    At an active coordinate, $p_j>0$ forces $\nu_j=0$ and
    $\sqrt{p_j}=(z_{ij}^\sigma-\lambda_i^\sigma)/2>0$; at an inactive one, $p_j=0$ and
    $z_{ij}^\sigma-\lambda_i^\sigma=-\nu_j\leq0$. Both cases are captured by
    \[
        p_j = \left[ \frac{ z_{ij}^\sigma-\lambda_i^\sigma }{2} \right]_+^2 .
    \]
    The scalar $\lambda_i^\sigma$ is then fixed by $\sum_jp_j=1$: the map
    $\lambda\mapsto\sum_j[(z_{ij}^\sigma-\lambda)/2]_+^2$ is continuous, strictly decreasing
    where it is positive, at least $1$ at $\lambda=\max_jz_{ij}^\sigma-2$ and zero at
    $\lambda=\max_jz_{ij}^\sigma$, so it takes the value $1$ exactly once. This proves
    \eqref{eq:entmax-closed-form}.
    The essential difference from the Shannon case is the truncation: coordinates
    whose logit falls below $\lambda_i^\sigma$ are set exactly to zero rather than
    merely made small, which is the source of the sparsity exploited in
    Section~\ref{subse:entmax-identification}.
\end{proof}

\subsection{Proofs for Section~\ref{sec:objective}}
\label{app:proofs-objective}

\begin{proof}[Proof of Theorem~\ref{thm:uniform-gamma-convergence}]
    The argument rests on a single elementary observation: a log-sum-exp lies between the maximum of its arguments and that maximum plus $\log k$, and here the relevant maximum corresponds to the nearest centroid. Concretely, for each observation $x_i$ write
    \[
        \underline d_i(\mu)
        :=
        \min_{1\leq j\leq k}
        d_{ij}(\mu)
    \]
    for its distance to the nearest centroid. Since $d_{ij}(\mu)\geq \underline d_i(\mu)$ for every $j$, at least one term in the exponential sum equals $\exp(-\underline d_i(\mu)/(2\sigma^2))$, while every term is bounded above by this same quantity. Hence,
    \[
        \exp\!\left(
        -\frac{\underline d_i(\mu)}{2\sigma^2}
        \right)
        \leq
        \sum_{j=1}^k
        \exp\!\left(
        -\frac{d_{ij}(\mu)}{2\sigma^2}
        \right)
        \leq
        k
        \exp\!\left(
        -\frac{\underline d_i(\mu)}{2\sigma^2}
        \right).
    \]
    Taking logarithms and multiplying by the negative quantity
    $-2\sigma^2$ reverses the inequalities, giving
    \[
        \underline d_i(\mu)-2\sigma^2\log k
        \leq
        -2\sigma^2
        \log
        \left[
            \sum_{j=1}^k
            \exp\!\left(
            -\frac{d_{ij}(\mu)}{2\sigma^2}
            \right)
            \right]
        \leq
        \underline d_i(\mu).
    \]
    Summing over $i=1,\ldots,n$ proves
    \eqref{eq:uniform-softmin-bound}. Since the resulting approximation error
    does not depend on $\mu$, we obtain
    \eqref{eq:uniform-convergence-rate}, and therefore uniform convergence.

    We next establish $\Gamma$-convergence. Let $\sigma_h\downarrow0$ and let
    $\mu_h\to\mu$ in $\mathcal K$. From
    \eqref{eq:uniform-softmin-bound},
    \[
        F_{\sigma_h}(\mu_h)
        \geq
        J(\mu_h)-2n\sigma_h^2\log k.
    \]
    The functional $J$ is continuous because it is a finite sum of finite
    minima of continuous functions. Therefore,
    \[
        \liminf_{h\to\infty}
        F_{\sigma_h}(\mu_h)
        \geq
        J(\mu),
    \]
    which proves the $\Gamma$-$\liminf$ inequality.

    For the recovery sequence, fix $\mu\in\mathcal K$ and choose the constant
    sequence $\widetilde{\mu}_h=\mu$. Uniform convergence implies
    \[
        F_{\sigma_h}(\widetilde{\mu}_h)
        \longrightarrow
        J(\mu),
    \]
    which proves the $\Gamma$-$\limsup$ inequality.

    Finally, let $\mu_h^\star$ minimize $F_{\sigma_h}$ over $\mathcal K$.
    Because $\mathcal K$ is compact, every subsequence has a further
    subsequence, not relabeled, satisfying
    \[
        \mu_h^\star\longrightarrow\bar\mu\in\mathcal K.
    \]
    For every $\nu\in\mathcal K$, optimality gives $F_{\sigma_h}(\mu_h^\star) \leq F_{\sigma_h}(\nu)$. Passing to the limit using uniform convergence yields
    \[
        J(\bar\mu)
        \leq
        J(\nu).
    \]
    Since $\nu$ was arbitrary, $\bar\mu \in \operatorname*{arg\,min}_{\mu\in\mathcal K}J(\mu)$, so every accumulation point of the smoothed minimizers is a hard minimizer. Convergence of the optimal values themselves is then immediate from the same uniform estimate, since
    \[
        \left|
        \min_{\mu\in\mathcal K}F_\sigma(\mu)
        -
        \min_{\mu\in\mathcal K}J(\mu)
        \right|
        \leq
        2n\sigma^2\log k,
    \]
    which completes the proof.
\end{proof}

\begin{proof}[Proof of Corollary~\ref{cor:gamma-whole-space}]
    Let $P_X$ be the Euclidean projection onto $\operatorname{conv}(X)$ and write $P(\mu):=\bigl(P_X(\mu_1),\ldots,P_X(\mu_k)\bigr)$. For $x\in\operatorname{conv}(X)$ and $u\in\mathbb R^d$, the variational characterization of the projection, $\langle u-P_X(u),\,x-P_X(u)\rangle\leq0$, gives
    \[
        \|x-u\|^2
        =
        \|x-P_X(u)\|^2+\|u-P_X(u)\|^2-2\langle u-P_X(u),\,x-P_X(u)\rangle
        \geq
        \|x-P_X(u)\|^2+\|u-P_X(u)\|^2 .
    \]
    Hence $d_{ij}(P(\mu))\leq d_{ij}(\mu)$ for all $i,j$, and the inequality is strict for every $i$ whenever $\mu_j\notin\operatorname{conv}(X)$. Both $J$ and $F_\sigma$ are nondecreasing in each $d_{ij}$, so $J(P(\mu))\leq J(\mu)$ and $F_\sigma(P(\mu))\leq F_\sigma(\mu)$. For $\sigma>0$ the dependence of $F_\sigma$ on the distances is strictly increasing, because
    \[
        \frac{\partial F_\sigma}{\partial d_{ij}}
        =
        r_{ij}^\sigma(\mu)>0 ;
    \]
    therefore $F_\sigma(P(\mu))<F_\sigma(\mu)$ whenever $\mu\notin\mathcal K_X$, and no global minimizer of $F_\sigma$ lies outside $\mathcal K_X$. Since $\mathcal K_X$ is compact and both objectives are continuous, their minima over $\mathcal K_X$ exist, and by the projection inequalities they coincide with the infima over $\mathbb R^{k\times d}$.

    The value estimate now follows from \eqref{eq:uniform-softmin-bound}. Finally, every $\mu_h^\star$ lies in the compact set $\mathcal K_X$ and minimizes $F_{\sigma_h}$ over it, so Theorem~\ref{thm:uniform-gamma-convergence} with $\mathcal K=\mathcal K_X$ shows that its accumulation points belong to $\operatorname*{arg\,min}_{\mathcal K_X}J$. Because the minimal values of $J$ over $\mathcal K_X$ and over $\mathbb R^{k\times d}$ coincide, these points are global minimizers of $J$ over $\mathbb R^{k\times d}$.
\end{proof}

\begin{proof}[Proof of Proposition~\ref{prop:tsallis-uniform-approximation}]
    For every $r\in\mathcal P$,
    \[
        \sum_{i,j}
        r_{ij}d_{ij}(\mu)
        \geq
        J(\mu),
    \]
    because each row of $r$ defines a convex combination of the squared
    distances. Moreover,
    \[
        0
        \leq
        H_{3/2}(r_i)
        \leq
        \frac{4}{3}
        \left(
        1-\frac{1}{\sqrt{k}}
        \right).
    \]
    Therefore,
    \[
        \mathcal E_\sigma^{(3/2)}(\mu,r)
        \geq
        J(\mu)
        -
        \frac{8n\sigma^2}{3}
        \left(
        1-\frac{1}{\sqrt{k}}
        \right).
    \]
    Taking the minimum over $r$ gives the lower bound.

    For the upper bound, choose any hard nearest-centroid assignment. Its
    Tsallis entropy is zero, and its distortion equals $J(\mu)$. Hence,
    \[
        F_\sigma^{(3/2)}(\mu)
        \leq
        J(\mu).
    \]
\end{proof}

\begin{proof}[Proof of Proposition~\ref{prop:shannon-tsallis-ordering}]
    Both entropies are sums of one-dimensional terms, $H(p)-H_{3/2}(p)=\sum_{j=1}^k\varphi(p_j)$ with
    \[
        \varphi(t)
        :=
        -t\log t-\frac43\left(t-t^{3/2}\right),
        \qquad
        t\in[0,1].
    \]
    Clearly $\varphi(0)=\varphi(1)=0$. For $t\in(0,1)$ write $t=u^2$ with $u\in(0,1)$, so that $\varphi(t)=t\,\psi(u)$ with $\psi(u):=-2\log u-\tfrac43(1-u)$. Since $\psi(1)=0$ and $\psi'(u)=-2/u+4/3<0$ on $(0,1]$, $\psi$ is strictly decreasing and $\psi(u)>0$ for $u\in(0,1)$. Hence $\varphi(t)>0$ for every $t\in(0,1)$, which proves \eqref{eq:shannon-tsallis-entropy-ordering} together with its equality case.

    The upper inequality in \eqref{eq:shannon-tsallis-value-ordering} is Proposition~\ref{prop:tsallis-uniform-approximation}. For the lower one, let $r^{3/2}:=r^{\sigma,3/2}(\mu)$ denote the Entmax-$1.5$ responsibilities. Using first that $F_\sigma$ is a minimum over $\mathcal P$ and then \eqref{eq:shannon-tsallis-entropy-ordering} row by row,
    \[
        F_\sigma(\mu)
        \leq
        \mathcal E_\sigma\bigl(\mu,r^{3/2}\bigr)
        =
        \mathcal E_\sigma^{(3/2)}\bigl(\mu,r^{3/2}\bigr)
        -
        2\sigma^2\sum_{i=1}^n\Bigl[H\bigl(r_i^{3/2}\bigr)-H_{3/2}\bigl(r_i^{3/2}\bigr)\Bigr]
        \leq
        F_\sigma^{(3/2)}(\mu).
    \]
    If both inequalities were equalities, the first would force $r^{3/2}=r^\sigma(\mu)$ by the uniqueness in Corollary~\ref{cor:softmax-responsibilities}. The Softmax responsibilities are strictly positive, so for $k\geq2$ no row of $r^\sigma(\mu)$ is a vertex, and the second inequality would then be strict by the equality case of \eqref{eq:shannon-tsallis-entropy-ordering}. Hence $F_\sigma(\mu)<F_\sigma^{(3/2)}(\mu)$.
\end{proof}

The same argument applies to the Tsallis entropy $H_\alpha$ of any order $\alpha>1$ defined in Appendix~\ref{app:EntmaxHardLimit}. Writing $H(p)-H_\alpha(p)=\sum_j t\,\psi_\alpha(t)$ at $t=p_j$, with $\psi_\alpha(t):=-\log t-(1-t^{\alpha-1})/(\alpha(\alpha-1))$, we have $\psi_\alpha(1)=0$ and $\psi_\alpha'(t)=(t^{\alpha-1}/\alpha-1)/t<0$ on $(0,1]$. Hence $H\geq H_\alpha$, and the reduced objective of every Entmax-$\alpha$ relaxation lies between $F_\sigma$ and $J$.

\begin{proof}[Proof of Proposition~\ref{prop:value-gap-margin}]
    Factoring the nearest-centroid term out of each log-sum-exp in \eqref{eq:reduced-entropic-functional} gives the exact identity
    \begin{equation}
        J(\bar\mu)-F_\sigma(\bar\mu)
        =
        2\sigma^2
        \sum_{i=1}^n
        \log\left(1+S_i\right),
        \qquad
        S_i
        :=
        \sum_{\ell\neq j_i}
        \exp\!\left(
        -\frac{d_{i\ell}(\bar\mu)-d_{ij_i}(\bar\mu)}{2\sigma^2}
        \right).
        \label{eq:value-gap-identity}
    \end{equation}
    Since $d_{i\ell}(\bar\mu)-d_{ij_i}(\bar\mu)\geq\Delta_{\min}$ for $\ell\neq j_i$, we have $S_i\leq(k-1)e^{-\beta}$, and $\log(1+S_i)\leq S_i$ gives the upper bound. For the lower bound, keep only an observation $i^\star$ attaining $\Delta_{\min}$: then $S_{i^\star}\geq e^{-\beta}$, and since $\log(1+x)\geq x/(1+x)\geq x/2$ for $0\leq x\leq1$,
    \[
        J(\bar\mu)-F_\sigma(\bar\mu)
        \geq
        2\sigma^2\log\left(1+e^{-\beta}\right)
        \geq
        \sigma^2e^{-\beta}.
    \]
    For the Entmax statement, Theorem~\ref{thm:entmax-finite-identification} shows that for $\sigma\leq\sigma_{\mathrm{id}}$ the minimizing responsibilities are the hard assignments $e_{j_i}$. Their Tsallis entropy is zero and their distortion is $J(\bar\mu)$, so $F_\sigma^{(3/2)}(\bar\mu)=J(\bar\mu)$.
\end{proof}

\begin{proof}[Proof of Proposition~\ref{prop:critical-temperature}]
    At $\bar{\boldsymbol x}$ all responsibilities equal $1/k$, so every soft centroid
    is the data mean, $T_\sigma(\bar{\boldsymbol x})=\bar{\boldsymbol x}$, and the point is
    stationary by \eqref{eq:fixed-point-stationarity}. Differentiating
    \eqref{eq:reduced-functional-gradient}, the Hessian blocks are
    $2\delta_{j\ell}m_jI_d+2\sum_i(\mu_j-x_i)(\nabla_{\mu_\ell}r_{ij}^\sigma)^\top$ with
    $\nabla_{\mu_\ell}r_{ij}^\sigma=r_{ij}^\sigma(\delta_{j\ell}-r_{i\ell}^\sigma)(x_i-\mu_\ell)/\sigma^2$.
    At $\bar{\boldsymbol x}$, $m_j=n/k$ and $r_{ij}^\sigma=1/k$, so the blocks equal
    $(2n/k)\bigl[\delta_{j\ell}I_d-(\delta_{j\ell}-1/k)C/\sigma^2\bigr]$; that is, the Hessian
    is $(2n/k)\bigl[I_{kd}-(I_k-\tfrac1k\mathbf 1\mathbf 1^\top)\otimes C/\sigma^2\bigr]$. The
    projector $I_k-\tfrac1k\mathbf 1\mathbf 1^\top$ has eigenvalue $0$ on $\mathbf 1$ and $1$ on
    its orthogonal complement, which gives the stated spectrum, and (a) and (b)
    follow.

    For (c), the Hessian of the $i$-th term
    $-2\sigma^2\log\sum_j\exp(-d_{ij}/(2\sigma^2))$ equals
    $\sum_jr_{ij}\nabla^2d_{ij}-\tfrac{1}{2\sigma^2}\operatorname{Cov}_{r_i}(\nabla d_{ij})$, where
    the covariance is that of the gradients $\nabla d_{ij}=2(\mu_j-x_i)$, placed in
    block $j$, under the weights $r_{ij}$. The covariance is dominated by the second
    moment, which is block-diagonal with blocks $4r_{ij}(\mu_j-x_i)(\mu_j-x_i)^\top$, so
    the Hessian of the $i$-th term is bounded below by the block-diagonal matrix with
    blocks $2r_{ij}(1-\|\mu_j-x_i\|^2/\sigma^2)I_d$. On $\operatorname{conv}(X)^k$,
    $\|\mu_j-x_i\|\leq D_X$, so for $\sigma^2>D_X^2$ these blocks are positive definite
    and $F_\sigma$ is strictly convex there. Every stationary point satisfies
    $T_\sigma(\mu)=\mu$ and hence lies in $\operatorname{conv}(X)^k$, and every global
    minimizer lies there by Corollary~\ref{cor:gamma-whole-space}; strict convexity
    leaves $\bar{\boldsymbol x}$ as the only one.
\end{proof}

\subsection{Proofs for Section~\ref{sec:optimization}}
\label{app:proofs-optimization}

\begin{proof}[Proof of Proposition~\ref{prop:weighted-centroid-minimizer}]
    Setting the gradient \eqref{eq:fixed-responsibility-gradient} to zero gives the stationarity condition $m_j(r)\mu_j=b_j(r)$, and since $m_j(r)>0$ this determines each centroid uniquely as
    \[
        \mu_j^\star(r)=\frac{b_j(r)}{m_j(r)}.
    \]
    That this stationary point is in fact the global minimizer, rather than merely a critical point, follows from the second-order structure recorded in \eqref{eq:fixed-responsibility-hessian}: the Hessian of $Q$ is block diagonal with positive-definite blocks $2m_j(r)I_d\succ0$, so $Q(\cdot;r)$ is strictly convex and its unique stationary point is the unique global minimizer.
\end{proof}

\begin{proof}[Proof of Proposition~\ref{prop:alternating-descent}]
    Since $r^{(t+1)}$ minimizes the joint energy for fixed $\mu^{(t)}$,
    \[
        F_\sigma(\mu^{(t)})
        =
        \mathcal E_\sigma
        \left(
        \mu^{(t)},r^{(t+1)}
        \right).
    \]
    Since $\mu^{(t+1)}$ minimizes the centroid block for fixed
    $r^{(t+1)}$,
    \[
        \mathcal E_\sigma
        \left(
        \mu^{(t+1)},r^{(t+1)}
        \right)
        \leq
        \mathcal E_\sigma
        \left(
        \mu^{(t)},r^{(t+1)}
        \right).
    \]
    Finally,
    \[
        F_\sigma(\mu^{(t+1)})
        =
        \min_{r\in\mathcal P}
        \mathcal E_\sigma(\mu^{(t+1)},r)
        \leq
        \mathcal E_\sigma
        \left(
        \mu^{(t+1)},r^{(t+1)}
        \right).
    \]
    Combining these inequalities gives
    \[
        F_\sigma(\mu^{(t+1)})
        \leq
        F_\sigma(\mu^{(t)}).
    \]

    For $\sigma>0$, all Softmax responsibilities are strictly positive, and
    therefore every responsibility mass is positive. Hence, the centroid
    subproblem is strictly convex. If
    $\mu^{(t+1)}\neq\mu^{(t)}$, its exact minimization produces a strict decrease
    in the joint energy and consequently a strict decrease in $F_\sigma$.
\end{proof}

\begin{proof}[Proof of Proposition~\ref{prop:alternating-accumulation-points}]
    Each centroid update is a convex combination of the observations.
    Consequently, $\mu_j^{(t)} \in \operatorname{conv}(X)$ after the first update. Thus, the subsequent iterates belong to the compact set
    \[
        \mathcal K
        =
        \operatorname{conv}(X)^k.
    \]

    By Proposition~\ref{prop:alternating-descent}, the sequence of objective
    values is nonincreasing. Moreover,
    \[
        F_\sigma(\mu)
        \geq
        -2n\sigma^2\log k,
    \]
    so the sequence is bounded below and therefore converges.

    The functions $m_j^\sigma(\mu)$ are continuous and strictly positive on
    the compact set $\mathcal K$. Hence, there exists
    $\underline m_\sigma>0$ such that
    \[
        m_j^\sigma(\mu)
        \geq
        \underline m_\sigma
    \]
    for every $\mu\in\mathcal K$ and every $j$. From
    \eqref{eq:alternating-sufficient-decrease}, it follows that $\sum_{t=1}^{\infty} \left\| \mu^{(t+1)}-\mu^{(t)} \right\|^2 <\infty$. Therefore,
    \[
        \left\|
        \mu^{(t+1)}-\mu^{(t)}
        \right\|
        \longrightarrow0.
    \]

    Let $\mu^{(t_q)}\to\bar\mu$ be a convergent subsequence. Since consecutive
    differences vanish, $\mu^{(t_q+1)} \longrightarrow \bar\mu$. The alternating centroid map is continuous for $\sigma>0$, so passing to the limit in
    \[
        \mu^{(t_q+1)}
        =
        T_\sigma(\mu^{(t_q)})
    \]
    gives $\bar\mu=T_\sigma(\bar\mu)$. The fixed-point identity \eqref{eq:fixed-point-stationarity} then implies $\nabla F_\sigma(\bar\mu)=0$.
\end{proof}

\begin{proof}[Proof of Theorem~\ref{thm:alternating-full-convergence}]
    By the proof of Proposition~\ref{prop:alternating-accumulation-points}, the
    iterates $\mu^{(t)}$ with $t\geq1$ lie in $\mathcal K=\operatorname{conv}(X)^k$,
    where every mass satisfies $m_j^\sigma\geq\underline m_\sigma>0$. Using
    \eqref{eq:gradient-soft-centroid-map} with $T_\sigma(\mu^{(t)})=\mu^{(t+1)}$ gives
    $\nabla_{\mu_j}F_\sigma(\mu^{(t)})=2m_j^\sigma(\mu^{(t)})(\mu_j^{(t)}-\mu_j^{(t+1)})$,
    hence $\|\nabla F_\sigma(\mu^{(t)})\|\leq2n\|\mu^{(t+1)}-\mu^{(t)}\|$ since
    $m_j^\sigma\leq n$. Combined with the sufficient decrease
    \eqref{eq:alternating-sufficient-decrease},
    \[
        F_\sigma(\mu^{(t)})-F_\sigma(\mu^{(t+1)})
        \geq
        \underline m_\sigma\|\mu^{(t+1)}-\mu^{(t)}\|^2
        \geq
        \frac{\underline m_\sigma}{2n}\,
        \|\nabla F_\sigma(\mu^{(t)})\|\,\|\mu^{(t+1)}-\mu^{(t)}\|
    \]
    for every $t\geq1$, and by Proposition~\ref{prop:alternating-descent}
    $F_\sigma(\mu^{(t+1)})=F_\sigma(\mu^{(t)})$ implies $\mu^{(t+1)}=\mu^{(t)}$. These are
    the strong descent conditions of \citet[Definition~3.1]{absil2005convergence}.
    The function $F_\sigma$ is real-analytic on $\mathbb R^{k\times d}$, being a sum of
    logarithms of positive sums of exponentials of polynomials, so by
    \citet[Theorem~3.2]{absil2005convergence}, which rests on the
    \L{}ojasiewicz gradient inequality, the iterates either diverge or converge to
    a single point. They are bounded, so they converge, and the limit is stationary
    by Proposition~\ref{prop:alternating-accumulation-points}.
\end{proof}

\begin{proof}[Proof of Proposition~\ref{prop:lloyd-preconditioned-gradient}]
    For every $j$,
    \[
        \nabla_{\mu_j}Q_S(\mu)
        =
        2|S_j|\mu_j
        -
        2\sum_{i\in S_j}x_i.
    \]
    The $j$-th block of
    \eqref{eq:lloyd-preconditioned-gradient} is therefore
    \[
        \mu_j^+
        =
        \mu_j
        -
        \frac{1}{2|S_j|}
        \left(
        2|S_j|\mu_j
        -
        2\sum_{i\in S_j}x_i
        \right),
    \]
    which simplifies to $\mu_j^+ = \frac{1}{|S_j|} \sum_{i\in S_j}x_i$.
\end{proof}

\begin{proof}[Proof of Proposition~\ref{prop:entmax-descent}]
    The argument of Proposition~\ref{prop:alternating-descent} applies with the
    Tsallis energy. With $r=r^{(t+1)}$ fixed, $F_\sigma^{(3/2)}(\mu^{(t)})=\mathcal E_\sigma^{(3/2)}(\mu^{(t)},r)$,
    the entropy term does not depend on $\mu$, and
    $Q(\mu;r)=Q(\mu^{(t+1)};r)+\sum_jm_j\|\mu_j-\mu_j^{(t+1)}\|^2$ for every $\mu$: for
    $m_j>0$ because $\mu_j^{(t+1)}$ is the weighted mean, and for $m_j=0$ because $Q$
    does not depend on $\mu_j$. Taking $\mu=\mu^{(t)}$ and using
    $F_\sigma^{(3/2)}(\mu^{(t+1)})\leq\mathcal E_\sigma^{(3/2)}(\mu^{(t+1)},r)$ gives the
    inequality. If $\mu^{(t+1)}\neq\mu^{(t)}$, some centroid with $m_j>0$ moved, so the
    decrease is positive. The values are nonincreasing and bounded below by
    $-\tfrac{8n\sigma^2}{3}(1-1/\sqrt k)$ (Proposition~\ref{prop:tsallis-uniform-approximation}),
    so they converge. Finally, by \eqref{eq:tsallis-reduced-gradient},
    $\nabla_{\mu_j}F_\sigma^{(3/2)}(\mu)=2\sum_ir_{ij}^{\sigma,3/2}(\mu)(\mu_j-x_i)$. For
    $m_j(\mu)>0$ this equals $2m_j(\mu)\bigl(\mu_j-T_{\sigma,j}^{(3/2)}(\mu)\bigr)$, which
    vanishes at a fixed point; for $m_j(\mu)=0$ every $r_{ij}^{\sigma,3/2}(\mu)$ is zero
    and so is the gradient.
\end{proof}

\subsection{Proofs for Section~\ref{sec:maps}}
\label{app:proofs-maps}

\begin{proof}[Proof of Lemma~\ref{lem:responsibility-concentration}]
    Factoring the contribution of the nearest centroid from the denominator
    gives
    \[
        r_{i j_i}^{\sigma}(\bar\mu)
        =
        \frac{1}{
            1+
            \displaystyle
            \sum_{\ell\neq j_i}
            \exp\!\left(
            -\frac{
                d_{i\ell}(\bar\mu)-d_{ij_i}(\bar\mu)
            }{
                2\sigma^2
            }
            \right)
        }.
    \]
    Therefore,
    \[
        1-r_{i j_i}^{\sigma}(\bar\mu)
        =
        \frac{
            \displaystyle
            \sum_{\ell\neq j_i}
            \exp\!\left(
            -\frac{
                d_{i\ell}(\bar\mu)-d_{ij_i}(\bar\mu)
            }{
                2\sigma^2
            }
            \right)
        }{
            1+
            \displaystyle
            \sum_{\ell\neq j_i}
            \exp\!\left(
            -\frac{
                d_{i\ell}(\bar\mu)-d_{ij_i}(\bar\mu)
            }{
                2\sigma^2
            }
            \right)
        }.
    \]
    Since $d_{i\ell}(\bar\mu)-d_{ij_i}(\bar\mu) \geq \Delta_i$ for every $\ell\neq j_i$, we obtain
    \[
        1-r_{i j_i}^{\sigma}(\bar\mu)
        \leq
        (k-1)
        \exp\!\left(
        -\frac{\Delta_i}{2\sigma^2}
        \right).
    \]

    Similarly, for $\ell\neq j_i$,
    \[
        r_{i\ell}^{\sigma}(\bar\mu)
        =
        \frac{
            \exp\!\left(
            -\frac{
                d_{i\ell}(\bar\mu)-d_{ij_i}(\bar\mu)
            }{
                2\sigma^2
            }
            \right)
        }{
            1+
            \displaystyle
            \sum_{s\neq j_i}
            \exp\!\left(
            -\frac{
                d_{is}(\bar\mu)-d_{ij_i}(\bar\mu)
            }{
                2\sigma^2
            }
            \right)
        }
        \leq
        \exp\!\left(
        -\frac{\Delta_i}{2\sigma^2}
        \right).
    \]
    The uniform componentwise estimate follows from
    $\Delta_i\geq\Delta_{\min}$.

    Finally,
    \[
        \left\|
        r_{i\cdot}^{\sigma}(\bar\mu)-e_{j_i}
        \right\|_1
        =
        \left(1-r_{i j_i}^{\sigma}(\bar\mu)\right)
        +
        \sum_{\ell\neq j_i}
        r_{i\ell}^{\sigma}(\bar\mu).
    \]
    Because the responsibilities sum to one, $\sum_{\ell\neq j_i} r_{i\ell}^{\sigma}(\bar\mu) = 1-r_{i j_i}^{\sigma}(\bar\mu)$, and hence
    \[
        \left\|
        r_{i\cdot}^{\sigma}(\bar\mu)-e_{j_i}
        \right\|_1
        =
        2\left(1-r_{i j_i}^{\sigma}(\bar\mu)\right)
        \leq
        2q_i(\sigma).
    \]
\end{proof}

\begin{proof}[Proof of Theorem~\ref{thm:soft-centroid-exponential-bound}]
    Write $\mu_j^{\mathrm L}:=T_{\mathrm{Lloyd},j}(\bar\mu)$ and $\widetilde\mu_j^\sigma(\bar\mu):=T_{\sigma,j}(\bar\mu)$.
    Let $m_j^\sigma(\bar\mu) = \sum_{i=1}^n r_{ij}^{\sigma}(\bar\mu)$. For every $i\in S_j$, Lemma~\ref{lem:responsibility-concentration} implies
    \[
        r_{ij}^{\sigma}(\bar\mu)
        \geq
        1-q(\sigma).
    \]
    Therefore,
    \begin{equation}
        m_j^\sigma(\bar\mu)
        \geq
        |S_j|\left(1-q(\sigma)\right)
        \geq
        \rho n\left(1-q(\sigma)\right).
        \label{eq:soft-mass-lower-bound}
    \end{equation}

    Using the definition of $\mu_j^{\mathrm{L}}$,
    \[
        \sum_{i\in S_j}
        \left(
        x_i-\mu_j^{\mathrm{L}}
        \right)
        =
        0.
    \]
    Hence,
    \begin{align}
        \widetilde\mu_j^\sigma(\bar\mu)
        -
        \mu_j^{\mathrm{L}}
         & =
        \frac{1}{m_j^\sigma(\bar\mu)}
        \sum_{i=1}^n
        r_{ij}^{\sigma}(\bar\mu)
        \left(
        x_i-\mu_j^{\mathrm{L}}
        \right)
        \nonumber \\
         & =
        \frac{1}{m_j^\sigma(\bar\mu)}
        \left[
            \sum_{i\in S_j}
            \left(
            r_{ij}^{\sigma}(\bar\mu)-1
            \right)
            \left(
            x_i-\mu_j^{\mathrm{L}}
            \right)
            \right.
            \nonumber\\
            &\hspace{3.7cm}
            \left.
            +
            \sum_{i\notin S_j}
            r_{ij}^{\sigma}(\bar\mu)
            \left(
            x_i-\mu_j^{\mathrm{L}}
            \right)
            \right].
        \label{eq:centroid-error-decomposition}
    \end{align}

    Taking norms and using $\|x_i-\mu_j^{\mathrm{L}}\| \leq B_j$ gives
    \[
        \left\|
        \widetilde\mu_j^\sigma(\bar\mu)
        -
        \mu_j^{\mathrm{L}}
        \right\|
        \leq
        \frac{B_j}{m_j^\sigma(\bar\mu)}
        \sum_{i=1}^n
        \left|
        r_{ij}^{\sigma}(\bar\mu)
        -
        \mathbf 1_{\{i\in S_j\}}
        \right|.
    \]
    And by the Lemma~\ref{lem:responsibility-concentration},
    \[
        \sum_{i=1}^n
        \left|
        r_{ij}^{\sigma}(\bar\mu)
        -
        \mathbf 1_{\{i\in S_j\}}
        \right|
        \leq
        nq(\sigma).
    \]
    Combining this estimate with
    \eqref{eq:soft-mass-lower-bound} yields
    \[
        \left\|
        \widetilde\mu_j^\sigma(\bar\mu)
        -
        \mu_j^{\mathrm{L}}
        \right\|
        \leq
        \frac{
            B_jnq(\sigma)
        }{
            \rho n(1-q(\sigma))
        }
        =
        \frac{B_j}{\rho}
        \frac{q(\sigma)}{1-q(\sigma)}.
    \]
    If $q(\sigma)\leq1/2$, then $1/(1-q(\sigma)) \leq 2$, which gives \eqref{eq:simplified-centroid-approximation-bound}.
\end{proof}

\begin{proof}[Proof of Lemma~\ref{lem:entmax-one-hot-condition}]
    Suppose first that the output equals $e_{j^\star}$. From
    \eqref{eq:entmax-closed-form}, the active coordinate must satisfy
    \[
        1
        =
        \left[
            \frac{
                z_{j^\star}-\lambda
            }{2}
            \right]^2,
    \]
    and therefore $\lambda=z_{j^\star}-2$. Every other coordinate must be inactive, so
    \[
        z_\ell-\lambda\leq0
        \qquad
        \text{for every }\ell\neq j^\star.
    \]
    Substituting the value of $\lambda$ yields $z_{j^\star}-z_\ell\geq2$. Conversely, assume $z_{j^\star}-z_\ell\geq2$ for every $\ell\neq j^\star$. Choosing $\lambda=z_{j^\star}-2$ in \eqref{eq:entmax-closed-form} gives
    \[
        p_{j^\star}=1,
        \qquad
        p_\ell=0
        \quad
        \text{for }\ell\neq j^\star.
    \]
    Since the Entmax optimization problem has a unique solution, its output is
    $e_{j^\star}$.
\end{proof}

\begin{proof}[Proof of Theorem~\ref{thm:entmax-finite-identification}]
    The strategy is to translate the geometric margin into the logit gap of two units required by Lemma~\ref{lem:entmax-one-hot-condition}. For every $\ell\neq j_i$, the gap between the winning and losing logits at $\bar\mu$ is the margin rescaled by the temperature,
    \[
        z_{ij_i}^{\sigma}(\bar\mu)
        -
        z_{i\ell}^{\sigma}(\bar\mu)
        =
        \frac{
            d_{i\ell}(\bar\mu)
            -
            d_{ij_i}(\bar\mu)
        }{
            2\sigma^2
        }
        \geq
        \frac{\Delta_{\min}}{2\sigma^2},
    \]
    uniformly in $i$ because $\Delta_{\min}$ is the worst margin over the dataset. Hence, as soon as $\sigma^2 \leq \Delta_{\min}/4$—that is, $\sigma\leq\sigma_{\mathrm{id}}$—this gap is at least $2$ for every point and every competitor simultaneously. The one-hot condition of Lemma~\ref{lem:entmax-one-hot-condition} is then met at every observation, giving $r_{i\cdot}^{\sigma,3/2}(\bar\mu)=e_{j_i}$ for all $i$.
\end{proof}

\begin{proof}[Proof of Proposition~\ref{prop:per-observation-identification}]
    (a) The logits $-d_{i\ell}/(2\sigma^2)$ of $x_i$ have gap $\Delta_i(\mu)/(2\sigma^2)$
    between the nearest centroid and the nearest competitor, so the first statement
    is Lemma~\ref{lem:entmax-one-hot-condition}; if $\Delta_i(\mu)=0$ the observation is
    tied and not one-hot. For the second, with $j=j_i$,
    $d_{i\ell}-d_{ij}=2\langle x_i,\mu_j-\mu_\ell\rangle+\|\mu_\ell\|^2-\|\mu_j\|^2
    =2\|\mu_j-\mu_\ell\|\operatorname{dist}(x_i,\Pi_{j\ell})$, because $x_i$ lies on the side of
    $\Pi_{j\ell}$ closer to $\mu_j$.
    (b) Both value functions are sums over observations. For an identified
    observation the minimizing assignment is $e_{j_i}$, with zero entropy, so its
    contributions to $J$ and $F_\sigma^{(3/2)}$ coincide; for any observation the
    difference lies in $[0,\,2\sigma^2H_{3/2}^{\max}]$, as in the proof of
    Proposition~\ref{prop:tsallis-uniform-approximation}.
    (c) As in \eqref{eq:centroid-error-decomposition},
    $\|T_{\sigma,j}^{(3/2)}(\mu)-T_{\mathrm{Lloyd},j}(\mu)\|\leq(B_j/m_j)\sum_i|r_{ij}^{\sigma,3/2}-\mathbf 1_{\{i\in S_j\}}|$.
    Identified observations contribute zero and the others at most one each, and
    $m_j\geq|S_j|-N_\sigma(\mu)$ because every identified observation of $S_j$ contributes
    a full unit of mass.
    (d) By (a), an unidentified observation lies at distance less than
    $2\sigma^2/\|\mu_j-\mu_\ell\|\leq2\sigma^2/s_{\min}(\mu)$ from one of the $k(k-1)/2$
    bisectors. For a fixed bisector with unit normal $u$, the probability that
    $\langle x,u\rangle$ falls in an interval of length $4\sigma^2/s_{\min}(\mu)$ is at most
    $4\kappa\sigma^2/s_{\min}(\mu)$, and a union bound over the bisectors bounds the
    probability that one observation is unidentified. By linearity of expectation,
    $\mathbb E[N_\sigma(\mu)]$ is $n$ times this probability, whatever the dependence
    between the observations.
\end{proof}

\begin{proof}[Proof of Lemma~\ref{lem:partition-stability}]
    For every $i$ and $\ell$,
    $d_{i\ell}(\mu)-d_{i\ell}(\bar\mu)=\langle\bar\mu_\ell-\mu_\ell,\,2(x_i-\bar\mu_\ell)+(\bar\mu_\ell-\mu_\ell)\rangle$,
    so $|d_{i\ell}(\mu)-d_{i\ell}(\bar\mu)|\leq r_\star(2R+r_\star)$. Hence, for $\ell\neq j_i$,
    $d_{i\ell}(\mu)-d_{ij_i}(\mu)\geq\Delta_{\min}(\bar\mu)-2r_\star(2R+r_\star)=\Delta_{\min}(\bar\mu)/2>0$.
    The nearest centroids, and therefore the partition and the Lloyd update, are
    unchanged.
\end{proof}

\begin{proof}[Proof of Theorem~\ref{thm:soft-fixed-point-branch}]
    Let $\mu\in\mathfrak B_{r_\star}(\bar\mu)$. By Lemma~\ref{lem:partition-stability},
    $\mu$ induces the clusters $S_j$ of $\bar\mu$, has margin at least
    $\Delta_{\min}/2$, and $T_{\mathrm{Lloyd}}(\mu)=T_{\mathrm{Lloyd}}(\bar\mu)=\bar\mu$.
    Theorem~\ref{thm:soft-centroid-exponential-bound} applied at $\mu$, whose
    constants are the balance ratio $\rho$, the radii
    $B_j=\max_i\|x_i-\bar\mu_j\|\leq R$ and a factor $q\leq q'(\sigma)$, gives
    $\|T_\sigma(\mu)-\bar\mu\|_{2,\infty}\leq\varepsilon(\sigma)\leq r_\star$, so $T_\sigma$ maps
    $\mathfrak B_{r_\star}(\bar\mu)$ into itself.

    We bound the derivative of $T_\sigma$ on $\mathfrak B_{r_\star}(\bar\mu)$. With
    $z_{i\ell}=-d_{i\ell}/(2\sigma^2)$ we have
    $\partial r_{ij}^\sigma/\partial z_{i\ell}=r_{ij}^\sigma(\delta_{j\ell}-r_{i\ell}^\sigma)$
    and $\nabla_{\mu_\ell}z_{i\ell}=(x_i-\mu_\ell)/\sigma^2$. Every
    $|r_{ij}^\sigma(\delta_{j\ell}-r_{i\ell}^\sigma)|$ is at most $q'(\sigma)$: if
    $j=\ell=j_i$ it equals $r_{ij_i}^\sigma(1-r_{ij_i}^\sigma)\leq1-r_{ij_i}^\sigma$, and
    otherwise it is at most a non-winning responsibility, and both are bounded by
    $q'(\sigma)$ by Lemma~\ref{lem:responsibility-concentration} at margin
    $\Delta_{\min}/2$. Writing $T_{\sigma,j}=b_j/m_j$ with $b_j=\sum_ir_{ij}^\sigma x_i$ and
    $m_j=\sum_ir_{ij}^\sigma$,
    \[
        \frac{\partial T_{\sigma,j}}{\partial\mu_\ell}
        =
        \frac{1}{m_j}\sum_{i=1}^n
        \left(x_i-T_{\sigma,j}\right)
        \left(\nabla_{\mu_\ell}r_{ij}^\sigma\right)^{\!\top}.
    \]
    Since $\|x_i-T_{\sigma,j}\|\leq R+r_\star$, $\|x_i-\mu_\ell\|\leq R+r_\star$ and
    $m_j\geq\rho n(1-q'(\sigma))$, each block has operator norm at most
    $(R+r_\star)^2q'(\sigma)/(\rho(1-q'(\sigma))\sigma^2)$, so $T_\sigma$ is
    $L(\sigma)$-Lipschitz on the convex set $\mathfrak B_{r_\star}(\bar\mu)$ for
    $\|\cdot\|_{2,\infty}$. The Banach fixed-point theorem gives (a) and (c), and
    $\|\mu_\sigma-\bar\mu\|_{2,\infty}=\|T_\sigma(\mu_\sigma)-\bar\mu\|_{2,\infty}\leq\varepsilon(\sigma)$.

    For (b), differentiating \eqref{eq:reduced-functional-gradient} gives the Hessian
    blocks
    $\nabla^2_{\mu_j\mu_\ell}F_\sigma=2\delta_{j\ell}m_jI_d+2\sum_i(\mu_j-x_i)(\nabla_{\mu_\ell}r_{ij}^\sigma)^\top$.
    The second term has blocks of norm at most $2n(R+r_\star)^2q'(\sigma)/\sigma^2$, hence
    spectral norm at most $k$ times that. By Weyl's inequality the smallest
    eigenvalue of the Hessian is at least
    $2\rho n(1-q'(\sigma))-2kn(R+r_\star)^2q'(\sigma)/\sigma^2=2n\rho(1-q'(\sigma))(1-L(\sigma))>0$.
    A strongly convex function on a convex set has at most one stationary point
    there, which is its minimizer, and $\mu_\sigma$ is stationary by
    \eqref{eq:fixed-point-stationarity}.

    For the converse, $T_{\sigma_h}(\mu_h)=\mu_h$ by
    \eqref{eq:fixed-point-stationarity}. Apply Lemma~\ref{lem:partition-stability}
    at $\mu^\star$, with radius $r^\star$: for $h$ large, $\mu_h\in\mathfrak B_{r^\star}(\mu^\star)$,
    so $T_{\mathrm{Lloyd}}(\mu_h)=T_{\mathrm{Lloyd}}(\mu^\star)$ and the margin of
    $\mu_h$ is at least $\Delta_{\min}(\mu^\star)/2$.
    Theorem~\ref{thm:soft-centroid-exponential-bound} at $\mu_h$ then gives
    $\|\mu_h-T_{\mathrm{Lloyd}}(\mu^\star)\|=\|T_{\sigma_h}(\mu_h)-T_{\mathrm{Lloyd}}(\mu_h)\|\to0$,
    so $\mu^\star=T_{\mathrm{Lloyd}}(\mu^\star)$.
\end{proof}

\begin{proof}[Proof of Theorem~\ref{thm:trajectory-tracking}]
    By induction on $t$; the case $t=0$ is trivial. Assume
    $\|\mu_{\mathrm S}^{(t)}-\mu_{\mathrm L}^{(t)}\|_{2,\infty}\leq\varepsilon(\sigma)\leq r_\star$
    for some $t<T$. Since $\Delta_{\min}(\mu_{\mathrm L}^{(t)})\geq\delta$ and
    $\max_{i,j}\|x_i-\mu_{\mathrm L,j}^{(t)}\|\leq R$, the radius $r_\star$ satisfies the
    condition of Lemma~\ref{lem:partition-stability} at $\mu_{\mathrm L}^{(t)}$. Hence
    $\mu_{\mathrm S}^{(t)}$ induces the partition of $\mu_{\mathrm L}^{(t)}$, has margin at
    least $\delta/2$, and
    $T_{\mathrm{Lloyd}}(\mu_{\mathrm S}^{(t)})=T_{\mathrm{Lloyd}}(\mu_{\mathrm L}^{(t)})=\mu_{\mathrm L}^{(t+1)}$.
    Theorem~\ref{thm:soft-centroid-exponential-bound} at $\mu_{\mathrm S}^{(t)}$, with
    radii $\max_i\|x_i-\mu_{\mathrm L,j}^{(t+1)}\|\leq R$, balance ratio at least $\rho$
    and factor at most $q_\delta(\sigma)$, gives
    $\|\mu_{\mathrm S}^{(t+1)}-\mu_{\mathrm L}^{(t+1)}\|_{2,\infty}=\|T_\sigma(\mu_{\mathrm S}^{(t)})-T_{\mathrm{Lloyd}}(\mu_{\mathrm S}^{(t)})\|_{2,\infty}\leq\varepsilon(\sigma)$.
\end{proof}

\begin{proof}[Proof of Corollary~\ref{cor:entmax-trajectory}]
    By induction on $t$. If $\mu_{\mathrm E}^{(t)}=\mu_{\mathrm L}^{(t)}$ with $t<T$, the
    condition $\Delta_{\min}(\mu_{\mathrm L}^{(t)})\geq4\sigma^2$ is
    $\sigma\leq\sigma_{\mathrm{id}}(\mu_{\mathrm L}^{(t)})$, so the Entmax responsibilities
    at $\mu_{\mathrm L}^{(t)}$ are the hard assignments, no cluster has zero mass, and
    Corollary~\ref{cor:entmax-lloyd-identification} gives
    $\mu_{\mathrm E}^{(t+1)}=T_\sigma^{(3/2)}(\mu_{\mathrm L}^{(t)})=T_{\mathrm{Lloyd}}(\mu_{\mathrm L}^{(t)})=\mu_{\mathrm L}^{(t+1)}$.
\end{proof}

\newpage
\section{The Entmax-\texorpdfstring{$\alpha$}{alpha} Family}

This appendix gives a self-contained account of the zero-temperature limit for the whole Entmax-$\alpha$ family, of which the Softmax and Entmax-$1.5$ limits of the main text are special cases. The main text uses it for the sharpness of the identification threshold (Theorem~\ref{thm:entmax-finite-identification}) and for general $\alpha$ (Remark~\ref{rem:entmax-alpha-descent}).

\subsection{Zero-temperature limit and finite identification for the Entmax-\texorpdfstring{$\alpha$}{alpha} family}
\label{app:EntmaxHardLimit}

Fix a single observation and write $d=(d_1,\ldots,d_k)$ for its squared distances to the centroids, $d_{\min}:=\min_j d_j$, $A:=\operatorname*{arg\,min}_j d_j$ for the set of nearest centroids, and $w:=|A|$. If $A=\{1,\ldots,k\}$ the problem below is invariant under every permutation of the coordinates and all the statements are trivial, so we assume $A\neq\{1,\ldots,k\}$ and define the gap to the nearest competitor outside $A$,
\[
    \delta
    :=
    \min_{\ell\notin A}d_\ell-d_{\min}
    >0 .
\]
For a temperature $\tau>0$ (in the main text $\tau=2\sigma^2$) set $z(\tau):=-d/\tau$. For $\alpha>1$ the Tsallis entropy of order $\alpha$ is
\[
    H_\alpha(p)
    :=
    \frac{1}{\alpha(\alpha-1)}
    \sum_{j=1}^k\left(p_j-p_j^\alpha\right),
\]
and $H_1(p):=-\sum_j p_j\log p_j$ is its limit as $\alpha\to1$. For $\alpha=3/2$ this is \eqref{eq:tsallis-three-halves-entropy}. The Entmax-$\alpha$ map is
\[
    \operatorname{entmax}_\alpha(z)
    :=
    \operatorname*{arg\,max}_{p\in\Delta^{k-1}}
    \left\{p^\top z+H_\alpha(p)\right\},
\]
which is Softmax for $\alpha=1$ (Lemma~\ref{lem:gibbs-variational-principle}), Entmax-$1.5$ for $\alpha=3/2$ (Proposition~\ref{prop:entmax-variational-characterization}), and sparsemax \citep{martins2016softmax} for $\alpha=2$. Since $H_\alpha$ is strictly concave on $\Delta^{k-1}$ for every $\alpha\geq1$, the maximizer is unique. We write $u_A$ for the uniform distribution on $A$, so that $u_A=e_{j^\star}$ when $A=\{j^\star\}$.

\begin{theorem}[Zero-temperature limit of Entmax-$\alpha$]
    \label{thm:entmax-alpha-limit}
    Let $p(\tau):=\operatorname{entmax}_\alpha(z(\tau))$.
    \begin{enumerate}
        \item[(a)] If $\alpha=1$, then for every $\tau>0$
        \[
            \left\|p(\tau)-u_A\right\|_1
            \leq
            \frac{2(k-w)}{w}\,e^{-\delta/\tau}.
        \]
        \item[(b)] If $\alpha>1$, then $p(\tau)=u_A$ exactly for every
        \[
            0<\tau\leq(\alpha-1)\,w^{\alpha-1}\,\delta .
        \]
        \item[(c)] If $\alpha>1$ and the nearest centroid is unique, $A=\{j^\star\}$, then
        \[
            p(\tau)=e_{j^\star}
            \quad\Longleftrightarrow\quad
            \tau\leq(\alpha-1)\,\delta .
        \]
    \end{enumerate}
    In particular, $p(\tau)\to u_A$ as $\tau\to0^+$ for every $\alpha\geq1$; the convergence is exponential in $1/\tau$ for Softmax and is attained at a finite temperature for every $\alpha>1$.
\end{theorem}

\begin{proof}
    (a) Let $S:=\sum_{\ell\notin A}e^{-(d_\ell-d_{\min})/\tau}\leq(k-w)e^{-\delta/\tau}$. Dividing numerator and denominator of the Softmax by $e^{-d_{\min}/\tau}$ gives $p_j(\tau)=1/(w+S)$ for $j\in A$, and $\sum_{\ell\notin A}p_\ell(\tau)=S/(w+S)$. Hence
    \[
        \left\|p(\tau)-u_A\right\|_1
        =
        w\left(\frac1w-\frac{1}{w+S}\right)+\frac{S}{w+S}
        =
        \frac{2S}{w+S}
        \leq
        \frac{2S}{w},
    \]
    and the bound on $S$ concludes.

    For (b) and (c), let $\alpha>1$ and $f(p):=p^\top z+H_\alpha(p)$. Because $t\mapsto t^\alpha$ is continuously differentiable on $[0,\infty)$ with derivative $0$ at the origin, $f$ is concave and continuously differentiable on $\Delta^{k-1}$, with
    \[
        g_j(p)
        :=
        \frac{\partial f}{\partial p_j}(p)
        =
        z_j+\frac{1-\alpha\,p_j^{\alpha-1}}{\alpha(\alpha-1)} .
    \]
    For a concave differentiable function on the simplex, $p$ is a maximizer if and only if $\nabla f(p)^\top(q-p)\leq0$ for every $q\in\Delta^{k-1}$. Testing this with the vertices $q=e_\ell$ shows that it is equivalent to
    \begin{equation}
        g_j(p)=\max_{1\leq\ell\leq k}g_\ell(p)
        \qquad
        \text{for every }j\text{ with }p_j>0.
        \label{eq:simplex-first-order}
    \end{equation}

    (b) Take $p=u_A$ and write $z_{\max}:=-d_{\min}/\tau$. For $j\in A$ all the $g_j(u_A)=z_{\max}+\bigl(1-\alpha w^{1-\alpha}\bigr)/\bigl(\alpha(\alpha-1)\bigr)$ are equal; for $\ell\notin A$, $p_\ell=0$ and $g_\ell(u_A)=z_\ell+1/\bigl(\alpha(\alpha-1)\bigr)$. Condition \eqref{eq:simplex-first-order} therefore reads
    \[
        z_{\max}-z_\ell
        \geq
        \frac{w^{1-\alpha}}{\alpha-1}
        \qquad
        \text{for every }\ell\notin A .
    \]
    Since $z_{\max}-z_\ell=(d_\ell-d_{\min})/\tau\geq\delta/\tau$, it holds whenever $\tau\leq(\alpha-1)w^{\alpha-1}\delta$. Then $u_A$ is a maximizer, and by uniqueness $p(\tau)=u_A$.

    (c) Sufficiency is (b) with $w=1$. Conversely, if $p(\tau)=e_{j^\star}$, condition \eqref{eq:simplex-first-order} at $j=j^\star$ gives $g_{j^\star}(e_{j^\star})=z_{j^\star}-1/\alpha\geq z_\ell+1/\bigl(\alpha(\alpha-1)\bigr)=g_\ell(e_{j^\star})$ for every $\ell\neq j^\star$, that is,
    \[
        \frac{d_\ell-d_{j^\star}}{\tau}
        =
        z_{j^\star}-z_\ell
        \geq
        \frac1\alpha+\frac{1}{\alpha(\alpha-1)}
        =
        \frac{1}{\alpha-1}.
    \]
    Choosing $\ell$ with $d_\ell-d_{j^\star}=\delta$ gives $\tau\leq(\alpha-1)\delta$.

    The limit statement follows from (a) for $\alpha=1$ and from (b) for $\alpha>1$.
\end{proof}

\begin{remark}[Consequences for the main text]
    \label{rem:entmax-alpha-consequences}
    The logit gap $1/(\alpha-1)$ in part (c) is the separation margin of the Tsallis Fenchel--Young loss \citep{blondel2020learning}; parts (a) and (b) are elementary consequences of the closed form of $\operatorname{entmax}_\alpha$ \citep{peters2019sparse}, recorded here in the form we need, together with the translation into distances below. With $\alpha=3/2$, $\tau=2\sigma^2$ and $\delta=\Delta_i$, part (c) states that $r_{i\cdot}^{\sigma,3/2}$ is one-hot if and only if $\sigma\leq\sqrt{\Delta_i}/2$. Taking the worst observation recovers Theorem~\ref{thm:entmax-finite-identification} and shows that its threshold $\sigma_{\mathrm{id}}=\sqrt{\Delta_{\min}}/2$ is \emph{sharp}: for every $\sigma>\sigma_{\mathrm{id}}$ the observation attaining $\Delta_{\min}$ has a non-one-hot assignment. Lemma~\ref{lem:entmax-one-hot-condition} is the case $\alpha=3/2$ of the logit-gap condition $z_{j^\star}-z_\ell\geq1/(\alpha-1)$ used in the proof. For general $\alpha>1$ the identification threshold is
    \[
        \sigma_{\mathrm{id}}(\alpha)
        =
        \sqrt{\frac{(\alpha-1)\,\Delta_{\min}}{2}},
    \]
    which increases with $\alpha$; sparsemax ($\alpha=2$) identifies the hard assignments up to $\sigma=\sqrt{\Delta_{\min}/2}$. Part (b) also covers observations that are tied between $w$ nearest centroids: for $\alpha>1$ they are split \emph{exactly} evenly among them at every sufficiently small temperature, whereas Softmax approaches the even split only asymptotically.
\end{remark}

\newpage
\section{Evaluation Protocol Details and Algorithms}
\label{app:protocol}

This appendix completes the protocol of Section~\ref{subsec:evaluation-setting}.

\paragraph{Declared tolerances.} All computations are carried out in IEEE
double precision, with unit roundoff $u=2^{-53}$. An observed quantity counts
as satisfying an analytic bound $U$ when it does not exceed $U+\mathrm{tol}$,
where $\mathrm{tol}$ is an a priori bound on the rounding error of the
corresponding computation, multiplied by the safety factor $c=4$:
$\mathrm{tol}=c\,k\,u$ for the responsibility errors, which are
normalizations of $k$ terms in $[0,1]$; $\mathrm{tol}=c\,n\,k\,u\max_i\|x_i\|$
for the centroid-map errors, which are sums of $n$ weighted observations over
$k$ centroids; and $\mathrm{tol}=c\,n\,k\,u\max_{i,j}d_{ij}(\bar\mu^{(m)})$ in
absolute terms for the value-function gaps of
Section~\ref{subsubsec:value-function-metrics}, with the relative tolerance
$c\,n\,k\,u$ for the Softmax gap, which is evaluated to full relative precision.
These are standard worst-case estimates for $n$-term floating-point sums, so they
are conservative. A check is informative only where the analytic bound exceeds
its tolerance; at the largest values of $\beta$ both lie below the rounding level.
The same tolerances are used for every bound check of Section~\ref{sec:experiments}.

\paragraph{Entmax identification.} The implementation evaluates the closed form
\eqref{eq:entmax-closed-form} with an explicit positive part, so inactive coordinates
are exactly zero and identification is observed as an exact zero. The tolerance
$\varepsilon_{\mathrm{id}}=10^{-9}$ makes the criterion independent of the accuracy of
the root-finding step for the threshold $\lambda_i^\sigma$. Sweeping the ordered grid,
the empirical threshold is the largest temperature at which identification holds
for that temperature and every colder one,
\begin{equation}
    \widehat\sigma_{\mathrm{id}}^{(m)}
    =
    \max
    \left\{
    \sigma_\ell:
    E_{r,\mathrm{ent}}^{(m)}(\sigma_s)
    \leq
    \varepsilon_{\mathrm{id}}
    \text{ for every }
    \sigma_s\leq\sigma_\ell
    \right\}.
    \label{eq:empirical-identification-threshold}
\end{equation}
The comparison of interest is then between this empirical threshold and its
theoretical prediction. On the grid of the map experiment their ratio is resolved
only to the grid spacing, so we report the number of records that violate
identification below the theoretical threshold, a direct count of failures of
the theorem, and measure the ratio separately on a grid of relative resolution
$10^{-3}$ around $\sigma_{\mathrm{id}}$, for several members of the Entmax family
(Section~\ref{subsec:convresults}). Because Entmax can send the total mass of a
centroid to zero, unlike Softmax, the same measurement counts the clusters left
without mass. No decay
exponent is fitted here: the prediction under test is finite exact
identification, and fitting an $O(\sigma)$ rate would misrepresent a step as a
slope.

\paragraph{Value gaps.} Because $J$ and $F_\sigma$ nearly cancel at low temperature,
$G_\sigma^{(m)}$ is evaluated directly from the identity \eqref{eq:value-gap-identity}
with a numerically stable $\log(1+\cdot)$, rather than as the difference of two
large numbers, which would be dominated by rounding error. We report the fraction
of admissible runs satisfying the uniform inequalities, the two-sided estimate
\eqref{eq:value-gap-exponential} and the ordering of the two gaps, the fitted slope
of $\log\bigl(G_\sigma^{(m)}/\sigma^2\bigr)$ against $\beta$, and the smallest $\beta$ beyond which
$G_{\sigma,\mathrm{ent}}^{(m)}$ vanishes. Theorem~\ref{thm:uniform-gamma-convergence} also
asserts the convergence of global minimizers, but global minimizers of $F_\sigma$ and
$J$ cannot be computed exactly at these problem sizes, so that statement is not
tested directly; the experiments of Section~\ref{subsec:optimization-results}
measure instead how far the fixed points reached by the soft methods lie from the
Lloyd fixed points reached from the same initialization.

\paragraph{Optimization diagnostics.}
Along each trajectory we log the quantities needed to distinguish the four
methods rather than merely to confirm that they stop. Two of these are the
objectives themselves, the reduced functional $F_\sigma(\mu^{(t)})$ and the
hard distortion $J(\mu^{(t)})$, whose monotonicity is the empirical content of
Propositions~\ref{prop:alternating-descent} and~\ref{prop:entmax-descent}. The remainder characterize the path and its cost: the
centroid-step norm $\|\mu^{(t+1)}-\mu^{(t)}\|_F$ and, for gradient methods,
the gradient norm; the number of iterations to tolerance and the wall-clock
time, which measure efficiency; the count of assignment changes between
consecutive iterations, which reveals how quickly the partition stabilizes;
and the incidence of empty-cluster events and of termination at the iteration
cap, which flag the numerical difficulties the theory warns about. Efficiency
is deliberately \emph{not} reduced to a shared iteration budget, because a
Lloyd update, an exact weighted-centroid update, and a single gradient step
accomplish different amounts of optimization per iteration, and equating their
iteration counts would compare the methods on an axis none of them shares.
The permutation alignment of \eqref{eq:permutation-invariant-centroid-distance} is
required for comparisons between terminal solutions. It is not required for the fixed-reference centroid-map
experiments, because the hard assignment labels are already determined by
the reference configuration.

\paragraph{Summaries.} An expectation $\mathbb E[Z(\sigma)]$ is written only when a
probability distribution over initializations or algorithmic randomness has been
specified; otherwise the reported value is the empirical average
$\overline Z_M(\sigma)=\tfrac{1}{M}\sum_{m=1}^M Z^{(m)}(\sigma)$ over the runs. All intervals
are deterministic, so the tables are exactly reproducible from the stored runs, and
individual-run curves or quantile bands are shown wherever an average would hide
nonmonotone behavior or a change of optimization basin.

\begin{algorithm}[ht]
    \caption{\textbf{Lloyd's algorithm} with $k$ clusters and at most $T$ iterations. At each step every observation is assigned to its nearest centroid and each nonempty cluster's centroid is replaced by the mean of its observations; an empty cluster keeps its centroid.}
    \begin{algorithmic}[1]
        \Require Dataset $X\in\mathbb{R}^{n\times d}$, clusters $k$, iterations $T$, tolerance $\varepsilon_\mu$
        \State Initialize $\mu \leftarrow \mathrm{InitCenters}(X,k)$
        \For{$t = 1,\ldots,T$}
        \State $d_{ij}\leftarrow\|x_i-\mu_j\|^2$ for all $i,j$
        \State $j_i \leftarrow \arg\min_{j} d_{ij}$ for all $i$
        \For{$j = 1,\ldots,k$}
        \State $S_j \leftarrow \{i : j_i = j\}$
        \If{$|S_j| > 0$}
        \State $\mu_j^{\mathrm{new}} \leftarrow \frac{1}{|S_j|}\sum_{i\in S_j} x_i$
        \Else
        \State $\mu_j^{\mathrm{new}} \leftarrow \mu_j$
        \EndIf
        \EndFor
        \State Stop if \eqref{eq:centroid-stopping-criterion} holds; otherwise $\mu\leftarrow\mu^{\mathrm{new}}$
        \EndFor
        \State \Return $\mu$
    \end{algorithmic}
    \label{alg:kmeans}
\end{algorithm}

\begin{algorithm}[ht]
    \caption{\textbf{Alternating Entmax-1.5}. Each step computes the Entmax-1.5 responsibilities of the distance logits \eqref{eq:entmax-distance-logits} and replaces each centroid by the responsibility-weighted mean of the observations, which is the preconditioned gradient step $\mu_j-\nabla_{\mu_j}Q/(2m_j)$ of Section~\ref{subsubsec:fixed-responsibilities}. A cluster left without mass is reseeded at the observation farthest from its nearest centroid, which then belongs to that cluster alone in the current step.}
    \label{alg:rbf-gradient}
    \begin{algorithmic}[1]
        \Require Dataset $X \in \mathbb{R}^{n \times d}$, clusters $k$, bandwidth $\sigma$, iterations $T$, tolerance $\varepsilon_\mu$
        \State Initialize $\mu \leftarrow \mathrm{InitCenters}(X, k)$
        \For{$t = 1, \dots, T$}
        \State $d_{ij}\leftarrow\|x_i-\mu_j\|^2$ and $z_{ij}\leftarrow -d_{ij}/(2\sigma^2)$ for all $i,j$
        \State $r_{i\cdot}\leftarrow\mathrm{entmax}_{1.5}(z_{i\cdot})$ for all $i$ \Comment{Proposition~\ref{prop:entmax-variational-characterization}}
        \For{each $j$ with $\sum_i r_{ij}=0$}
        \State $i^\star\leftarrow\arg\max_i\min_\ell d_{i\ell}$ among observations not yet reseeded
        \State $\mu_j\leftarrow x_{i^\star}$;\quad $r_{i^\star\cdot}\leftarrow e_j$
        \EndFor
        \State $m_j \leftarrow \sum_{i=1}^n r_{ij}$ and $b_j \leftarrow \sum_{i=1}^n r_{ij} x_i$ for all $j$
        \State $\mu_j^{\mathrm{new}} \leftarrow b_j/m_j$ if $m_j>0$, and $\mu_j^{\mathrm{new}} \leftarrow \mu_j$ otherwise
        \State Stop if \eqref{eq:centroid-stopping-criterion} holds; otherwise $\mu\leftarrow\mu^{\mathrm{new}}$
        \EndFor
        \State \Return $\mu$
    \end{algorithmic}
\end{algorithm}

\begin{algorithm}[ht]
    \caption{\textbf{Paired evaluation protocol} of Section~\ref{subsec:evaluation-setting}. The same $M$ initializations are shared by every method and every temperature. The map-level and value-level quantities are evaluated at the Lloyd reference configuration of each admissible run, at temperatures adapted to its own margin; the optimization experiment runs every method from the same initialization.}
    \label{alg:evaluation}
    \begin{algorithmic}[1]
        \Require Dataset $X$, clusters $k$, initializations $M$, grid $\beta_1<\dots<\beta_L$, cap $T$, tolerance $\varepsilon_\mu$
        \For{$m=1,\ldots,M$}
        \State $\mu^{(0,m)}\leftarrow$ k-means++ initialization \Comment{shared by all methods}
        \State $\bar\mu^{(m)}\leftarrow$ Lloyd's algorithm (Algorithm~\ref{alg:kmeans}) from $\mu^{(0,m)}$
        \State $\Delta_{\min}^{(m)}\leftarrow$ \eqref{eq:experimental-minimum-margin} at $\bar\mu^{(m)}$
        \If{$\Delta_{\min}^{(m)}=0$ or some cluster of $\bar\mu^{(m)}$ is empty}
        \State record run $m$ as not admissible and skip it \Comment{Section~\ref{subsubsec:reference-configurations}}
        \EndIf
        \For{$\ell=1,\ldots,L$}
        \State $\sigma\leftarrow\bigl(\Delta_{\min}^{(m)}/(2\beta_\ell)\bigr)^{1/2}$ \Comment{\eqref{eq:margin-adapted-temperature-grid}}
        \State Softmax and Entmax errors of responsibilities and centroid maps at $\bar\mu^{(m)}$, and their bounds
        \State Value gaps $J-F_\sigma$ and $J-F_\sigma^{(3/2)}$ at $\bar\mu^{(m)}$, and their bounds
        \EndFor
        \EndFor
        \State $\sigma_{\mathrm{opt}}\leftarrow\bigl(\widetilde\Delta_{\min}/8\bigr)^{1/2}$ with $\widetilde\Delta_{\min}$ the median of $\Delta_{\min}^{(m)}$ over admissible runs \Comment{$\widetilde\beta=4$}
        \For{each admissible run $m$ and each method in \{Lloyd, alt.\ Softmax, GD on $F_\sigma$, alt.\ Entmax\}}
        \State run the method from $\mu^{(0,m)}$ at $\sigma_{\mathrm{opt}}$ until \eqref{eq:centroid-stopping-criterion} or $T$ iterations
        \State record iterations, descent, sufficient decrease, empty clusters, time, $\|\nabla F_\sigma\|$, final $J$
        \State record $d_{\mathrm{match}}$ \eqref{eq:permutation-invariant-centroid-distance} to Lloyd's terminal centroids from $\mu^{(0,m)}$
        \EndFor
        \State \Return the per-run records, summarized as in Section~\ref{subsubsec:paired-initialization}
    \end{algorithmic}
\end{algorithm}

\newpage
\section{Per-Dataset Results}
\label{app:per-dataset}

This appendix collects the per-dataset tables of Section~\ref{sec:experiments}: the governing
constants of every dataset (Table~\ref{tab:constants}), the map and value checks
(Table~\ref{tab:rates}), the terminal comparison of the four methods
(Table~\ref{tab:optimization}), the trajectory statements (Table~\ref{tab:trajectories})
and deterministic annealing (Table~\ref{tab:annealing}).

\begin{table}[ht]
\centering
\small
\setlength{\tabcolsep}{4pt}
\begin{tabular}{lrrrrrrrr}
\toprule
Dataset & $n$ & $d$ & $k$ & Adm./$M$ & $\Delta_{\min}$ (med.) & $\Delta_{\min}$ (min) & $\rho$ (med.) & $B$ (med.) \\
\midrule
Blobs & 300 & 2 & 3 & 200/200 & $0.0993$ & $3.35\times10^{-3}$ & $0.333$ & $3.54$ \\
Moons & 300 & 2 & 2 & 200/200 & $0.0780$ & $0.0780$ & $0.490$ & $2.93$ \\
Spiral & 300 & 2 & 3 & 200/200 & $0.0148$ & $2.36\times10^{-5}$ & $0.263$ & $3.59$ \\
Circles & 300 & 2 & 2 & 200/200 & $0.0254$ & $2.37\times10^{-4}$ & $0.493$ & $2.82$ \\
Separation, $\delta{=}0.5$ & 300 & 2 & 3 & 200/200 & $3.87\times10^{-3}$ & $9.05\times10^{-5}$ & $0.280$ & $1.81$ \\
Separation, $\delta{=}1$ & 300 & 2 & 3 & 200/200 & $0.0291$ & $0.0120$ & $0.323$ & $2.78$ \\
Separation, $\delta{=}2$ & 300 & 2 & 3 & 200/200 & $0.276$ & $0.276$ & $0.333$ & $4.98$ \\
Separation, $\delta{=}4$ & 300 & 2 & 3 & 200/200 & $8.02$ & $1.80\times10^{-3}$ & $0.333$ & $9.04$ \\
Separation, $\delta{=}8$ & 300 & 2 & 3 & 200/200 & $48.1$ & $8.48\times10^{-3}$ & $0.333$ & $16.7$ \\
Blobs, $d{=}5$ & 300 & 5 & 3 & 200/200 & $5.67$ & $9.26\times10^{-4}$ & $0.333$ & $4.79$ \\
Blobs, $d{=}10$ & 300 & 10 & 3 & 200/200 & $13.8$ & $2.29\times10^{-5}$ & $0.333$ & $6.93$ \\
Blobs, $d{=}50$ & 300 & 50 & 3 & 196/200 & $106$ & $9.90\times10^{-4}$ & $0.333$ & $12.8$ \\
Iris & 150 & 4 & 3 & 200/200 & $0.0281$ & $0.0102$ & $0.307$ & $5.58$ \\
Wine & 178 & 13 & 3 & 200/200 & $0.543$ & $2.32\times10^{-3}$ & $0.287$ & $7.48$ \\
Breast cancer & 569 & 30 & 2 & 200/200 & $0.419$ & $0.0713$ & $0.332$ & $22.3$ \\
Digits & 1797 & 64 & 10 & 200/200 & $0.0688$ & $4.91\times10^{-4}$ & $2.23\times10^{-3}$ & $52.5$ \\
\bottomrule
\end{tabular}
\caption{Datasets, admissibility and the governing constants. Admissible runs are those whose Lloyd reference configuration has $\Delta_{\min}>0$ and no empty cluster. $\Delta_{\min}$, the balance ratio $\rho=\min_j|S_j|/n$ and the radius $B$ are computed at each admissible reference configuration and summarized over runs.}
\label{tab:constants}
\end{table}

\begin{table}[ht]
\centering
\footnotesize
\setlength{\tabcolsep}{4pt}
\begin{tabular}{lcccrcrc}
\toprule
Dataset & Slope $E_r$ & Slope $E_T$ & Slope $G_\sigma$ & Runs & Checks & Id.\ viol. & Id.\ $\beta<2$ \\
\midrule
Blobs & $-0.998$ \,[0.000] & $-1.005$ \,[0.000] & $-1.002$ \,[0.000] & 200 & $1.000$ & 0 & $0.000$ \\
Moons & $-0.998$ \,[0.000] & $-0.987$ \,[0.000] & $-1.002$ \,[0.000] & 200 & $1.000$ & 0 & $0.000$ \\
Spiral & $-0.998$ \,[0.000] & $-1.000$ \,[0.013] & $-1.002$ \,[0.005] & 200 & $1.000$ & 0 & $0.000$ \\
Circles & $-0.998$ \,[0.000] & $-0.997$ \,[0.018] & $-1.001$ \,[0.010] & 200 & $1.000$ & 0 & $0.000$ \\
Separation, $\delta{=}0.5$ & $-0.998$ \,[0.000] & $-0.998$ \,[0.031] & $-1.008$ \,[0.009] & 200 & $1.000$ & 0 & $0.000$ \\
Separation, $\delta{=}1$ & $-0.998$ \,[0.000] & $-1.001$ \,[0.005] & $-1.008$ \,[0.008] & 200 & $1.000$ & 0 & $0.000$ \\
Separation, $\delta{=}2$ & $-0.998$ \,[0.000] & $-0.999$ \,[0.000] & $-1.010$ \,[0.000] & 200 & $1.000$ & 0 & $0.000$ \\
Separation, $\delta{=}4$ & $-0.998$ \,[0.000] & $-1.157$ \,[0.000] & $-1.051$ \,[0.000] & 200 & $1.000$ & 0 & $0.000$ \\
Separation, $\delta{=}8$ & $-1.001$ \,[0.000] & $-1.199$ \,[0.000] & $-1.086$ \,[0.000] & 200 & $1.000$ & 0 & $0.000$ \\
Blobs, $d{=}5$ & $-0.999$ \,[0.000] & $-1.204$ \,[0.000] & $-1.055$ \,[0.000] & 200 & $1.000$ & 0 & $0.000$ \\
Blobs, $d{=}10$ & $-1.000$ \,[0.000] & $-1.246$ \,[0.000] & $-1.066$ \,[0.000] & 200 & $1.000$ & 0 & $0.000$ \\
Blobs, $d{=}50$ & $-1.017$ \,[0.000] & $-1.141$ \,[0.000] & $-1.088$ \,[0.000] & 196 & $1.000$ & 0 & $0.000$ \\
Iris & $-0.998$ \,[0.000] & $-0.999$ \,[0.010] & $-1.002$ \,[0.009] & 200 & $1.000$ & 0 & $0.000$ \\
Wine & $-0.998$ \,[0.000] & $-1.001$ \,[0.015] & $-1.002$ \,[0.013] & 200 & $1.000$ & 0 & $0.000$ \\
Breast cancer & $-0.998$ \,[0.000] & $-1.018$ \,[0.047] & $-1.007$ \,[0.013] & 200 & $1.000$ & 0 & $0.000$ \\
Digits & $-0.998$ \,[0.000] & $-1.009$ \,[0.033] & $-1.003$ \,[0.009] & 200 & $1.000$ & 0 & $0.000$ \\
\bottomrule
\end{tabular}
\caption{Rates and bound checks. Slopes: per-run least-squares slope of $\log E_r$, $\log E_T$ and $\log(G_\sigma/\sigma^2)$ against $\beta$ over $\beta\geq2$, with $E_r$ and $E_T$ censored at the declared tolerance; median [interquartile range] over runs, predicted $-1$. Runs: number of admissible runs fitted. Checks: smallest, over the seven bound checks of Section~\ref{subsec:evaluation-setting} (responsibility and centroid-map bounds, uniform and margin estimates of the value gaps, exact Entmax gap for $\beta\geq2$, and $G_{\mathrm{ent}}<G_\sigma$), of the fraction of records satisfying the check within its tolerance. Id.\ viol.: records with $\beta\geq2$ that are not exactly identified. Id.\ $\beta<2$: fraction identified below the threshold, predicted $0$.}
\label{tab:rates}
\end{table}

\begin{table}[ht]
\centering
\footnotesize
\setlength{\tabcolsep}{4pt}
\begin{tabular}{lrrrrrccc}
\toprule
 & \multicolumn{4}{c}{Iterations (median)} & GD & \multicolumn{3}{c}{$d_{\mathrm{match}}$ to Lloyd (median)} \\
\cmidrule(lr){2-5}\cmidrule(lr){7-9}
Dataset & Lloyd & Softmax & GD & Entmax & conv.\ \% & Softmax & GD & Entmax \\
\midrule
Blobs & $4$ & $6$ & $6$ & $4$ & $99$ & $1.75\times10^{-4}$ & $1.75\times10^{-4}$ & $0$ \\
Moons & $6$ & $8$ & $9$ & $6$ & $100$ & $1.50\times10^{-4}$ & $1.50\times10^{-4}$ & $0$ \\
Spiral & $12$ & $20$ & $37$ & $15$ & $100$ & $0.0106$ & $0.0222$ & $7.84\times10^{-3}$ \\
Circles & $7$ & $11$ & $13$ & $8$ & $100$ & $3.05\times10^{-4}$ & $0.0364$ & $0$ \\
Separation, $\delta{=}0.5$ & $12$ & $18$ & $34$ & $16$ & $100$ & $1.34\times10^{-4}$ & $0.0162$ & $0$ \\
Separation, $\delta{=}1$ & $7$ & $13$ & $15$ & $13$ & $100$ & $2.04\times10^{-4}$ & $2.04\times10^{-4}$ & $0$ \\
Separation, $\delta{=}2$ & $3$ & $5$ & $5$ & $3$ & $100$ & $1.37\times10^{-4}$ & $1.37\times10^{-4}$ & $0$ \\
Separation, $\delta{=}4$ & $2$ & $4$ & $4$ & $2$ & $100$ & $1.40\times10^{-3}$ & $1.40\times10^{-3}$ & $0$ \\
Separation, $\delta{=}8$ & $2$ & $6$ & $6$ & $2$ & $100$ & $0.0369$ & $0.0369$ & $0$ \\
Blobs, $d{=}5$ & $2$ & $4$ & $5$ & $2$ & $100$ & $2.89\times10^{-3}$ & $2.89\times10^{-3}$ & $0$ \\
Blobs, $d{=}10$ & $2$ & $5$ & $5$ & $2$ & $100$ & $8.88\times10^{-3}$ & $8.88\times10^{-3}$ & $0$ \\
Blobs, $d{=}50$ & $2$ & $7$ & $7$ & $2$ & $100$ & $0.185$ & $0.185$ & $0$ \\
Iris & $6$ & $12$ & $21$ & $7$ & $100$ & $0.0508$ & $0.0508$ & $0$ \\
Wine & $7$ & $11$ & $17$ & $8$ & $90$ & $1.29\times10^{-3}$ & $0.124$ & $0$ \\
Breast cancer & $8$ & $11$ & $28$ & $8$ & $98$ & $3.98\times10^{-4}$ & $7.51\times10^{-4}$ & $0$ \\
Digits & $22$ & $26$ & $150$ & $23$ & $15$ & $0.0246$ & $2.38$ & $0.0134$ \\
\bottomrule
\end{tabular}
\caption{Optimization dynamics at $\widetilde\beta=4$, from the paired initializations of the admissible runs. Iterations: median number of iterations until the stopping criterion \eqref{eq:centroid-stopping-criterion} or the cap $T$. GD conv.: percentage of gradient-descent runs that met the stopping criterion before the cap. $d_{\mathrm{match}}$: median permutation-invariant distance between the terminal centroids and Lloyd's terminal centroids from the same initialization.}
\label{tab:optimization}
\end{table}

\begin{table}[ht]
\centering
\footnotesize
\setlength{\tabcolsep}{4pt}
\begin{tabular}{lrrrrrr}
\toprule
 & \multicolumn{3}{c}{Same path (\%)} & \multicolumn{3}{c}{Same end (\%)} \\
\cmidrule(lr){2-4}\cmidrule(lr){5-7}
Dataset & Softmax & Entmax & GD & Softmax & Entmax & GD \\
\midrule
Blobs & $87$ & $88$ & $62$ & $98$ & $98$ & $97$ \\
Moons & $62$ & $72$ & $26$ & $100$ & $100$ & $100$ \\
Spiral & $12$ & $24$ & $0$ & $32$ & $48$ & $18$ \\
Circles & $38$ & $46$ & $6$ & $57$ & $64$ & $32$ \\
Separation, $\delta{=}0.5$ & $18$ & $21$ & $0$ & $55$ & $57$ & $37$ \\
Separation, $\delta{=}1$ & $8$ & $15$ & $1$ & $51$ & $51$ & $50$ \\
Separation, $\delta{=}2$ & $83$ & $84$ & $70$ & $100$ & $100$ & $100$ \\
Separation, $\delta{=}4$ & $97$ & $97$ & $96$ & $99$ & $100$ & $99$ \\
Separation, $\delta{=}8$ & $99$ & $99$ & $99$ & $99$ & $99$ & $99$ \\
Blobs, $d{=}5$ & $91$ & $91$ & $90$ & $98$ & $97$ & $95$ \\
Blobs, $d{=}10$ & $89$ & $90$ & $88$ & $96$ & $96$ & $94$ \\
Blobs, $d{=}50$ & $94$ & $94$ & $85$ & $97$ & $97$ & $96$ \\
Iris & $34$ & $68$ & $8$ & $41$ & $74$ & $35$ \\
Wine & $18$ & $32$ & $1$ & $54$ & $72$ & $36$ \\
Breast cancer & $42$ & $59$ & $0$ & $91$ & $96$ & $72$ \\
Digits & $0$ & $2$ & $0$ & $26$ & $44$ & $0$ \\
\bottomrule
\end{tabular}
\caption{Trajectory experiment at $\widetilde\beta=4$. Each soft method is run from the same paired initializations as Lloyd's algorithm. Same path: percentage of runs whose hard partition equals Lloyd's at every iteration. Same end: percentage of runs with the same terminal partition. Over all datasets and temperatures, the hypotheses of Theorem~\ref{thm:trajectory-tracking} hold in 1,663 alternating-Softmax runs, with 0 violations of its prediction, and those of Corollary~\ref{cor:entmax-trajectory} hold in 5,644 alternating-Entmax runs, with 0 violations.}
\label{tab:trajectories}
\end{table}

\begin{table}[ht]
\centering
\footnotesize
\setlength{\tabcolsep}{4pt}
\begin{tabular}{lrrrrrrr}
\toprule
 & \multicolumn{3}{c}{Better/worse than Lloyd (\%)} & \multicolumn{4}{c}{Best found (\%)} \\
\cmidrule(lr){2-4}\cmidrule(lr){5-8}
Dataset & DA-Soft. & DA-Ent. & EM-GMM & Lloyd & DA-Soft. & DA-Ent. & EM-GMM \\
\midrule
Blobs & $2/0$ & $2/0$ & $0/2$ & $98$ & $100$ & $100$ & $97$ \\
Moons & $0/0$ & $0/0$ & $0/0$ & $100$ & $100$ & $100$ & $100$ \\
Spiral & $100/0$ & $100/0$ & $46/22$ & $0$ & $100$ & $100$ & $2$ \\
Circles & $90/0$ & $90/0$ & $37/13$ & $10$ & $100$ & $100$ & $12$ \\
Separation, $\delta{=}0.5$ & $89/0$ & $89/0$ & $33/24$ & $11$ & $100$ & $100$ & $28$ \\
Separation, $\delta{=}1$ & $49/0$ & $49/0$ & $14/4$ & $51$ & $100$ & $100$ & $61$ \\
Separation, $\delta{=}2$ & $0/0$ & $0/0$ & $0/0$ & $100$ & $100$ & $100$ & $100$ \\
Separation, $\delta{=}4$ & $1/0$ & $1/0$ & $1/0$ & $99$ & $100$ & $100$ & $99$ \\
Separation, $\delta{=}8$ & $1/0$ & $1/0$ & $0/0$ & $99$ & $100$ & $100$ & $99$ \\
Blobs, $d{=}5$ & $2/0$ & $2/0$ & $2/1$ & $98$ & $100$ & $100$ & $97$ \\
Blobs, $d{=}10$ & $4/0$ & $4/0$ & $2/2$ & $96$ & $100$ & $100$ & $94$ \\
Blobs, $d{=}50$ & $3/0$ & $3/0$ & $2/2$ & $97$ & $100$ & $100$ & $97$ \\
Iris & $84/0$ & $84/0$ & $57/12$ & $16$ & $100$ & $100$ & $38$ \\
Wine & $63/0$ & $63/0$ & $46/37$ & $37$ & $100$ & $100$ & $2$ \\
Breast cancer & $87/0$ & $87/0$ & $50/13$ & $13$ & $100$ & $100$ & $0$ \\
Digits & $81/19$ & $92/8$ & $61/39$ & $0$ & $0$ & $0$ & $0$ \\
\bottomrule
\end{tabular}
\caption{Quality of the K-Means fixed point reached from each paired k-means++ initialization. DA-Soft., DA-Ent.: deterministic annealing with alternating Softmax or Entmax-1.5 along a geometric temperature schedule starting above the critical temperature of Proposition~\ref{prop:critical-temperature}; EM-GMM: EM for a spherical Gaussian mixture initialized at the same centroids. Every method ends with Lloyd's algorithm. Better/worse: percentage of runs whose final distortion is lower/higher than that of Lloyd's algorithm from the same initialization by more than a relative $10^{-6}$. Best: percentage of runs reaching the lowest distortion found by any method on that dataset.}
\label{tab:annealing}
\end{table}

\newpage
\section{Statistical Summaries}
\label{app:statistics}

This appendix gives the full summaries announced in
Section~\ref{subsubsec:paired-initialization}. Table~\ref{tab:stats-slopes} reports
the mean and standard deviation of the per-run slopes of $E_r$ and $G_\sigma$ in
Table~\ref{tab:rates}, whose predicted value is $-1$, and a 95\% confidence interval
for their median. The interval is
$[x_{(j)},x_{(M'-j+1)}]$, where $x_{(1)}\leq\dots\leq x_{(M')}$ are the slopes of the
$M'$ admissible runs and $j$ is the largest index with
$\Pr\{j\leq\mathrm{Bin}(M',1/2)\leq M'-j\}\geq0.95$; it requires no assumption on
the distribution of the slopes. The intervals are very narrow because the slope of
a run is set by the observation closest to a Voronoi boundary, which varies little
between runs that reach the same fixed point.

\begin{table}[ht]
\centering
\footnotesize
\setlength{\tabcolsep}{4pt}
\begin{tabular}{lrcccc}
\toprule
 & & \multicolumn{2}{c}{Slope $E_r$} & \multicolumn{2}{c}{Slope $G_\sigma$} \\
\cmidrule(lr){3-4}\cmidrule(lr){5-6}
Dataset & Runs & Mean (sd) & 95\% CI of median & Mean (sd) & 95\% CI of median \\
\midrule
Blobs & 200 & $-0.998$ (4e-16) & $[-0.998,\,-0.998]$ & $-1.002$ (0.0005) & $[-1.002,\,-1.002]$ \\
Moons & 200 & $-0.998$ (2e-16) & $[-0.998,\,-0.998]$ & $-1.002$ (2e-16) & $[-1.002,\,-1.002]$ \\
Spiral & 200 & $-0.998$ (2e-06) & $[-0.998,\,-0.998]$ & $-1.003$ (0.004) & $[-1.003,\,-1.000]$ \\
Circles & 200 & $-0.998$ (3e-14) & $[-0.998,\,-0.998]$ & $-1.006$ (0.007) & $[-1.004,\,-1.000]$ \\
Separation, $\delta{=}0.5$ & 200 & $-0.998$ (3e-14) & $[-0.998,\,-0.998]$ & $-1.005$ (0.005) & $[-1.008,\,-1.005]$ \\
Separation, $\delta{=}1$ & 200 & $-0.998$ (4e-15) & $[-0.998,\,-0.998]$ & $-1.004$ (0.004) & $[-1.008,\,-1.000]$ \\
Separation, $\delta{=}2$ & 200 & $-0.998$ (2e-16) & $[-0.998,\,-0.998]$ & $-1.010$ (4e-16) & $[-1.010,\,-1.010]$ \\
Separation, $\delta{=}4$ & 200 & $-0.998$ (1e-08) & $[-0.998,\,-0.998]$ & $-1.050$ (0.005) & $[-1.051,\,-1.051]$ \\
Separation, $\delta{=}8$ & 200 & $-1.001$ (0.0003) & $[-1.001,\,-1.001]$ & $-1.085$ (0.008) & $[-1.086,\,-1.086]$ \\
Blobs, $d{=}5$ & 200 & $-0.999$ (0.0002) & $[-0.999,\,-0.999]$ & $-1.054$ (0.008) & $[-1.055,\,-1.055]$ \\
Blobs, $d{=}10$ & 200 & $-1.000$ (0.0003) & $[-1.000,\,-1.000]$ & $-1.064$ (0.01) & $[-1.066,\,-1.066]$ \\
Blobs, $d{=}50$ & 196 & $-1.017$ (0.003) & $[-1.017,\,-1.017]$ & $-1.086$ (0.01) & $[-1.088,\,-1.088]$ \\
Iris & 200 & $-0.998$ (4e-15) & $[-0.998,\,-0.998]$ & $-1.004$ (0.005) & $[-1.002,\,-1.000]$ \\
Wine & 200 & $-0.998$ (5e-06) & $[-0.998,\,-0.998]$ & $-1.005$ (0.006) & $[-1.002,\,-1.000]$ \\
Breast cancer & 200 & $-0.998$ (1e-15) & $[-0.998,\,-0.998]$ & $-1.008$ (0.006) & $[-1.015,\,-1.005]$ \\
Digits & 200 & $-0.998$ (8e-05) & $[-0.998,\,-0.998]$ & $-1.005$ (0.006) & $[-1.004,\,-1.001]$ \\
\bottomrule
\end{tabular}
\caption{Full summaries of the per-run slopes of $E_r$ and $G_\sigma$ in Table~\ref{tab:rates}: mean (standard deviation) over the admissible runs, and a distribution-free 95\% confidence interval for the median, from order statistics.}
\label{tab:stats-slopes}
\end{table}

Table~\ref{tab:stats-paired} tests the paired comparisons of terminal distortion
with Lloyd's algorithm, after the Lloyd polish of
Section~\ref{subsec:optimization-results}, so that every difference is a change of
basin rather than the $O(q(\sigma))$ offset of a soft fixed point
(Theorem~\ref{thm:soft-fixed-point-branch}). On the well-separated datasets such
changes are rare (at most $9$ of about $200$ runs for any method), whereas on the
datasets with ambiguous cluster structure they are frequent and, except for
gradient descent on digits, mostly in favor of the soft method, with $p$-values
below $10^{-4}$ except on spiral ($p=0.0055$) and breast cancer ($p=0.016$) for
alternating Entmax, and on spiral, breast cancer and digits for gradient descent. These
$p$-values are descriptive. The runs of one dataset share the same data, many of
them reach the same fixed point, and no correction for multiple comparisons is
applied, so they quantify how consistent a difference is within one dataset, not
how it generalizes across datasets; the replication across dataset draws in
Appendix~\ref{app:sensitivity} addresses the latter.

\begin{table}[ht]
\centering
\scriptsize
\setlength{\tabcolsep}{4pt}
\begin{tabular}{lrrrrrrrr}
\toprule
 & \multicolumn{2}{c}{Alt.\ Softmax} & \multicolumn{2}{c}{Alt.\ Entmax} & \multicolumn{2}{c}{GD on $F_\sigma$} & \multicolumn{2}{c}{DA-Softmax} \\
\cmidrule(lr){2-3}\cmidrule(lr){4-5}\cmidrule(lr){6-7}\cmidrule(lr){8-9}
Dataset & L/H & $p$ & L/H & $p$ & L/H & $p$ & L/H & $p$ \\
\midrule
Blobs & $4/0$ & $0.12$ & $3/0$ & $0.25$ & $3/1$ & $0.62$ & $4/0$ & $0.12$ \\
Moons & $0/0$ & $1$ & $0/0$ & $1$ & $0/0$ & $1$ & $0/0$ & $1$ \\
Spiral & $99/36$ & $<10^{-7}$ & $66/37$ & $0.0055$ & $96/68$ & $0.035$ & $200/0$ & $<10^{-59}$ \\
Circles & $81/4$ & $<10^{-18}$ & $69/2$ & $<10^{-17}$ & $94/43$ & $<10^{-4}$ & $180/0$ & $<10^{-53}$ \\
Separation, $\delta{=}0.5$ & $90/1$ & $<10^{-25}$ & $84/2$ & $<10^{-22}$ & $99/27$ & $<10^{-10}$ & $178/0$ & $<10^{-53}$ \\
Separation, $\delta{=}1$ & $98/0$ & $<10^{-29}$ & $98/0$ & $<10^{-29}$ & $98/0$ & $<10^{-29}$ & $98/0$ & $<10^{-29}$ \\
Separation, $\delta{=}2$ & $0/0$ & $1$ & $0/0$ & $1$ & $0/0$ & $1$ & $0/0$ & $1$ \\
Separation, $\delta{=}4$ & $1/1$ & $1$ & $1/0$ & $1$ & $1/1$ & $1$ & $2/0$ & $0.5$ \\
Separation, $\delta{=}8$ & $2/0$ & $0.5$ & $1/0$ & $1$ & $2/0$ & $0.5$ & $2/0$ & $0.5$ \\
Blobs, $d{=}5$ & $5/0$ & $0.062$ & $5/1$ & $0.22$ & $3/6$ & $0.51$ & $5/0$ & $0.062$ \\
Blobs, $d{=}10$ & $5/0$ & $0.062$ & $7/0$ & $0.016$ & $4/1$ & $0.38$ & $8/0$ & $0.0078$ \\
Blobs, $d{=}50$ & $5/0$ & $0.062$ & $5/0$ & $0.062$ & $5/0$ & $0.062$ & $5/0$ & $0.062$ \\
Iris & $117/1$ & $<10^{-33}$ & $52/0$ & $<10^{-15}$ & $125/5$ & $<10^{-30}$ & $169/0$ & $<10^{-50}$ \\
Wine & $89/2$ & $<10^{-23}$ & $54/1$ & $<10^{-14}$ & $96/30$ & $<10^{-8}$ & $126/0$ & $<10^{-37}$ \\
Breast cancer & $18/0$ & $<10^{-5}$ & $7/0$ & $0.016$ & $27/23$ & $0.67$ & $174/0$ & $<10^{-52}$ \\
Digits & $130/16$ & $<10^{-22}$ & $79/29$ & $<10^{-5}$ & $83/113$ & $0.038$ & $162/38$ & $<10^{-18}$ \\
\bottomrule
\end{tabular}
\caption{Paired comparison of the hard distortion after a Lloyd polish of the terminal centroids with that of Lloyd's algorithm from the same initialization: number of runs with a lower/higher distortion by more than a relative $10^{-6}$, and the two-sided $p$-value of the exact sign test on these runs (ties excluded, no correction for multiple comparisons). L/H: lower/higher. Alternating methods at $\beta=4$ (Section~\ref{subsec:optimization-results}); DA-Softmax: deterministic annealing (Section~\ref{subsec:annealing-results}). Up to 200 paired runs per cell.}
\label{tab:stats-paired}
\end{table}

\section{Additional Results and Sensitivity}
\label{app:sensitivity}

Table~\ref{tab:deep-nmi-ari} completes Table~\ref{tab:deep} with the NMI and ARI of
every method.

\begin{table}[ht]
\centering
\scriptsize
\setlength{\tabcolsep}{3pt}
\begin{tabular}{lcccccccccc}
\toprule
 & \multicolumn{2}{c}{MNIST} & \multicolumn{2}{c}{Fashion-MNIST} & \multicolumn{2}{c}{USPS} & \multicolumn{2}{c}{CIFAR-10} & \multicolumn{2}{c}{STL-10} \\
\cmidrule(lr){2-3}\cmidrule(lr){4-5}\cmidrule(lr){6-7}\cmidrule(lr){8-9}\cmidrule(lr){10-11}
Method & NMI & ARI & NMI & ARI & NMI & ARI & NMI & ARI & NMI & ARI \\
\midrule
AE + K-Means & $0.757$ & $0.692$ & $0.553$ & $0.383$ & $0.692$ & $0.602$ & $0.559$ & $0.448$ & $0.683$ & $0.527$ \\
DEC & $0.776$ & $0.716$ & $0.577$ & $0.403$ & $0.785$ & $0.713$ & $0.593$ & $0.500$ & $0.738$ & $0.592$ \\
IDEC & $0.830$ & $0.767$ & $0.602$ & $0.421$ & $0.793$ & $0.722$ & $0.621$ & $0.532$ & $0.741$ & $0.599$ \\
DCN (hard) & $0.820$ & $0.735$ & $0.588$ & $0.394$ & $0.741$ & $0.624$ & $0.583$ & $0.381$ & $0.695$ & $0.455$ \\
$D_\sigma$ & $0.830$ & $0.750$ & $0.584$ & $0.367$ & $0.729$ & $0.585$ & $0.567$ & $0.352$ & $0.700$ & $0.456$ \\
$F_\sigma$ (ours) & $0.827$ & $0.760$ & $0.568$ & $0.380$ & $0.672$ & $0.566$ & $0.572$ & $0.434$ & $0.700$ & $0.529$ \\
$F_\sigma^{(3/2)}$ (ours) & $0.831$ & $0.762$ & $0.584$ & $0.386$ & $0.687$ & $0.582$ & $0.574$ & $0.420$ & $0.718$ & $0.530$ \\
\midrule
$F_\sigma$, Adam on $\mu$ & $0.826$ & $0.740$ & $0.568$ & $0.352$ & $0.713$ & $0.552$ & $0.486$ & $0.264$ & $0.672$ & $0.421$ \\
$F_\sigma$, no recon. & $0.811$ & $0.687$ & $0.439$ & $0.223$ & $0.481$ & $0.323$ & $0.306$ & $0.169$ & $0.019$ & $0.005$ \\
$F_\sigma$, fixed $\sigma$ & $0.034$ & $0.016$ & $0.395$ & $0.154$ & $0.000$ & $0.000$ & $0.000$ & $0.000$ & $0.000$ & $0.000$ \\
$F_\sigma$, annealed $\sigma$ & $0.476$ & $0.366$ & $0.524$ & $0.382$ & $0.566$ & $0.463$ & $0.426$ & $0.296$ & $0.704$ & $0.588$ \\
\bottomrule
\end{tabular}
\caption{NMI and ARI of the methods of Table~\ref{tab:deep}, mean over 5 seeds.}
\label{tab:deep-nmi-ari}
\end{table}

\paragraph{Replication.} The paired runs of one dataset are not independent
replicates: many initializations reach the same Lloyd fixed point, and the
$3{,}196$ reference runs contain only $415$ distinct reference configurations (one
on moons and on the separation sweep with $\delta=2$, $200$ on digits). The
uncertainty that matters for the conclusions is therefore that of the data, and
Table~\ref{tab:replication} measures it by repeating the main and the additional
experiments with five further seeds, each of which draws every synthetic dataset
and every paired initialization again, and by repeating the main experiments with
$1{,}000$ instead of $300$ points per synthetic dataset. Every bound check is
satisfied by every record, Entmax identification has no violation above the
threshold, and the median slopes stay within the same narrow range around $-1$.
The predictions of Theorem~\ref{thm:trajectory-tracking} and
Corollary~\ref{cor:entmax-trajectory} hold without exception in all six draws, in
$9{,}883$ and $35{,}905$ runs in total, and the fractions of runs that follow
Lloyd's path vary by a few percentage points. The quantities that depend on the
particular sample move more: after the Lloyd polish, alternating Softmax reaches a
better K-Means fixed point than Lloyd's algorithm in $19.6\%$ to $28.1\%$ of the
runs across the six draws and alternating Entmax in $11.3\%$ to $18.8\%$, always
with a worse one in at most $2.3\%$. Deterministic annealing, which is never worse
than Lloyd's algorithm outside digits in the reported draw, is worse in $162$ of
$14{,}982$ such runs in the other five, all on circles and spiral and by at most
$0.24\%$ of the distortion: in those draws the single fixed point that the
schedule selects is not the best one that some initialization of Lloyd's
algorithm reaches.

\begin{table}[ht]
\centering
\footnotesize
\setlength{\tabcolsep}{4pt}
\begin{tabular}{lccc}
\toprule
Quantity & Reported & Seeds 1--5 & $n=1000$ \\
\midrule
Smallest fraction satisfying a bound check & $1.000$ & $1.000$ & $1.000$ \\
Entmax identification violations ($\beta\geq2$) & $0$ & $0$ & $0$ \\
Median slope of $E_r$ (range over datasets) & $[-1.017,-0.998]$ & $[-1.020,-0.998]$ & $[-1.016,-0.998]$ \\
Median slope of $G_\sigma$ (range over datasets) & $[-1.088,-1.001]$ & $[-1.092,-1.000]$ & $[-1.100,-1.000]$ \\
Gradient descent converged (\%) & $93.8$ & $92.5--93.9$ & $93.7$ \\
Alt.\ Softmax: lower / higher $J$ than Lloyd (\%) & $23.3$ / $1.9$ & $19.6--28.1$ / $1.0--1.4$ & $19.5$ / $1.0$ \\
Alt.\ Entmax: lower / higher $J$ than Lloyd (\%) & $16.6$ / $2.3$ & $11.3--18.8$ / $1.1--2.2$ & $11.8$ / $1.7$ \\
Theorem~\ref{thm:trajectory-tracking}: runs / violations & $1{,}663$ / $0$ & $8{,}220$ / $0$ & -- \\
Corollary~\ref{cor:entmax-trajectory}: runs / violations & $5{,}644$ / $0$ & $30{,}261$ / $0$ & -- \\
Same path as Lloyd: Softmax / Entmax / GD (\%) & $53/61/35$ & $50--55/58--63/29--36$ & -- \\
DA-Softmax better than Lloyd (\%) & $41.1$ & $43.4--51.1$ & -- \\
DA-Softmax worse than Lloyd, outside digits & $0$ of $2{,}996$ & $162$ of $14{,}982$ ($0.24\%$ at most) & -- \\
Largest fraction of vanishing gradients & $0$ & $0$ & -- \\
\bottomrule
\end{tabular}
\caption{Replication of the headline results. Reported: the runs of the main text ($300$ points per synthetic dataset, seed $0$). Seeds $1$--$5$: five further draws of every synthetic dataset and of every paired initialization, summarized by the range over the five draws (counts are summed). $n=1000$: the main experiments with $1{,}000$ points per synthetic dataset (the real datasets are unchanged); the additional experiments were not repeated at this size. The comparisons of $J$ use the Lloyd polish of Section~\ref{subsec:optimization-results}.}
\label{tab:replication}
\end{table}

\paragraph{Deep clustering.} Table~\ref{tab:sens-deep} varies the weight $\gamma$ of the
clustering term and the relative temperature on MNIST, USPS and CIFAR-10, from
the pretrained networks of the first three seeds of Table~\ref{tab:deep}; the
default configuration reproduces those runs to three decimals. These grids
are evaluated with the ground-truth labels, so the best value of a row is an
optimistic estimate of what tuning without labels would achieve; we read them
for trends, not as tuned results. Two findings qualify the main text. First, the K-Means-type losses prefer a smaller weight:
with $\gamma=0.01$ every one of them improves on USPS and CIFAR-10, and $F_\sigma$
reaches $0.647$ on CIFAR-10, the accuracy of IDEC at its best weight ($0.644$). On
USPS IDEC remains better at every weight. With $\gamma=1$ all K-Means-type losses
degrade, by up to half of their accuracy on CIFAR-10, consistently with the
contraction of the representation, whereas IDEC is almost unaffected. Second,
a larger relative temperature collapses $F_\sigma$, to an accuracy of $0.27$ on MNIST
and $0.17$ on CIFAR-10 at $0.3$, which is again the critical-temperature
transition of Proposition~\ref{prop:critical-temperature}, whereas
$F_\sigma^{(3/2)}$ degrades much less (from $0.831$ to $0.807$ on MNIST, from $0.671$ to
$0.589$ on USPS and from $0.575$ to $0.516$ on CIFAR-10). This is a side effect of sparsity rather than of exactness: Entmax
responsibilities vanish outside a margin of $4\sigma^2$, so distant clusters exert no
force and the centroids do not merge.
Within the range $0.03$ to $0.1$ both reduced objectives are insensitive to the
temperature.

\begin{table}[ht]
\centering
\scriptsize
\setlength{\tabcolsep}{4pt}
\begin{tabular}{lccccccccc}
\toprule
 & \multicolumn{3}{c}{MNIST} & \multicolumn{3}{c}{USPS} & \multicolumn{3}{c}{CIFAR-10} \\
\cmidrule(lr){2-4}\cmidrule(lr){5-7}\cmidrule(lr){8-10}
\midrule
\multicolumn{10}{l}{\emph{Weight $\gamma$ (relative temperature $0.1$)}} \\
Method & $0.01$ & $0.1$ & $1$ & $0.01$ & $0.1$ & $1$ & $0.01$ & $0.1$ & $1$ \\
\midrule
IDEC & $\mathbf{0.830}$ & $0.820$ & $0.798$ & $0.753$ & $\mathbf{0.787}$ & $0.778$ & $0.600$ & $\mathbf{0.644}$ & $0.636$ \\
DCN (hard) & $\mathbf{0.825}$ & $0.812$ & $0.794$ & $\mathbf{0.746}$ & $0.722$ & $0.692$ & $\mathbf{0.591}$ & $0.531$ & $0.281$ \\
$D_\sigma$ & $\mathbf{0.831}$ & $0.823$ & $0.811$ & $\mathbf{0.740}$ & $0.666$ & $0.666$ & $\mathbf{0.537}$ & $0.504$ & $0.358$ \\
$F_\sigma$ (ours) & $0.822$ & $\mathbf{0.827}$ & $0.782$ & $\mathbf{0.734}$ & $0.657$ & $0.531$ & $\mathbf{0.647}$ & $0.560$ & $0.299$ \\
$F_\sigma^{(3/2)}$ (ours) & $0.824$ & $\mathbf{0.831}$ & $0.805$ & $\mathbf{0.737}$ & $0.671$ & $0.628$ & $\mathbf{0.620}$ & $0.575$ & $0.324$ \\
\midrule
\multicolumn{10}{l}{\emph{Relative temperature ($\gamma=0.1$)}} \\
Method & $0.03$ & $0.1$ & $0.3$ & $0.03$ & $0.1$ & $0.3$ & $0.03$ & $0.1$ & $0.3$ \\
\midrule
$D_\sigma$ & $\mathbf{0.824}$ & $0.823$ & $0.823$ & $0.643$ & $\mathbf{0.666}$ & $0.661$ & $0.488$ & $\mathbf{0.504}$ & $0.409$ \\
$F_\sigma$ (ours) & $\mathbf{0.831}$ & $0.827$ & $0.271$ & $\mathbf{0.673}$ & $0.657$ & $0.368$ & $\mathbf{0.572}$ & $0.560$ & $0.172$ \\
$F_\sigma^{(3/2)}$ (ours) & $\mathbf{0.831}$ & $0.831$ & $0.807$ & $\mathbf{0.674}$ & $0.671$ & $0.589$ & $0.573$ & $\mathbf{0.575}$ & $0.516$ \\
\bottomrule
\end{tabular}
\caption{Sensitivity of the deep-clustering losses: accuracy, mean over 3 seeds, from the same pretrained networks as Table~\ref{tab:deep}. Top: weight $\gamma$ of the clustering term; bottom: relative temperature $\sigma^2/\overline{d}_{\min}$ of the soft objectives. Best value of each row and dataset in bold.}
\label{tab:sens-deep}
\end{table}

Table~\ref{tab:lr-sweep} varies the learning rate of Adam on the centroids for
every method that trains them by gradient, with five seeds. DEC and IDEC are best
at the default rate, and $D_\sigma$ gains up to $0.03$ with a larger one. For
$F_\sigma$, the gap between Adam at the default rate and the closed-form update
closes once the rate is tuned: Adam then matches the closed form on MNIST
($0.824$ against $0.830$) and USPS ($0.659$ against $0.668$) and exceeds it on
CIFAR-10 at the rate $10^{-2}$ ($0.641$ against $0.574$). The best rate differs across
datasets, from $10^{-4}$ on MNIST to $10^{-2}$ on CIFAR-10, which is the tuning that
the closed form avoids.

\begin{table}[ht]
\centering
\footnotesize
\setlength{\tabcolsep}{4pt}
\begin{tabular}{llccc}
\toprule
Method & Centroid learning rate & MNIST & USPS & CIFAR-10 \\
\midrule
$F_\sigma$, closed form & -- & $0.830\pm0.014$ & $0.668\pm0.023$ & $0.574\pm0.054$ \\
$F_\sigma^{(3/2)}$, closed form & -- & $0.831\pm0.012$ & $0.681\pm0.023$ & $0.569\pm0.027$ \\
\midrule
DEC & $10^{-4}$ & $0.790\pm0.015$ & $0.780\pm0.007$ & $0.624\pm0.045$ \\
DEC & $10^{-3}$ & $0.790\pm0.013$ & $0.775\pm0.007$ & $0.622\pm0.043$ \\
DEC & $10^{-2}$ & $0.788\pm0.013$ & $0.765\pm0.011$ & $0.617\pm0.045$ \\
IDEC & $10^{-4}$ & $0.818\pm0.013$ & $0.789\pm0.004$ & $0.648\pm0.049$ \\
IDEC & $10^{-3}$ & $0.816\pm0.010$ & $0.779\pm0.006$ & $0.646\pm0.057$ \\
IDEC & $10^{-2}$ & $0.807\pm0.012$ & $0.768\pm0.008$ & $0.616\pm0.049$ \\
$D_\sigma$ & $10^{-4}$ & $0.828\pm0.023$ & $0.679\pm0.042$ & $0.494\pm0.069$ \\
$D_\sigma$ & $10^{-3}$ & $0.826\pm0.011$ & $0.687\pm0.033$ & $0.544\pm0.055$ \\
$D_\sigma$ & $10^{-2}$ & $0.789\pm0.037$ & $0.710\pm0.012$ & $0.533\pm0.042$ \\
$F_\sigma$, Adam on $\mu$ & $10^{-4}$ & $0.824\pm0.028$ & $0.656\pm0.072$ & $0.360\pm0.127$ \\
$F_\sigma$, Adam on $\mu$ & $10^{-3}$ & $0.800\pm0.025$ & $0.659\pm0.074$ & $0.419\pm0.114$ \\
$F_\sigma$, Adam on $\mu$ & $10^{-2}$ & $0.808\pm0.027$ & $0.653\pm0.026$ & $0.641\pm0.045$ \\
\bottomrule
\end{tabular}
\caption{Learning rate of Adam on the centroids: accuracy, mean $\pm$ standard deviation over 5 seeds, from the same pretrained networks as Table~\ref{tab:deep}. The closed-form methods have no centroid learning rate; the network is trained with Adam at $10^{-4}$ in every row.}
\label{tab:lr-sweep}
\end{table}

\paragraph{Codebook learning.} Table~\ref{tab:sens-vq} repeats the comparison of
Table~\ref{tab:vq} with $128$ and $1{,}024$ codes and with relative temperatures
$0.03$ and $0.3$. The ordering is the same at every codebook size: Softmax
responsibilities use essentially every code and give the lowest reconstruction
error, and with $1{,}024$ codes they use $1{,}023$ and $991$ codes against $524$ for the
original loss and $132$ for its moving-average variant. Over this range of
temperatures the reconstruction error of the Softmax rules changes by less than
$5\%$, and every rule uses more codes as the temperature grows, the Entmax rules
because their support widens.

\begin{table}[ht]
\centering
\scriptsize
\setlength{\tabcolsep}{4pt}
\begin{tabular}{llrrrrrr}
\toprule
 & & \multicolumn{2}{c}{$K=128$} & \multicolumn{2}{c}{$K=512$} & \multicolumn{2}{c}{$K=1024$} \\
\cmidrule(lr){3-4}\cmidrule(lr){5-6}\cmidrule(lr){7-8}
Assignment & Update & MSE & Used & MSE & Used & MSE & Used \\
\midrule
Hard & Gradient & $4.80$ & $124$ & $4.16$ & $372$ & $3.86$ & $524$ \\
Hard & Closed form & $5.58$ & $56$ & $4.98$ & $104$ & $4.80$ & $132$ \\
Softmax & Gradient & $4.77$ & $128$ & $3.93$ & $511$ & $3.53$ & $1023$ \\
Softmax & Closed form & $4.83$ & $128$ & $4.06$ & $504$ & $3.61$ & $991$ \\
Entmax-1.5 & Gradient & $4.78$ & $127$ & $3.99$ & $455$ & $3.62$ & $798$ \\
Entmax-1.5 & Closed form & $5.19$ & $85$ & $4.49$ & $211$ & $4.33$ & $280$ \\
\midrule
 & & \multicolumn{2}{c}{$\sigma^2$ scale $0.03$} & \multicolumn{2}{c}{$\sigma^2$ scale $0.1$} & \multicolumn{2}{c}{$\sigma^2$ scale $0.3$} \\
\cmidrule(lr){3-4}\cmidrule(lr){5-6}\cmidrule(lr){7-8}
Softmax & Gradient & $3.98$ & $507$ & $3.93$ & $511$ & $4.10$ & $512$ \\
Softmax & Closed form & $4.13$ & $382$ & $4.06$ & $504$ & $4.20$ & $512$ \\
Entmax-1.5 & Gradient & $4.07$ & $413$ & $3.99$ & $455$ & $3.94$ & $501$ \\
Entmax-1.5 & Closed form & $4.75$ & $144$ & $4.49$ & $211$ & $4.36$ & $292$ \\
\bottomrule
\end{tabular}
\caption{Sensitivity of the VQ-VAE codebook rules of Table~\ref{tab:vq}, mean over 3 seeds: test reconstruction error ($\times10^{-3}$) and number of codes used. Top: codebook size $K$ at relative temperature $0.1$; bottom: relative temperature of the soft rules at $K=512$.}
\label{tab:sens-vq}
\end{table}

\bibliographystyle{tmlr}
\bibliography{sample}

\end{document}